\documentclass{article}

\usepackage{microtype}
\usepackage{graphicx}
\usepackage{subfigure}
\usepackage{booktabs} 
\usepackage{comment}

\usepackage{hyperref}
\usepackage{enumitem}

\newcommand{\id}{\mathtt{id}} 
\newcommand\mId[1]{\mathrm{Id}_{#1}}
\newcommand\eq[1]{\begin{equation}#1\end{equation}}
\newcommand{\RR}{\mathbb{R}}

\newcommand{\RD}{\mathbb{R}^D} 
\newcommand{\Rd}{\mathbb{R}^d} 
\newcommand{\ouv}{\Omega} 
\newcommand{\bigouv}{\tilde{\Omega}}
\newcommand{\vdim}{\mathrm{dim}}
\newcommand{\linspan}{\mathrm{span}}

\newcommand{\lie}{\mathrm{Lie}}
\newcommand{\R}{\mathbb{R}}
\newcommand{\W}{\mathcal{W}_{\phi}}

\newcommand{\x}{\theta} 
\newcommand{\z}{\alpha}
\newcommand{\smallvec}[1]{ \left(\begin{smallmatrix}#1\end{smallmatrix}\right) }
\newcommand{\diag}{\mathrm{diag}}

\usepackage[accepted]{icml2024}

\usepackage{mathtools}
\usepackage{amsmath,amssymb,amsthm}
\usepackage{thmtools}

\usepackage[capitalize,noabbrev]{cleveref}

\theoremstyle{plain}
\newtheorem{theorem}{Theorem}[section]
\newtheorem{proposition}[theorem]{Proposition}
\newtheorem{lemma}[theorem]{Lemma}

\theoremstyle{definition}
\newtheorem{definition}[theorem]{Definition}
\newtheorem{assumption}[theorem]{Assumption}
\theoremstyle{remark}
\newtheorem{remark}[theorem]{Remark}
\newtheorem{example}[theorem]{Example} 

\usepackage{thm-restate}

\usepackage[textsize=tiny]{todonotes}

\icmltitlerunning{Conservation Laws beyond Euclidean Gradient Flows}

\begin{document}

\twocolumn[
\icmltitle{Keep the Momentum: Conservation Laws beyond Euclidean Gradient Flows}



\icmlsetsymbol{equal}{*}

\begin{icmlauthorlist}
\icmlauthor{Sibylle Marcotte}{ens}
\icmlauthor{Rémi Gribonval}{lyon}
\icmlauthor{Gabriel Peyré}{ens,cnrs}
\end{icmlauthorlist}

\icmlaffiliation{ens}{ENS - PSL Univ.}
\icmlaffiliation{lyon}{Univ Lyon, EnsL, UCBL,
  CNRS, Inria,  LIP.}
\icmlaffiliation{cnrs}{CNRS}

\icmlcorrespondingauthor{Sibylle Marcotte}{sibylle.marcotte@ens.fr}

\icmlkeywords{Machine Learning, ICML}

\vskip 0.3in
]



\printAffiliationsAndNotice{}  
%
\begin{abstract}
Conservation laws are well-established in the context of Euclidean gradient flow dynamics, notably for linear or ReLU neural network training. Yet, their existence and principles for non-Euclidean geometries and momentum-based dynamics remain largely unknown.  In this paper, we characterize ``all'' conservation laws in this general setting.  In stark contrast to the case of gradient flows, we prove that the conservation laws for momentum-based dynamics exhibit temporal dependence. Additionally, we {often observe a} ``conservation loss'' when transitioning from gradient flow to momentum dynamics. Specifically, for linear networks, our framework allows us to identify all momentum conservation laws, which are less numerous than in the gradient flow case {except in sufficiently over-parameterized regimes}. With ReLU networks, no conservation law remains. This phenomenon also manifests in non-Euclidean metrics, used e.g. for Nonnegative Matrix Factorization (NMF): all conservation laws can be determined in the gradient flow context, yet none persists in the momentum case.
\end{abstract}

\section{Introduction}

Discovering functions that remain unchanged during the optimization of neural networks is important to gain insight into the properties of trained models. 
While these laws are understood in Euclidean gradient flows, they are much less studied in non-Euclidean metrics or momentum dynamics. 
\paragraph{Conservation laws of Euclidean gradient flows} 
are known and extensively used to study the training of linear and ReLU networks without momentum.
These laws correspond to ``balancedness properties'' between neurons across layers~\cite{Saxe, Du, Arora18a}.  %
They can be leveraged to understand which specific attributes (e.g. sparsity, low-rank, etc.) the optimization process tends to select from the potentially infinite pool of solutions~\cite{Saxe, Bah, Arora18b, Tarmoun, Min}.
They can also be used to prove under restrictive conditions the global convergence of gradient descent~\cite{Du, Arora18a, Bah, chizat, Ji, Min}.
For linear and ReLU neural networks, the number of conservation laws for a gradient flow dynamic is determined by the dimension of a Lie algebra and there are no more conservation laws than these ``balancedness'' laws \cite{marcotte2023abide}.

\paragraph{Momentum and non-Euclidean metrics.}
Momentum and non-Euclidean metrics are two key ideas to accelerate the convergence of optimization schemes, enable the use of larger step sizes, and take into account constraints on the weights. 
The initial idea dates back to Polyak's heavy ball~\cite{polyak1964some}, which introduces a momentum in the gradient descent algorithm to perform an extrapolation step. 
In the continuous-time limit of using small step sizes, this corresponds to using a second-order differential equation. 
Nesterov's acceleration~\cite{nesterov1983method} goes one step further by progressively increasing the momentum strength during the dynamics, reaching a faster convergence rate on the class of smooth functions. 
A complementary idea to better capture the curvature of the loss is to use spatially varying metrics. The simplest cases are data-independent metrics, which are Hessian of some potential function~\cite{raskutti2015information}. This corresponds to the continuous-time counterpart to the mirror descent algorithm~\cite{nemirovskij1983problem}, which is closely related to optimization schemes using Bregman's divergences~\cite{bregman1967relaxation}. Another advantage of these non-Euclidean metrics is that they can naturally enforce constraints, such as positivity when using the mirror descent metric associated with the Shannon entropy potential~\cite{bubeck2015convex}. 
Data-dependent metrics estimate the local curvature using variations around the idea of natural gradient~\cite{amari1998natural}. They are popular for training neural networks, using efficient low-rank approximations of the metric \cite{martens2010deep, martens2012training, martens2015optimizing}.
\paragraph{Conservation laws and momentum.}
In sharp contrast with first-order flows, conservation laws of momentum flows and non-Euclidean metrics remain mostly unexplored. 
The simplest approach to derive conservation laws is to apply Noether's theorem \cite{Noether1918} to transformations leaving the loss invariant. Leveraging Lagrangian's formulations of the flows~\cite{jordan}, this leads to preserving some form of inertial quantities~\cite{Tanaka}. 
While gradient flows can be seen as a small-momentum limit of second-order flow, this limit is singular. 
In particular, invariances of the loss do not immediately lead to conservation laws for gradient flows~\cite{zhao,Tanaka, Kunin}.
One of the goals of this paper is to expose the fundamental differences in the structure and number of the conservation laws for gradient and momentum flows.
\paragraph{Contributions.}  
 We define the {\bf concept of conservation laws for Momentum flows} ( \Cref{section:Phase-spacelifting}) and show how to extend the framework from paper \cite{marcotte2023abide} for non-Euclidean gradient flows and momentum flow settings (\Cref{prop:link_mom_grad}).
We prove the {\bf time independence of conservation laws in the gradient flow case}, and {\bf a non-trivial time dependence in the momentum case } (\Cref{thm:structure}).
We uncover {\bf new conservation laws for linear networks in the Euclidean momentum case} (\Cref{thm:conservedfunctionsEuclidean}). These new laws {\bf are complete}, as proved theoretically for depth-two cases (\Cref{theorem:dimliealgebra}, \Cref{prop:counting}), and algorithmically through formal computations for deeper cases (\Cref{section:linear});
We show that, in contrast, there is {\bf no conservation law for ReLU networks in the Euclidean momentum case}, as proved theoretically for two-layer cases, and by formal computations for deeper cases (\Cref{section:ReLU}).
We shed light on a {quasi systematic} {\bf loss of conservation} when transitioning from Euclidean gradient flows to the Euclidean momentum setting (\Cref{section:loss});
this loss also occurs in a non-Euclidean context, such as in Non-negative Matrix Factorization \textbf{(NMF)} or for Input Convex Neural Networks (\textbf{ICNN}) implemented with two-layer ReLU networks
where we find out {\bf new conservation laws for gradient flows} (\Cref{thm:conservedfunctionsNMF} and \Cref{thm:ICNN}, and find {\bf none in the momentum case} (See \Cref{section:NMF}, \Cref{section:ICNN}) ;
{\bf We obtain new conservation laws in the Natural Gradient Flow case} (\Cref{section:naturalgradient}). 

\section{Conservation Laws for Momentum Flows}
We formalize the concept of 
conservation laws for momentum dynamics and establish their generic time-dependence properties. We characterize their most important properties in relation to certain linear spaces of vector fields.
\subsection{Momentum dynamics} \label{Context}
We explore learning problems with features 
$x_i \in \RR^m$ and targets $y_i \in \mathcal{Y}$ (typically for regression, with $\mathcal{Y} = \RR^n$) or labels (for classification) within the scope of supervised learning. In the context of unsupervised or self-supervised learning, the $y_i$ can be treated as constant.
We denote $z_i \coloneqq (x_i, y_i)$ and $Z = (x_i, y_i)_i$.
Prediction is accomplished through a parametric function $g(\theta, \cdot): \RR^m \to \RR^n$ (such as a neural network), which is trained by empirically minimizing a \textbf{cost} over the parameters $\x \in \Theta \subseteq \RD$ 
\begin{equation}\label{eq:erm} 
        \mathcal{E}_{Z}(\theta) \coloneqq \textstyle\sum_i \ell(g({\x},x_i),y_i),     
\end{equation} 
with $\ell$ a \textbf{loss} function. The parameter set $\Theta \subseteq \RD$ is typically either the whole parameter space 
$\RD$ or an open set of ``non-degenerate'' parameters for linear or ReLU networks.

This paper studies quantities 
that are preserved during the minimization of $\mathcal{E}_Z$ defined in~\eqref{eq:erm} using dynamical flows. 
First-order dynamics corresponds to gradient flows 
\eq{\label{gradientflow}
        \dot{\x}(t) =  -M_Z(\x(t))\nabla \mathcal{E}_Z (\x(t))
}
where $M_Z(\theta) \in \mathbb{R}^{D \times D}$ is typically a positive semi-definite matrix.
Conservation laws $h(\theta)$ for these flows 
have been studied in-depth in the Euclidean case ($M$ is the identity matrix $\mId{D}$) \cite{marcotte2023abide}, and we study what happens in a non-Euclidean setting.
Another goal is to go beyond first-order dynamics and to analyze which functions $h(t, \x,\dot{\x} )$ are preserved during the momentum flow of $\mathcal{E}_Z$:
\begin{equation} \label{momentumflow}
        \ddot {\x}(t) + 
       \tau(t) \dot{\x}(t) = - M_Z (t, \x(t), \dot{\x}(t) ) \nabla \mathcal{E}_Z (\x(t)).
\end{equation}
{We will always assume that $t \mapsto \tau(t)$ is infinitely smooth.}
To anticipate our findings, let us immediately mention that  
introducing momentum leads to {\em conserved quantities that depend on time and velocity}, and results {in {many} 
cases {of interest}} in a {\em reduction of conservation properties}. The latter phenomenon consistently emerges across all our examples, whether in Euclidean ($M=\mId{D}$) or non-Euclidean settings, such as in non-negative matrix factorization. Consequently, we will draw comparisons with the gradient flow scenario. 

\subsection{Main Examples}
We consider several settings of practical interest. 
\paragraph{Examples of models.} Prime examples are {\bf \em  two-layer linear or ReLU networks}, where $\x = (U,V)$ with matrices $U,V$ of appropriate sizes and $g(\x,x) \coloneqq U \sigma(V^\top x)$. Here $\sigma = \id$ for the linear case, while $\sigma$ is the entrywise ReLU activation function for ReLU networks. In the latter case, deeper examples as well as biases can also be considered.

\paragraph{Example of flows.}
In terms of flows, we consider: 
\textbf{Gradient flows (GF)},  corresponding to~\eqref{gradientflow}. This can informally be thought of as using $\tau=\infty$ in \eqref{momentumflow} {(noticing that the matrix $M_Z$ in \eqref{momentumflow} can depend on $\tau$ via $t$)}. 
\textbf{Heavy ball}, i.e. \eqref{momentumflow} with a fixed $\tau(t)=\tau < \infty$. This corresponds to a continuous-time limit of Poliak's heavy ball acceleration method~\cite{polyak1964some}. 
\textbf{Nesterov acceleration}, i.e. \eqref{momentumflow} with $\tau(t)=3/t$. This corresponds to the flow introduced by~\cite{su2016differential} as a continuous-time limit of Nesterov accelerated gradient descent scheme~\cite{nesterov1983method}. 

\paragraph{Example of metrics.}
We illustrate our findings on:
\textbf{Euclidean geometry}, i.e.  $M_Z(t,\theta,\dot\theta) = \mId{D}$. 
%
\textbf{Mirror geometry}, {associated to the Shannon entropy potential,}  uses $M_Z(\x)= \diag(\x)$ for gradient flows, and 
in the Heavy ball case, it uses $M_Z(t,\theta,\dot \theta)=\diag(\dot \theta + \tau \theta)$ \cite{jordan}.
The associated flow is a continuous time limit of mirror descent~\cite{nemirovskij1983problem}.
While such approaches were initially developed for first-order schemes, their extension to second-order flows that we use is derived in~\cite{jordan} as a flow for a Bregman-type Lagrangian. 
{If this paper particularly focuses on the case of the Shannon entropy potential as an example, note that our theory applies to any mirror potential.}
\textbf{Natural gradient}~\cite{amari1998natural} avoids the issues of Newton's method using a first-order estimation of the curvature (hence the naming ``Hessian free''). 
Assuming again $\tau(t)=\tau$ is constant, it uses a data-dependent metric $M_Z(t,\theta,\dot\theta) = H_Z(\dot \theta + \tau \theta)^\dagger$, where $A^{\dag}$ denotes the pseudo-inverse of $A$ and where (for the mean square loss function) $H_Z(\xi) \coloneqq \frac{1}{n} \sum_i  \partial_1 g(\xi,x_i)^\top \partial_1 g(\xi,x_i)$ is a proxy for the Hessian that captures the curvature of the loss. Here  $\partial_1 g$ is the differential of $g$ with respect to its first variable.

\paragraph{Running examples.}
We consider several examples.
\textbf{Principal component analysis (PCA)} corresponds to linear neural networks ($\sigma=\id$) with Euclidean geometry to perform dimensionality reduction via matrix factorization. 
\textbf{Multilayer Perceptrons (MLP)} use $\sigma=$ReLU and Euclidean geometry. 
\textbf{Non-negative matrix factorization (NMF)}  uses a linear network (i.e. $\sigma=$Id) with mirror geometry {(for the Shannon entropy)} to impose positivity on the factors~\cite{lee1999learning}.  
\textbf{Input Convex Neural Networks (ICNN)}~\cite{amos2017input} use 
a hybrid Euclidean/Mirror {(for the Shannon entropy)} geometry with $\sigma=$ReLU to impose positivity on some weights, to represent convex functions. It finds applications in implicit deep learning~\cite{amos2017input} and to compute optimal transport in high dimension~\cite{makkuva2020optimal}.

\begin{minipage}{.54\linewidth}
This table indicates the number of conservation laws that we characterize (for two layers and $r$ hidden neurons) for {\color{blue} gradient flow ($\tau(t)=\infty$)} and for {\color{red} momentum flows ($\tau(t) < +\infty$)}. 
\end{minipage}
\begin{minipage}{.45\linewidth}
\includegraphics[width=\linewidth]{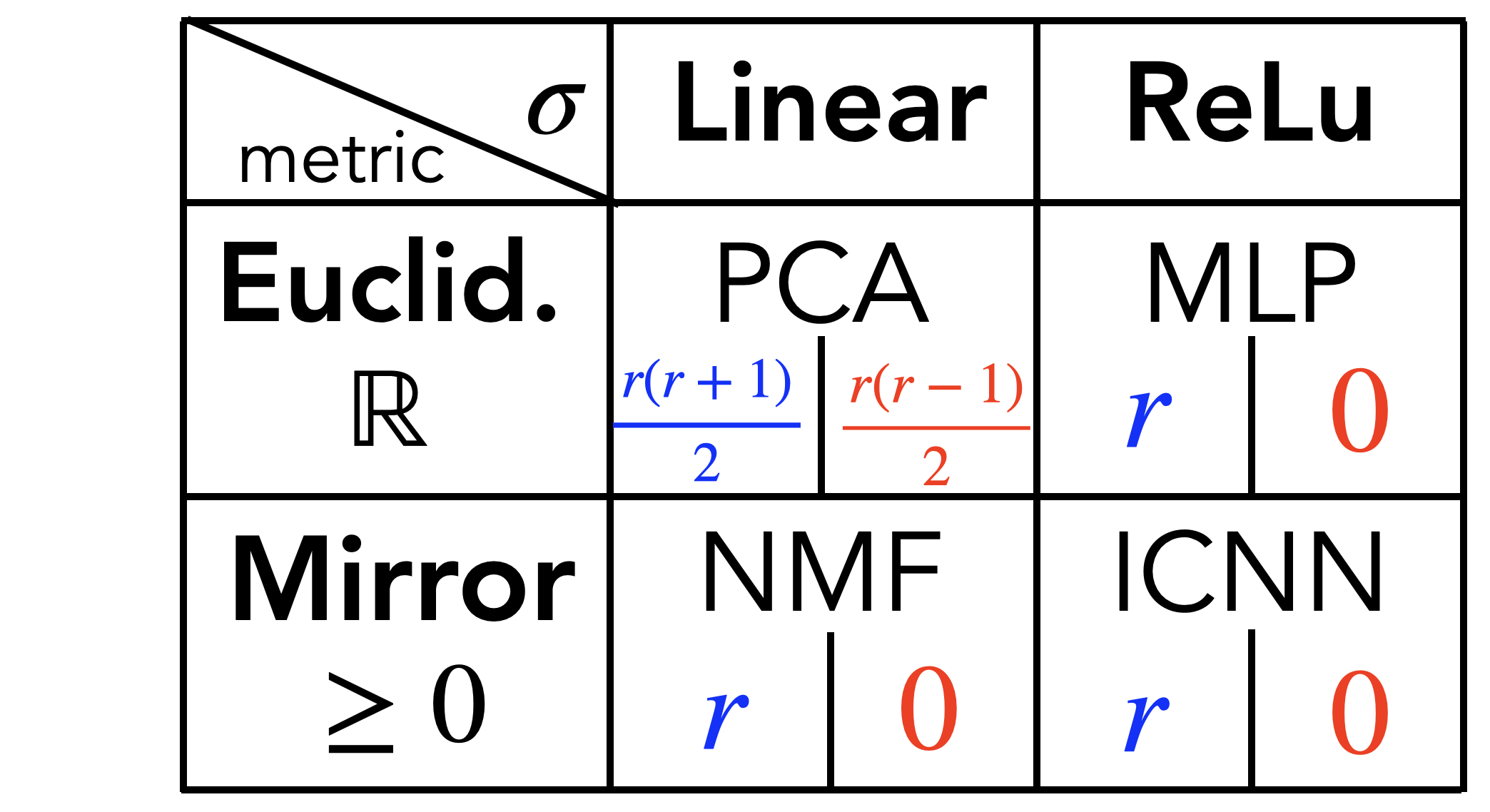}
\end{minipage}

{For the example of PCA, the information presented in the table corresponds to the case where $r$ is small enough, and the general case is fully addressed in \Cref{prop:lossnb}.}
\subsection{Conserved functions}
\vspace{-0.2cm} 
In the GF (resp. MF -- momentum flow) scenario, a function $h(t, \x)$ (resp. $h(t, \x, \dot \x)$) is conserved if for each solution\footnote{The existence of such a solution is discussed in \Cref{def:conserved_through_flow}.} 
 $\x(\cdot)$ to the ODE \eqref{gradientflow} (resp. \eqref{momentumflow}) with arbitrary initialization, the quantity $h(t, \x(t))$ (resp. $h(t, \x(t), \dot \x(t))$) remains constant in time.

\subsubsection{Time-dependence: GF {\em vs} MF}
\vspace{-0.2cm}
We postpone the formal definition of conserved functions and their characterization to first highlight an important fact (See \Cref{appendix:structurethm} for a proof): in the momentum setting conservation laws can depend non-trivially on time. This is in sharp contrast with the gradient flow case.
 
\begin{restatable}[Structure theorem]{theorem}{structurethm} \label{thm:structure}
Let $h(t, \x)$ (resp. $h(t, \x, \dot \x)$) be a conserved function for the ODE \eqref{gradientflow} (resp. \eqref{momentumflow} with $\tau(t) = \tau$) when its right-hand side is zero. For all $t$ and $\x$, one has $h(t, \x) = h(0, \x)$ (resp. $h(t, \x, \dot \x) = {H}(\x + \frac{\dot \x}{\tau}, \dot \x \exp(\tau t))$, where ${H}(a, b) \coloneqq h(0, a-\frac{b}{\tau}, b)$). In particular, conserved functions in the MF scenario can be expressed with $2D$ variables (instead of $2D +1$).
\end{restatable}
Conservation laws (to be soon formally defined) are notably conserved functions of the ODE when the right-hand side is zero, hence the above theorem directly applies.
\subsubsection{Formal definition via phase-space lifting} \label{section:Phase-spacelifting}
\vspace{-0.2cm}
To formally define conserved functions in the momentum case, 
we generalize the related notions from \cite{marcotte2023abide} to encompass dependencies on time $t$ and velocity $\dot{\x}$. 

\textbf{Notations} Given an open subset $\ouv \subseteq \Theta \subseteq \RD$, we denote $\bigouv \coloneqq \R \times \ouv \times \RD \subseteq \R^{2D+1}$. In particular $\tilde{\Theta} \coloneqq \R \times \Theta \times \RD$.
\begin{definition}[Conservation through a flow]  \label{def:conserved_through_flow}
Consider an open subset $\Omega \subseteq \RD $ and a function $F \in \mathcal{C}^1(\bigouv, \R^{2D} )$. 
By the Cauchy-Lipschitz theorem, for each initial condition $\texttt{init} \coloneqq (t_{\text{init}}, \omega_{\text{init}}) \in \bigouv$, there exists a unique maximal solution $t \in \left(t_{\text{init}} -  \eta_{\text{init}}, t_{\text{init}} +  \eta_{\text{init}}\right) \mapsto \omega(t,\texttt{init})$ of the ODE $ \dot \omega(t) = F(t, \omega(t))$ with $\omega(t_{\text{init}}) = \omega_{\text{init}}$. 
A function $h: \bigouv  \subseteq \R^{2D +1} \to \RR$ is {\em conserved on $\ouv$ through the flow $F$}  if $h(t, \omega(t,\texttt{init}))=h(t_{\text{init}}, \omega_{\text{init}})$ for each choice of $\texttt{init} \in \bigouv$  and every $t \in \left(t_{\text{init}} -  \eta_{\text{init}}, t_{\text{init}} +  \eta_{\text{init}}\right)$.
\end{definition}
Rewrite \eqref{momentumflow} as $\dot{\omega}(t) = F(t,\omega(t))$ with $\omega \coloneqq (\x,\dot{\x})$ and
\begin{equation}\label{eq:DefFZ}
    F_Z\left(t, \begin{pmatrix}
   {\theta} \\ \dot{\theta}
 \end{pmatrix}\right)\coloneqq \begin{pmatrix}
    \dot{\theta} \\ -\tau(t) \dot \x
     -M_Z(t, \x, \dot{\x}) \nabla \mathcal{E}_Z(\theta) \end{pmatrix}.
     \end{equation}
     With this expression, \Cref{def:conserved_through_flow} can be specialized.
\begin{definition}[Conservation during the flow \eqref{momentumflow} with a given dataset]
Given an open subset $\ouv \subseteq \Theta \subseteq \R^D$ and a dataset $Z$ such that $(t, \x, \dot \x) \mapsto M_Z(t, \x, \dot \x) \nabla \mathcal{E}_Z (\x) \in \mathcal{C}^{1}(\bigouv, \R^D)$, a function $h: \bigouv \subseteq \R^{2D +1} \to \RR$ is {\em conserved on $\ouv$ during the flow} \eqref{momentumflow} if 
it is conserved through $F_Z$.
\end{definition}
The next definition allows us to study which functions are conserved during {\em ``all'' flows} defined by the ODE \eqref{momentumflow}. The smoothness assumptions will enable simpler characterizations of such functions in due time.
\begin{definition}[Conservation during the flow \eqref{momentumflow} with ``any'' dataset]
  Consider an open subset $\Omega \subset \Theta$ and a loss $\ell(z, y)$ such that $\ell(\cdot, y)$ is $\mathcal{C}^2$-differentiable for all $y \in \mathcal{Y}$. 
  A function $h:  \bigouv \subseteq \R^{2D+1} \to \RR$ is {\em conserved on $\ouv$ during the flow \eqref{momentumflow} for any data set} if, for each data set $Z$ {\em such that} $g(\cdot, x_i) \in \mathcal{C}^{2}(\ouv, \R)$ for each $i$ and $(t, \x, \dot \x) \mapsto M_Z(t, \x, \dot \x)  \in \mathcal{C}^{1}(\bigouv, \RD)$, the function $h$ is conserved on $\ouv$ during the flow \eqref{momentumflow}. 
\end{definition}
While the above definitions are local, we are rather interested in functions defined {\em for the whole parameter space} $\Theta$, hence the following definition which mimics its equivalent in the gradient flow case \cite{marcotte2023abide}. 
\begin{definition} \label{def:locally_conserved_any}
    A function $h: \tilde{\Theta} \mapsto \R$ is {\em locally conserved during the flow \eqref{momentumflow} on $\Theta$ for any data set} if for each open subset $\ouv \subseteq \Theta$, $h$ is conserved on $\ouv$ for any data set.   
\end{definition}
\begin{example} \label{example:firsttriviallaw}
As a first simple example, consider
a two-layer {\em linear} neural network in dimension 1 (both for the input and output), with two hidden neurons. 
In that case, the parameter is $\x = (u_1,u_2, v_1, v_2) \subseteq \R^4 = \Theta$ and the model writes $g(\x,x) = (u_1v_1 + u_2v_2)x$.
Computing the derivative of $t \mapsto h(t,\x(t),\dot{\x}(t))$, where $h(t, \x, \dot \x) \coloneqq u_1 \dot u_2 - \dot u_1 u_2 + v_1 \dot v_2 - \dot v_1 v_2$, one can directly check that it vanishes, hence $h$ is {\em locally conserved on $\tilde{\Theta} = \R^{9}$ for any data set during the flow \eqref{momentumflow} for $M= \mId{4}$ and $\tau = 0$}.
\end{example}
A characterization of $\mathcal{C}^1$ 
conserved functions 
(see proof in \Cref{appendix:prop:orthogonality})
is the ``orthogonality'' of their gradients to an associated vector field. This is the momentum analog to a similar property for gradient flows \cite{marcotte2023abide}.
\begin{restatable}[Smooth functions conserved through a given flow] {proposition}{proporthogonality} \label{prop:orthogonality}
Given $F \in \mathcal{C}^1(\bigouv, \R^{2D})$, a function $h \in \mathcal{C}^1(\bigouv, \mathbb{R})$ is conserved through the flow induced by $F$ 
if and only if $\langle \nabla h (\z), (1, F(\z)^\top)^\top\rangle = 0$ for all $\z \in \bigouv$.
\end{restatable}
The following characterization is proved in \Cref{appendix:tracerewritten}.
\begin{restatable}{proposition}{proptracerewritten} \label{proptracerewritten}
    A function $h: \tilde{\Theta} \mapsto \R \in \mathcal{C}^1(\tilde{\Theta}, \R)$ is locally conserved on $\Theta$ for any data set if and only if $\nabla h(\z) \perp W_\z^{\mathtt{mom}} $ for all $\z \in \tilde{\Theta}$,
where for all $\z\coloneqq(t, \x, \dot \x) \in \tilde{\Theta}$:
    $$
W_\z^{\mathtt{mom}} \coloneqq \underset{Z \in \mathcal{Z}_\x}{\linspan} \left\{  (1, F_Z(\z)^\top)^\top\right\}, 
$$
with $\mathcal{Z}_\x$ the set of all data sets $Z=(x_i, y_i)_i$ such that there exists a neighborhood ${\ouv}$ of $\x$ such that for all $i$, $g(\cdot,x_i)\in \mathcal{C}^{2}(\ouv, \R^n)$ and such that 
$M_Z(\cdot) \in \mathcal{C}^{1}(\tilde{\ouv}, \R^{D \times D})$.
\end{restatable}
The subspace $W_\z^{\mathtt{mom}} \subseteq \R^{2D+1}$ (that characterizes conserved functions for a momentum flow) is linked with its counterpart in a Euclidean GF dynamic \citep{marcotte2023abide}, a subspace denoted $W_{\x} \subseteq \RD$ where $\z = (t, \x, \dot{\x})$. The space $W_\x$ characterizes conserved functions for a gradient flow in an Euclidean setting, and its definition is recalled in \Cref{appendix:link_MG} where the following proposition is proved. 

\begin{restatable}{proposition}{linkMGflow} \label{prop:link_mom_grad}
Assume that $M_Z(t,\x,\dot{\x}) = M(t,\x,\dot{\x})$ does not depend on the data set $Z$ and that $M(\cdot) \in \mathcal{C}^1(\tilde{\Theta}, \R^{D \times D})$. Assume  that for each $y \in \mathcal{Y}$ the loss $\ell(z, y)$ is $\mathcal{C}^2$-differentiable with respect to $z \in \R^n$ and that for each $\x \in {\Theta}$, there exists a training dataset $Z=(x_i, y_i)_i$ such that $\nabla \mathcal{E}_Z(\x) = 0$ and such that for all $i$,  $\x \mapsto 
    g(\x,x_i)$ is $\mathcal{C}^{2}$ in a neighborhood of $\x$.
Then, for each $\z \coloneqq (t, \x, \dot{\x}) \in \tilde{\Theta}$:
\begin{equation} \label{eq:eqprop2.9}
    W_\z^{\mathtt{mom}} = \R 
    \smallvec{1 \\ \dot{\theta} \\ -\tau(t)  \dot{\theta}}
    + 
    \smallvec{ 0 \\ 0 \\ M(t, \x, \dot \x) W_{\x} }.
\end{equation}
\end{restatable}
When the subspace $W_\x$ for Euclidean gradient flows is known \cite{marcotte2023abide}, the above link allows us to leverage this knowledge in {\em non-Euclidean, momentum} flow scenarios.
Extensions to matrices $M_Z$ that {\em depend} on the dataset $Z$, e.g. with natural gradient metrics, are briefly discussed (in the gradient case) in \Cref{section:naturalgradient}.
\begin{remark}
    The assumption on the loss $\ell$ 
    in \Cref{prop:link_mom_grad} holds for classical losses, see \Cref{appendix:link_MG}. 
\end{remark}
\begin{example}\label{example:firsttriviallawbis}
  {
Revisiting~\Cref{example:firsttriviallaw}, {we know that} 
  $h(t, \x, \dot \x) := u_1 \dot u_2 - \dot u_1 u_2 + v_1 \dot v_2 - \dot v_1 v_2$
  is conserved for any data set (with $M = \mId{4})$ during \eqref{momentumflow}, {hence by~\Cref{proptracerewritten}} one has for each $\z = (t, \x, \dot \x)$: $\nabla h(\z) \perp W^{\mathtt{mom}}_\z$. By \eqref{eq:eqprop2.9}, one has in particular  $\nabla h(\z) = (0, \dot u_2, -\dot u_1, \dot v_2, -\dot v_1, -u_2, u_1, -v_2, v_1)^\top \perp (1, \dot u_1, \dot u_2, \dot v_1, \dot v_2, 0,\ldots, 0) = \smallvec{1 \\ \dot{\theta} \\ 0}$ (remember that $\tau = 0$). {This will be further elaborated on in \Cref{example:firsttriviallawter}.}
  }
    \end{example}
 The analog of \Cref{prop:link_mom_grad} for gradient flows was established in \cite{marcotte2023abide} in the Euclidean case, and is naturally extended to the non-Euclidean case.
\begin{proposition}[Locally conserved function for any data set for \eqref{gradientflow}] \label{proplinkGFnonEuclidean}
Assume $M_Z(\cdot) \in \mathcal{C}^{1}(\Theta, \RD)$.
     A function $h:\Theta \mapsto \R \in \mathcal{C}^1(\Theta, \R)$ is locally conserved on $\Theta$ for any data set during the flow \eqref{gradientflow} if and only if $\nabla h(\x) \perp W_\x^{\texttt{grad}} \coloneqq \linspan \{ M_Z(\x) \nabla \mathcal{E}_Z (\x) : Z  \in \mathcal{Z}_\x'\}$, with $\mathcal{Z}_\x'$ the set of all data sets $Z=(x_i, y_i)_i$ such that there exists a neighborhood ${\ouv}$ of $\x$ such that for all $i$, $g(\cdot,x_i)\in \mathcal{C}^{2}(\ouv, \R^n)$.
When $M_Z$ does not depend on $Z$, we simple have $W_\x^{\texttt{grad}} = M(\x) W_\x$.
\end{proposition}

\subsection{From conserved functions to conservation laws}
To provide an algorithmic procedure to determine these functions, \cite{marcotte2023abide} makes the fundamental hypothesis that the \textbf{model} $g(\theta, x)$ 
can be (locally) factored via a \textbf{reparametrization} $\phi$ as $
f(\phi(\x),x)$. They require that the model $g(\theta, x)$ satisfies the following central assumption.

\begin{assumption}[Local reparameterization] \label{as:main_assumption}
There exists a dimension $d$ and a function $\phi \in \mathcal{C}^2(\Theta,\Rd)$ such that: for each parameter $\theta_0$ in the open set $\Theta \subseteq \RD$, for each $x \in \mathcal{X}$ such that $\x \mapsto 
    g(\x,x)$ is $\mathcal{C}^{2}$ in a neighborhood of $\x_0$, there is a neighborhood $\ouv$ of $\theta_0$ and a function $f(\cdot, x) \in \mathcal{C}^2(\phi(\ouv), \R^n)$ such that
\begin{equation}
    \label{eq:elr-general}
    \vspace{-0.2cm}
      \forall \x \in \ouv, \quad  
     g(\theta,x) = 
     f(\phi(\theta), x).
\end{equation}
\end{assumption} 

Moreover, \cite{marcotte2023abide} shows that such reparametrizations exist for linear and layered ReLU neural networks of any depth. They are respectively denoted $\phi_{\mathtt{Lin}}(\cdot)$ and $\phi_{\mathtt{ReLU}}(\cdot)$, and detailed in \citep[Examples 2.10, 2.11 and Appendix C]{marcotte2023abide}. We recall the expression of such reparametrizations in the two-layer case.

\begin{example}(Factorization for two-layer {\em linear} neural networks) \label{ex:param-linear}
In the two-layer case, with $r$ neurons, denoting $\x=(U,V) \in \RR^{n \times r} \times \RR^{m \times r}$ (so that $D=(n+m)r$), we can factorize  $g(\x,x) \coloneqq U V^\top x$ by the reparametrization 
$\phi_{\mathtt{Lin}}(\x) \coloneqq U V^\top \in \R^{n \times m}$ (identified with $\Rd$, $d=mn$). 
\end{example}

\begin{example}[Factorization for two-layer ReLU networks] \label{ex:param-ReLU} 
Consider $g(\x,x) = \big( \sum_{j=1}^{r} u_{k, j} \sigma(\langle v_j, x \rangle + b_j)+ c_k\big)_{k = 1}^{n}
$, with $\sigma(t) \coloneqq \max(t,0)$ the ReLU activation function and $v_j \in \R^m$, $u_{k,j} \in \R$, $b_j, c_k \in \R$. 
Then, denoting $\x = (U,V, b, c)$ with $U = (u_{k,j})_{k, j} =: (u_1, \cdots, u_r) \in \R^{n \times r}$, $V = (v_1, \cdots, v_r) \in \R^{m \times r}$, $b = (b_1, \cdots, b_r)^\top \in \R^r$ and $c = (c_1, \cdots, c_n) \in \R^n$  (so that $D = (n+m+1)r + n$), we can locally factorize $g(\x,x)$ by the reparametrization:
$\phi_{\mathtt{ReLU}}(\x) =  ((u_j v_j^\top, u_j b_j )_{j=1}^r, c)$. In particular, in the case without bias ($(b, c) = (0, 0)$), the reparametrization is defined by
$\phi_{\mathtt{ReLU}}(\x) = (\phi_j)_{j=1}^r$ where $\phi_j = \phi_j(\x) \coloneqq u_j v_j^\top \in \R^{n \times m}$ (here $d =rmn$): the reparametrization $\phi_{\mathtt{ReLU}}(\x)$ contains $r$ matrices of size $m \times n$ (each of rank at most one).
\end{example}
Thanks to this reparametrization and under a mild assumption on the loss $\ell$, \citet[Theorem 2.14]{marcotte2023abide} show that for linear neural networks of any depth (resp. for two-layer ReLU neural networks), the functions that are locally conserved for any data set in a Euclidean gradient flow dynamic are entirely characterized by the {\em trace} $\mathcal{W}(\x) \coloneqq \linspan \{w(\x): w(\cdot) \in \mathcal{W}\} \subseteq \RD$ of a finite-dimensional linear space of functions, $\mathcal{W}$, determined by $\phi_{\mathtt{Lin}}$ (resp. $\phi_{\mathtt{ReLU}}$). For these cases, they show that for all $\x \in \Theta$, the space $W_\x$ defined in \Cref{appendix:link_MG} satisfies $W_{\x} = \W(\x)$ with
\eq{ \label{eq:W_phi_grad}
\W \coloneqq \linspan \{\nabla \phi_1(\cdot), \cdots, \nabla \phi_{d}(\cdot) \},
}
where the components  $\phi_i: \RD \to \R$, $1 \leq i \leq d$ of $\phi$ are assumed to satisfy  $\nabla \phi_i \in \mathcal{C}^1(\Theta, \RD)$.

Assuming that $M$ does not depend on $Z$, in
the following theorem, we combine 
 \citep[Theorem 2.14]{marcotte2023abide} and \Cref{prop:link_mom_grad} 
to show that, under an assumption on the loss $\ell$, 
$W_\z^{\mathtt{mom}}\subseteq \RD$ is the trace of 
$\mathcal{W}_\phi^{\mathtt{mom}} \coloneqq 
    \linspan \{\chi_i(\cdot): 0 \leq i \leq d\} \subseteq \mathcal{C}^1(\widetilde{\Theta}, \RD)$ where
\begin{equation} \label{eq:v-phi}
\chi_0(\z) \coloneqq 
    \smallvec{ 1 \\  \dot{\x} \\ -\tau (t) \dot{\x} }; 
\chi_i(\z) \coloneqq 
    \smallvec{ 0  \\ 0 \\ M(t, \x, \dot \x) \nabla \phi_i(\x) }, i \geq 1,
\end{equation}
with $\z = (t,\x,\dot \x)$.
Similarly the subspace $W_\x^{\texttt{grad}}$ of \Cref{proplinkGFnonEuclidean} is the trace of 
\begin{equation} \label{eq:v-phi-GF}
\vspace{-0.22cm}
    \mathcal{W}_\phi^{\texttt{grad}} \coloneqq 
   {\linspan}_i \{   
M(\cdot ) \nabla \phi_i(\cdot)\} \subseteq \mathcal{C}^1({\Theta}, \RD).
\end{equation}

\begin{theorem} \label{theorem:reformulation_pb}
Assume that the loss $(z,y) \mapsto\ell(z,y)$ is $\mathcal{C}^2$-differentiable with respect to $z \in \R^n$ for each $y \in \mathcal{Y}$, and that it satisfies the condition:
    \eq{ \label{eq:condition_loss}
    \underset{y \in \mathcal{Y}}{\linspan}\{\nabla_z \ell (z, y) \}= \R^n, \forall z \in \R^n.
    }
 Assume that $M(\cdot) \in \mathcal{C}^1(\tilde{\Theta}, \R^{D \times D})$ (resp. $M(\cdot) \in \mathcal{C}^1({\Theta}, \R^{D \times D})$). 
 Then, for linear neural networks, 
 one has for all $\z \in \tilde{\Theta}$ with $ \Theta \coloneqq \RD$ (resp. for all $\x \in \Theta$:
    \begin{equation}
        W_\z^{\mathtt{mom}} = \mathcal{W}_{\phi_{\mathtt{Lin}}}^{\mathtt{mom}}(\z) \quad \text{(resp. } W_\z^{\mathtt{mom}} = \mathcal{W}_{\phi_{\mathtt{Lin}}}^{\texttt{grad}}(\x)).
    \end{equation}
    The same holds for two-layer ReLU networks with $\phi_{\mathtt{ReLU}}$ from \Cref{ex:param-ReLU} and $\Theta$ the (open) set of all parameters $\x$ such that hidden neurons are associated to pairwise distinct  ``hyperplanes'' (see \citep[Theorem 2.8]{marcotte2023abide}).
\end{theorem}
Condition~\eqref{eq:condition_loss} is the same as in \cite{marcotte2023abide} and is satisfied for standard losses such as the quadratic loss.
\Cref{theorem:reformulation_pb} motivates the following definition.
\begin{definition}[Conservation {\em law} for \eqref{momentumflow}] \label{defconservationGF}
    A real-valued function $h \in \mathcal{C}^1(\bigouv, \mathbb{R})$ is a \emph{conservation law of 
    $\phi$} for \eqref{momentumflow} if for all $\z \coloneqq (t, \x, \dot{\x}) \in \bigouv$,  $\nabla h(\z) \perp \W^{\mathtt{mom}} (\z)$.
\end{definition}
Combining \Cref{proptracerewritten} and \Cref{theorem:reformulation_pb}, the functions $h \in \mathcal{C}^1(\bigouv, \mathbb{R}) $ that are locally conserved on $\Theta$ for any data set are exactly the conservation laws of $\phi_{\mathtt{Lin}}$ (resp. of $\phi_{\mathtt{ReLU}}$ for linear (resp. ReLU) two-layer networks).
\begin{example}\label{example:firsttriviallawter}
 Revisiting again~\Cref{example:firsttriviallaw}, 
    the model $g$ is factorized by 
    $
    \phi_\mathtt{Lin}(\x) \coloneqq (u_1, u_2) \times (v_1, v_2)^\top = u_1v_1 + u_2v_2 \in \mathbb{R}$ with $\x \coloneqq (u, v) \in \Theta=\R^4$.  
    We know that $h(t, \x, \dot \x) \coloneqq u_1 \dot u_2 - \dot u_1 u_2 + v_1 \dot v_2 - \dot v_1 v_2$ is locally conserved during \eqref{momentumflow} (with $M=\mId{4}$) on $\tilde{\Theta}$ for any data set and indeed: {by \Cref{example:firsttriviallawbis}, we already have $\nabla h(\alpha) \perp \chi_0(\alpha)$ for all $\alpha \coloneqq (t, \x, \dot \x)$ {(remember that $\tau=0$)} and we also have}
    $\nabla h(\alpha) = (0, \dot u_2, -\dot u_1, \dot v_2, -\dot v_1, -u_2, u_1, -v_2, v_1)^\top \perp (0,\ldots, 0, v_1, v_2, u_1, u_2)^\top = (0,\ldots, 0, \nabla \phi(\x)^\top)^\top = {\chi_1(\alpha)}$. 
    {Thus $\nabla h(\z) \perp \W^{\mathtt{mom}} (\z)$: $h$ is a conservation law of~$\phi_{\mathtt{Lin}}$.}
\end{example}

Again, we generalize from \cite{marcotte2023abide} the notion of conservation laws for $\phi$ for a non-Euclidean GF, when $M$ does not depend on the data set $Z$.

\begin{definition}[Conservation law for \eqref{gradientflow}]
\label{CL:nonEuclideanGF}
     A function $h:\Theta \mapsto \R \in \mathcal{C}^1(\Theta, \R)$ is a conservation law of $\phi$ for \eqref{gradientflow} if 
     $\nabla h(\x) \perp   \mathcal{W}_{\phi}^{\texttt{grad}}(\x)$.
\end{definition} 
Conservation laws are known for the linear case $\phi_{\mathtt{Lin}}$ \cite{Arora18a, Arora18b}
and the ReLu case $\phi_{\mathtt{ReLu}}$ \cite{Du}.
\begin{example}[Conservation laws for linear and ReLu neural networks in Euclidean GF scenario] \label{conservation}
If $\x \coloneqq(U_1, \cdots, U_q)$ satisfies \eqref{gradientflow} with $M_Z(\theta) = \mId{D}$, then for each $i = 1, \cdots, q-1$ the function  
$\x \mapsto  U_{i}^\top U_{i} - U_{i+1} U_{i+1}^\top$ (resp. the function $\x \mapsto  \text{diag} \left(U_{i}^\top U_{i} - U_{i+1} U_{i+1}^\top  \right)$)
defines 
a set of conservation laws for $\phi_{\mathtt{Lin}}$ (resp. 
for $\phi_{\mathtt{ReLu}}$).
\end{example}

\section{Finding  Conservation Laws}
In this section, we propose a constructive way to build \textit{some} conservation laws. Then we explain how we can certify if there are conservation laws that are missing or not.

\subsection{Constructing Conservation Laws} \label{section:changeofvariable}
Conservation laws can be built using formal calculus, via the orthogonal relation that defines conservation laws (\Cref{CL:nonEuclideanGF} and \Cref{defconservationGF}), as done for Euclidean-GF in \cite{marcotte2023abide}.
One can also exploit invariances in the spirit of Noether's theorem \cite{Noether1918,jordan,Tanaka}.
\paragraph{Using formal calculus.}
By \Cref{defconservationGF} (resp. \Cref{CL:nonEuclideanGF}), $h$ is a conservation law for \eqref{momentumflow} (resp. \eqref{gradientflow}) if 
\eq{ \label{eq:PolyMF}  
\langle \nabla h (\z), \chi_i (\z) \rangle = 0,\ \forall \z = (t,\x,\dot \x) \in \tilde{\Theta},\ 0 \leq i \leq d,
}
where the $\chi_i$ are the $d+1$ vector fields defined in \eqref{eq:v-phi} that span the linear function space 
$W_\z^{\mathtt{mom}}$
(resp. if  
\eq{ \label{polynonE}
\langle \nabla h (\x), M(\x) \nabla \phi_i (\x) \rangle = 0,\ \forall \x \in \Theta,\ 1 \leq i \leq d.
}
One could seek conservation laws that belong to the finite-dimensional linear space of polynomials of a given degree, as in \cite{marcotte2023abide} for the Euclidean GF scenario. However, for MF scenarios,
\Cref{thm:structure} suggests that the conserved functions include a non-polynomial term in $t$, specifically $\exp(\tau t)$ when $\tau(t) = \tau \neq 0$.
Thus, in the MF case when  $\tau(t) = \tau$,  instead of considering conserved functions under the form $h(t, \x, \dot \x)$, we use a change a variable $s \coloneqq \exp(\tau t)$ to consider ``modified'' conserved functions $\tilde{h}$ that are under the form $
\tilde{h}(s, \x, \dot \x) \coloneqq h((\ln s)/\tau, \x, \dot \x)$.
Simple calculus yields
$
\nabla h(t, \x, \dot \x) =  
    \diag(\tau s, 1, \ldots 1) \nabla \tilde{h}(s, \x, \dot \x) ,
$
therefore, given any constant $c$ and any vector field $f: \ouv \times \RD \to \RD \times \RD$, we have: for every $(s, \x, \dot \x) \in \R_+^*\times \ouv \times \RD$:
$
\nabla h(t, \x, \dot \x) \perp 
\smallvec{ c \\ f(\x, \dot \x)  } 
\Leftrightarrow 
\nabla \tilde{h}( s, \x, \dot \x) \perp \smallvec{ \tau s c \\ f(\x, \dot \x)  }.
$
Recalling that we consider the case  $\tau(t) = \tau$, all the vector fields $\chi_i(t,\x,\dot{\x})$ from \eqref{eq:v-phi} involved in the definition of 
$W_\z^{\mathtt{mom}}$
do not depend on time, and their first coordinates are constant. 
Consequently, defining
$
\tilde{\chi}_i(s, \x, \dot \x) \coloneqq  
    \diag(\tau s, 1, \ldots 1) 
    \chi_i(0, \x, \dot \x),
$
we can solve\footnote{Concretely this is expressed as a linear system.} \eqref{eq:PolyMF} with the new vector fields $\tilde{\chi}_i$ and determine all polynomial $\tilde{h}$ of a given degree such that $\nabla \tilde{h}(s, \x, \dot \x) \perp \tilde{\chi}_i(s, \x, \dot \x)$ for all $i$. Finally, since $s=\exp(\tau t)$, we can determine all conservation laws that are polynomial in $\x, \dot \x$ and in $\exp(\tau t)$.

In the case where $\tau(t)$ is {\em not} constant, it remains possible to try directly to solve \eqref{eq:PolyMF} with the initial vector fields $\chi_i$ in $\W^{\mathtt{mom}} $ and see if there are polynomial conservation laws $h$ at a given degree. In particular, for a Nesterov flow with Euclidean metric, $\tau(t) = 3/t$, 
it turns out that there are polynomial conservation laws (see \Cref{section:linear}).
Our code to compute them is available at \url{https://github.com/sibyllema/Conservation_laws_ICML}.

\paragraph{Using invariances -- gradient flows.} 
%
For gradient flows, invariances of the cost \eqref{eq:erm} directly lead to conservation laws.

\begin{definition}[Invariant transformation on the cost \eqref{eq:erm}] 
A (one-parameter) transformation on an open set $\ouv \subseteq \Theta \subseteq \RD$ is a map $T: \R \times \ouv \to \RD$ such that $T(\cdot,\x)$ is differentiable for each $\x \in \ouv$ and $T(0,\cdot) = \id$. This transformation
leaves invariant the cost \eqref{eq:erm} if for all $\x \in \ouv$ and for all $\epsilon \geq 0$, $\mathcal{E}_Z(T(\epsilon,\x)) = \mathcal{E}_Z(T(0,\x)) = \mathcal{E}_Z(\x)$. When this holds, simple calculus yields 
for every $\x \in \ouv$:
\eq{ \label{eq:invarianceloss}
\Big\langle \nabla \mathcal{E}_Z(\x), \frac{\partial}{\partial \epsilon} T(\epsilon,\x) \Big|_{\epsilon =0} \Big\rangle= 0.
}
\end{definition}
We denote $\Delta_T (\cdot) \coloneqq \frac{\partial}{\partial \epsilon} T(\epsilon,\cdot) \Big|_{\epsilon =0}$.

\begin{example}[Fondamental example of linear transformation] \label{example:linear_transf}
Let us consider $\theta \coloneqq (U, V) \in \R^{n\times r} \times \R^{m\times r}$. 
We define $T^A$ the linear transformation where for all $\epsilon \geq 0$: 
\eq{ \label{linear_transf}
T^A(\epsilon,U, V) =
   \big(  U  \exp(\epsilon A), V \exp(-\epsilon A^\top)\big). 
}
Simple calculus yields $\Delta_{T^A }(U, V) =  (U A, -V A^\top)$. Considering $g(\x, x) \coloneqq U V^\top x$ and any $A \in \R^{r\times r}$, $T^A$ is an invariant transformation on \eqref{eq:erm}: for all $\epsilon$ and $\x$, 
$g(T^A(\epsilon,\x) , \cdot) = g(\x, \cdot)$ hence $\mathcal{E}_Z(T^A(\epsilon,\x)) = \mathcal{E}_Z(\x)$. Considering now $g(\x, x) \coloneqq U \sigma( V^\top x)$ with $\sigma$ the ReLU activation function, a similar reasoning shows that $T^A$ is an invariant transformation on \eqref{eq:erm} if $A \in \R^{r\times r}$ is \textit{diagonal}.
\end{example}

{\em From loss invariance to conservation laws for GF.} 
In the context of a gradient flow \eqref{gradientflow}, 
the consequence \eqref{eq:invarianceloss} of invariance rewrites as 
\eq{ \label{eq:invarianceandGF}
\langle M_Z \left(\x(t)\right)^{-1} \dot \x(t), \Delta_T (\x(t)) \rangle = 0
}
as soon as $M_Z(\x)$ is invertible for every $\x$. 

As a particular consequence, for gradient flows with linear networks in a Euclidean setting $(M \equiv \mId{D})$, since 
\eqref{eq:invarianceandGF}
holds for $T=T^A$ with {\em any} matrix $A \in \R^{r \times r}$, denoting $\langle M,N\rangle \coloneqq \mathtt{Tr}(M^\top N)$ we obtain $0 = \langle \dot U(t), U(t) A\rangle - \langle \dot V(t), V(t) A \rangle$ at each time and for any $A$. Specializing this to any \textit{symmetric} matrix, we obtain that $0 = \langle \dot U, U A\rangle - \langle \dot V, V A \rangle = 1/2 \frac{\mathrm{d}}{\mathrm{d}t} (\langle U, UA \rangle - \langle V, VA \rangle)$. Thus for every symmetric matrix $A$, $\langle U, UA \rangle - \langle V, VA \rangle$ is conserved, which coincides with all conservation laws in that case \cite{Arora18a, Arora18b, marcotte2023abide}. 

Similarly, for ReLU neural networks in a Euclidean setting (without bias for simplicity, this holds with bias too as detailed in the proof of \Cref{thm:ICNN}), by restricting ourselves to {\em elementary diagonal matrices} $A = E_{i, i}$ where $E_{i, i}$ is the one-hot matrix in $\R^{r \times r}$ with the $(i, i)$-th entry being $1$, $i =1, \cdots, r$, we obtain that for all $i$,  $0 = \langle \dot U_i(t), U_i(t)  \rangle - \langle \dot V_i(t), V_i(t)  \rangle = \tfrac{1}{2} \frac{\mathrm{d}}{\mathrm{d}t} (\langle U_i, U_i \rangle - \langle V_i, V_i\rangle)$. Thus for all $i$, $\| U_i \|^2 - \|V_i \|^2$ is conserved, recovering all conservation laws  \cite{Du, marcotte2023abide}.
\paragraph{Using invariances -- momentum flows.}
Invariances of the cost \eqref{eq:erm} are replaced in Noether's theorem by invariances of a {\em Lagrangian} compatible with the flow \eqref{momentumflow}. 

\begin{definition}[Lagrangian]
A Lagrangian \cite{Noether1918,jordan} is a function of three variables $\mathcal{L}\left(t, \x, v \right)$. The second order dynamical system \eqref{momentumflow} is compatible with this Lagrangian if for every solution $\x(t)$ of \eqref{momentumflow} the Euler-Lagrange equation of $\mathcal{L}$ holds:
 $$
    \frac{\mathrm{d}}{\mathrm{d}t} 
    \Big(\partial_{v} \mathcal{L} \big(t, \x(t), \dot{\x}(t)\big) \Big) = \partial_{{\x}} \mathcal{L}\big(t, \x(t), \dot{\x}(t) \big).
    $$
\end{definition}
\begin{example}[Euclidean momentum~\cite{jordan}] 
 In the case $M = \mId{D}$ and $\tau(t) = 
 \tau$ the map 
 \begin{equation} \label{eq:lagrangian}
\mathcal{L}\left(t, \x, v\right) \coloneqq \exp(\tau t) \big(  \tfrac{\|v\|^2}{2} - \mathcal{E}_Z(\x) \big)     
\end{equation}
 is a \textit{Lagrangian} of the second-order dynamical system \eqref{momentumflow}.  
 In the Euclidean case ($M = \mId{d}$) with a Nesterov flow ($\tau(t) =
 3/t$), 
 \begin{equation} \label{eq:lagrangian_nesterov}
\mathcal{L}\left(t, \x, v\right) \coloneqq \tfrac{t^3}{2} \big( \tfrac{\|v\|^2}{2} - \mathcal{E}_Z(\x) \big)     
\end{equation}
 is a \textit{Lagrangian} of the second-order dynamical system \eqref{momentumflow}. 
\end{example}

More generally, \cite{jordan} gives a Lagrangian associated to \eqref{momentumflow}, subject to: a) an hypothesis of an ``ideal scaling condition''  \citep[Equation 2.2b]{jordan}); and b) the assumption that the matrix $M$ is the inverse of the Hessian of an explicitly known metric $\psi$.

In general, a transformation leaving the cost invariant {\em does not necessarily} leave invariant the Lagrangian of the associated dynamic (e.g., revisiting \Cref{example:linear_transf}, $T^A$ leaves the cost invariant for {\em any} $A \in \R^{r \times r}$, yet only {\em skew} matrices $A$ leave the Lagrangian \eqref{eq:lagrangian} invariant). 
However, to obtain a conserved function with Noether theorem\footnote{The full version of Noether theorem with a non-zero right-hand side is recalled in \cite{Tanaka}.}, one needs to consider transformations leaving invariant the Lagrangian: in that case the function $\partial_{v} \mathcal{L} (t, \x(t), \dot{\x}(t)) \cdot \Delta_T (\x(t))$ is conserved. 
\begin{theorem}[Noether theorem]\label{thm:noether}
Let $(\x(t), \dot{\x}(t))$ be a solution of \eqref{momentumflow}. Then for each transformation $T$ that leaves $\mathcal{L}$ invariant,
 $
    \frac{\mathrm{d}}{\mathrm{d}t} \left(\partial_{v} \mathcal{L} (t, \x(t), \dot{\x}(t)) \cdot \Delta_T (\x(t)) \right) = 0.
    $
\end{theorem}

\subsection{Finding the number of conservation laws} \label{section:dim}
While being able to build conservation laws is beneficial, the question remains: how to ensure that we have derived ``all'' possible laws? This first requires to recall the definition of independent conserved functions \citep[Definition 2.18]{marcotte2023abide}, to avoid functional redundancies.
\begin{definition}
A family of $N$ functions $(h_1, \cdots, h_N)$ in $\mathcal{C}^1(\bigouv, \R)$ is  \textit{independent} if for all $\z \in \bigouv$ the vectors $(\nabla h_1(\z), \cdots, \nabla  h_N(\z))$ are linearly independent. 
\end{definition}

\citep[Theorem 3.3]{marcotte2023abide} link the number of independent conservation laws to the dimension of a space involving Lie algebras. Knowledge about Lie algebras is not mandatory in the main body of this paper, basics are recalled in \Cref{liealgebrabackground} (See \cite{marcotte2023abide}) to support the proofs. The space $\lie(
\mathcal{W}
)$ is \textit{the generated Lie algebra of $
\mathcal{W}$} (See \Cref{liealgebrabackground}), it is entirely characterized by $\mathcal{W}$.

\begin{theorem} \label{thm:algebrelie}
    If $\vdim (\lie(\mathcal{W}_{\phi}^\mathtt{mom})(\z))$ is locally constant then each $\z \in \bigouv \subseteq \R^{2D+1}$ admits a neighborhood $\tilde{\Omega'}$ such that there are $2D+1-\vdim (\lie(\mathcal{W}_{\phi}^\mathtt{mom})(\z))$ independent conservation laws of $\phi$ for \eqref{momentumflow} on $\tilde{\ouv'}$.
\end{theorem}
The same theorem holds for $\mathcal{W}_{\phi}^\texttt{grad}$ by replacing $2D+1$ by $D$.
Besides, given a finite set of vector fields, \citet{marcotte2023abide} provide a code (detailed in  Section 3.3) that computes the dimension of the trace of their generated Lie algebra: in their case, they apply it to the vector fields $(\nabla \phi_i(\cdot))_i$ that span $\W$ (as defined in \eqref{eq:W_phi_grad}). We can directly use this code with the finite set of vector fields that span $\mathcal{W}_{\phi}^\mathtt{mom}$ as defined in \eqref{eq:v-phi} (resp. on the fields $(M(\cdot) \nabla \phi_i (\cdot))_i$ that span $\mathcal{W}_{\phi}^\texttt{grad}$ for the non-Euclidean GF case, see \eqref{eq:v-phi-GF}).

By computing the number of independent conservation laws with this code in the Euclidean GF scenario, \citet{marcotte2023abide} 
(in their Section 3.3)
establish that the set of known conservation laws for linear and ReLu neural networks (recalled in \Cref{conservation}) is complete. In particular, they fully work out theoretically the $2$-layer case and show \citep[Proposition 4.2, Corollary 4.4]{marcotte2023abide}:
\begin{proposition} \label{nb_GF}
    If $(U; V)$ has full rank noted $\mathtt{rk}$, then in a neighborhood of $( U, V)$, the set of $\mathtt{rk}/2 \cdot (2r + 1-\mathtt{rk})$ independent conservation laws given by \Cref{conservation} is complete: there exists no other conservation law.
\end{proposition}
\section{Conservation for 
Neural Networks}
\label{section:loss}

We now exemplify our results in different settings.
We study the conservation laws for neural networks with $q$ layers, and either a linear or ReLU activation. We write $\x \coloneqq (U_1, \cdots, U_q)$ with $U_i \in \R^{n_{i-1} \times n_{i}}$ the weight matrices. We also consider the impact of the choice of metric.
The striking outcome of this analysis is that momentum flows radically change the structure of the conservation laws. There are fewer (or even none) conserved quantities when using momentum flows, and this phenomenon appears both for Euclidean and non-Euclidean geometries. 

\subsection{PCA: linear networks with Euclidean geometry} 
\label{section:linear}

The following theorem (proved in \Cref{appendix:preuveconservationeuclidienne}) gives a set of conserved functions when $g(\x,x) \coloneqq U_1\ldots U_qx$. Here
$\mathcal{A}_{n}$ denotes the set of skew-symmetric matrices in $\R^{n \times n}$.
\begin{restatable}{theorem}{conservedfunctionsEuclidean} \label{thm:conservedfunctionsEuclidean}
Consider the model $g(\x,x) \coloneqq U_1\ldots U_qx$.
For all $i = 1, \cdots q-1$ and for all $A \in \mathcal{A}_{n_i}$, the function
\begin{equation}\label{eq:NewLawPCA}
\exp\big( \textstyle\int^t \tau(s) \mathrm{d}s \big) \big( 
\langle \dot U_{i},  U_{i}A
   \rangle  +
   \langle 
    \dot U_{i+1},  A^\top U_{i+1}  \rangle  \big)
\end{equation}
is a conservation law for \eqref{momentumflow} with $M=\mId{D}$, where $\int^t \tau(s) \mathrm{d}s$ is a primitive of $\tau(\cdot)$.
Moreover, for each $i$ such that $n_{i-1} = n_{i+1}=1$ and all  $A \in \mathcal{A}_{n_i}$
the function
$    \exp\big( \textstyle\int^t \tau(s) \mathrm{d}s\big) \big( 
\langle \dot U_{i},   U_{i+1}^\top A
   \rangle  +
   \langle 
    \dot U_{i+1}^\top,  U_{i} A \rangle  \big)
$
is an additional conservation law.
\end{restatable}
In particular, for the heavy ball case $\tau(t)=\tau$, these functions exactly correspond to the ones obtained by solving \eqref{eq:PolyMF} with the change of variable explained in \Cref{section:changeofvariable}. For the Euclidean Nesterov case $\tau(t) = 3/t$, these functions are also directly obtained by solving \eqref{eq:PolyMF}, with a polynomial term $t^3$ in $t$. As discussed 
in \Cref{appendix:noether}, in the case $(n_{i-1}, n_{i+1})\neq (1, 1)$, the associated conserved functions can also be found using Noether theorem with $T^A$ (defined in \Cref{example:linear_transf}) for every skew matrix $A$, and for the case $(n_{i-1}, n_{i+1})=(1, 1)$, the new associated conserved functions can also be found using Noether theorem with a new linear transformation $T'^A$ (defined in \eqref{eq:newtransformation}).\\
\noindent
The above theorem gives a set of conserved functions. A priori, they are not all independent. The following proposition (proved in \Cref{appendix:counting}) gives the number of independent conserved functions given by \Cref{thm:conservedfunctionsEuclidean} in the case $q =2$. We rewrite $\x = (U, V)$ as a vertical matrix concatenation denoted $(U; V) \in \mathbb{R}^{(n+m) \times r}$.
\begin{restatable}{proposition}{countingtheconservedfunctions} \label{prop:counting}
Consider 
$\widetilde{\Theta}$ 
 the set of 
  $\alpha = (t,\x,\dot \x)$ such that $(\x,\dot\x) =(U; V; \dot U; \dot V) \in \R^{2(n+m) \times r}$ has full rank, and assume $(n, m) \neq (1, 1)$. Then if $2(n+m) \leq r$, \Cref{thm:conservedfunctionsEuclidean} gives exactly 
    $(n+m)[(r-2(n+m))+r-1] $
    independent conserved functions.
If $2(n+m) > r$, 
    \Cref{thm:conservedfunctionsEuclidean} gives exactly $r(r-1)/2$ independent conserved functions.
\end{restatable}
Now we want to establish if there are
other conservation laws independent from the ones obtained in \Cref{thm:conservedfunctionsEuclidean}. By using \Cref{thm:algebrelie}, we only need to determine the dimension of the trace of a Lie algebra characterized by $\mathcal{W}_{\phi_{\mathtt{Lin}}}^\mathtt{mom}$. We fully work out theoretically the case $q=2$ as detailed in the following theorem and show that there no other conservation laws. See \Cref{appendix:dimliealgebra} for a proof. 

\begin{restatable}{theorem}{dimliealgebra}\label{theorem:dimliealgebra}
Consider
$g(\x, x) \coloneqq U V^\top x$ with $(U, V) \in \R^{n \times r} \times \R^{m\times r}$. If $(U; V; \dot U; \dot V) \in \R^{2(n+m) \times r}$ has full rank noted $\mathtt{rk}$ then, in a neighborhood of $(t, U, V, \dot U, \dot V)$: 
 \begin{itemize}[nosep,wide, leftmargin=*] 
 \item 
    If $(n, m) \neq (1, 1)$, there are exactly 
    $\mathtt{rk}/2 \cdot \left( 2r-1-\mathtt{rk} \right)$ independent conservation laws of $\phi_{\mathtt{Lin}}$.
    \item 
    If $n= m =1$ and if $r \geq 4$, 
    there are exactly $4r - 6$ independent conservation laws of $\phi_{\mathtt{Lin}}$.
 \end{itemize}
\end{restatable}
For deeper cases, we computed $\dim [\lie (\mathcal{W}_{\phi_{\mathtt{Lin}}}^\mathtt{mom}){(\alpha)}]$
 using the method explained in \Cref{section:dim} with the vector fields that generate $W_{\phi_{\mathtt{Lin}}}^\mathtt{mom}$ on a sample of depths/widths of small size. This empirically confirmed (see \Cref{sec:sagemathEMF}) that there are no other conservation laws for deeper cases too.

Comparing the number of independent conservation laws in the GF scenario given in \Cref{nb_GF} with the one in the MF scenario given in the last \Cref{theorem:dimliealgebra}, we obtain as highlighted next a {\em loss of conservation} when transitioning from Euclidean gradient flows to the Euclidean momentum setting {and when $r \leq n+m$. Notice that this case includes the factorization by matrices {up to full rank and even in mildly over-parameterized regimes}.
By contrast, {in the more over-parameterized regime} $r > n+m$, we obtain {\em a gain in conservation}}. A proof can be found in \Cref{appendix:loss}.
\begin{restatable}{proposition}{lossnb}\label{prop:lossnb}
Consider $g(\x, x) \coloneqq U V^\top x$ with $(U, V) \in \R^{n \times r} \times \R^{m\times r}$. Assume both $(U; V) \in \R^{(n+m) \times r}$ and  $(U; V; \dot U; \dot V) \in \R^{2(n+m) \times r}$ have full rank 
and $(n, m) \neq (1, 1)$. Denote $N_{GF}$ (resp. $N_{MF}$) the number of independent conservation laws of $\phi_{\mathtt{Lin}}$ for \eqref{gradientflow} (resp. for \eqref{momentumflow}) in a neighborhood of $(U, V)$ (resp. of $(t, U, V,  \dot U, \dot V)$, then if $r \leq  n+m $:
$N_{GF}-N_{MF} 
=r > 0,
$
else:
$N_{GF}-N_{MF}  \leq 0$.
\end{restatable}

\subsection{MLP: ReLU Networks with Euclidean geometry} \label{section:ReLU}
For the case $q=2$ without bias, since $\phi_{\mathtt{ReLU}}(\x) = (\phi_j(\x))_{j=1}^r$ is decoupled into $r$ functions  $\phi_j(\x) \coloneqq u_j v_j^\top \in \R^{n \times m}$ each depending on a separate block of coordinates,  Jacobian matrices and Hessian matrices are block-diagonal. Therefore Lie brackets computations can be done separately for each block,  {using \Cref{theorem:dimliealgebra} with $r=1$ .}
We obtain that there is no conserved function for Euclidean momentum flow with a two-layer ReLU network without bias. 
For deeper ReLU networks including with bias, by computing on a sample of depths/widths of small size the number of conservation laws as explained in \Cref{section:dim} using the finite set of vector fields that generates $\mathcal{W}_{\phi_{\mathtt{ReLU}}}^\mathtt{mom}$, we obtain that there is no conservation law {(see details in \Cref{sec:sagemathEMF})}. 

\subsection{NMF: Linear Networks with Mirror geometry} 
\label{section:NMF}
Non-negative 
matrix factorization (NMF)
is an example of a two-layer linear neural network, and we use the mirror geometry {associated to the Shannon entropy potential} both on $U$ and $V$ to enforce non-negativity. 
 The following theorem (proved in \Cref{appendix:conservedfunctionsNMF}) gives a set of conserved functions in the gradient flow case \eqref{gradientflow}. We denote $1_{\ell} \in \R^{\ell}$ the vector with all coordinates equal to 1.
\begin{restatable}{theorem}{conservedfunctionsNMF} \label{thm:conservedfunctionsNMF}
Consider $g(\x, x) \coloneqq U V^\top x$ with $\x = (U, V) \in \R^{n \times r} \times \R^{m\times r}$. Then
\begin{equation}
\label{eq:NewLawNMF}
1_{n}^\top U - 1_{m}^\top V
\end{equation}
defines 
 $r$ independent conservation laws for \eqref{gradientflow} with $M(\x) \coloneqq \diag(\x)$.
\end{restatable}
These functions can be found 
by solving \eqref{polynonE} with formal calculus with the vector fields that generate $
\mathcal{W}_{\phi_{\mathtt{Lin}}}^\texttt{grad}$ in this non-Euclidean gradient flow case. 
As for the Euclidean GF case, these functions can also be linked to invariance on the cost as they coincide with the time integration of 
\eqref{eq:invarianceandGF}: $\langle \diag(\x)^{-1}  \dot \x, \Delta_{T^A }(\x) \rangle= 0$ when 
$A = E_{i, i}$ for each $i =1, \cdots, r$. 
By computing the number of independent conservation laws with the method explained in \Cref{section:dim}, we obtain that there are no other conservation laws (see \Cref{sec:sagemathNMF}). 
In contrast, by computing the number of independent conservation laws for the vector fields that generate $\W^\mathtt{mom}$ for the {\em non-Euclidean} momentum flow case, we found that there is no conservation law at all in that case. 
Once again there are fewer conservation laws for MF than GF.

\subsection{ICNN: mixing Mirror/Euclidean geometries} 
 \label{section:ICNN}
In this section, we only treat the two-layer ReLU case of ICNNs \cite{amos2017input}, corresponding to  $g(\x, x) = U \sigma(V^\top x + b^\top)$, with $\x = (U;V) \in \R^{(m+n)\times r}$ and $\sigma(t) \coloneqq \max(t,0)$.
We employ the mirror geometry {associated to the Shannon entropy potential} on $U$ to enforce non-negativity {($U \geq 0$)} and the Euclidean metric on $V$ and $b$. The following theorem (proved in \Cref{appendix:icnn}) gives a set of conservation laws for GF.

\begin{restatable}{theorem}{thmICNN} \label{thm:ICNN}
Consider $g(\x, x) = U \sigma(V^\top x + b^\top)$.
{Denote $U_j$ (resp. $V_j$ / $b_j$) the $j$-th column of $U$ (resp column of $V$ / entry of $b$).}
For all $j = 1, \cdots, r$,  the function
\begin{equation}\label{eq:NewLawICNN}
1_n^\top U_{j} - \frac{1}{2}\left( \| V_j\|^2 + b_j^2 \right)
\end{equation}
is a conservation law
for \eqref{gradientflow} with $M (U, V, b) \coloneqq \diag[(U,1_{m \times r},1_{1 \times r})].$
\end{restatable}
By computing the number of independent conservation laws with the method explained in \Cref{section:dim}, we obtain that there are no other conservation laws. In 
 contrast, by computing the number of independent conservation laws for the vector fields that generate 
 $\mathcal{W}_{\phi_{\mathtt{ReLU}}}^\mathtt{mom}$
 for the non-Euclidean momentum flow case, we found that there is no conservation law at all in that case (see \Cref{sec:sagemathICNN}). Thus, once again we observe a loss of conservation from GF to MF regimes.

\subsection{Natural gradient} \label{section:naturalgradient}
Our main result on conservation laws for momentum, \Cref{prop:link_mom_grad}, only holds when $M = M_{Z}$ is {\em independent} of the dataset $Z$. While we cannot apply it to the natural gradient setting, 
we can still conduct (a part of) our study for the non-Euclidean GF. 
We obtain that all conservation laws known for the Euclidean gradient flow case recalled in \Cref{conservation} are conservation laws for the Natural gradient flow. See \Cref{appendix:naturalgradient} for a proof.

\section*{Conclusion}
In this paper, we examined conservation laws for gradient or momentum flows within Euclidean as well as non-Euclidean geometries.  
One notable constraint of this theory is its limitation to continuous-time flow. Appendix~\ref{sec:numerics} studies numerically the impact of the time discretization on MLP and NMF problems. It also studies the impact of momentum on quantities preserved by GF. Understanding these approximate conservations is an important open problem. 
{We also anticipate that our theory will be adaptable to the study of the SGD case.}

\paragraph{Acknowledgement} 
The work of G. Peyré was supported by the European Research Council (ERC project NORIA) and the French government under management of Agence Nationale de la Recherche as part of the ``Investissements d’avenir'' program, reference ANR-19-P3IA-0001 (PRAIRIE 3IA Institute). The work of R. Gribonval was partially supported by the AllegroAssai ANR project ANR-19-CHIA-0009 and the SHARP ANR Project ANR-23-PEIA-0008 of the PEPR IA, funded in the framework of the France 2030 program.

\paragraph{Impact Statement.}
This paper presents work whose goal is to advance the field of Machine Learning. There are many potential societal consequences of our work, none of which we feel must be specifically highlighted here.

\bibliography{paper}

\begin{thebibliography}{33}
\providecommand{\natexlab}[1]{#1}
\providecommand{\url}[1]{\texttt{#1}}
\expandafter\ifx\csname urlstyle\endcsname\relax
  \providecommand{\doi}[1]{doi: #1}\else
  \providecommand{\doi}{doi: \begingroup \urlstyle{rm}\Url}\fi

\bibitem[Amari(1998)]{amari1998natural}
Amari, S.-I.
\newblock Natural gradient works efficiently in learning.
\newblock \emph{Neural computation}, 10\penalty0 (2):\penalty0 251--276, 1998.

\bibitem[Amos et~al.(2017)Amos, Xu, and Kolter]{amos2017input}
Amos, B., Xu, L., and Kolter, J.~Z.
\newblock Input convex neural networks.
\newblock In \emph{International Conference on Machine Learning}, pp.\  146--155. PMLR, 2017.

\bibitem[Arora et~al.(2018)Arora, Cohen, and Hazan]{Arora18b}
Arora, S., Cohen, N., and Hazan, E.
\newblock On the optimization of deep networks: Implicit acceleration by overparameterization.
\newblock In \emph{International Conference on Machine Learning}, pp.\  244--253. PMLR, 2018.

\bibitem[Arora et~al.(2019)Arora, Cohen, Golowich, and Hu]{Arora18a}
Arora, S., Cohen, N., Golowich, N., and Hu, W.
\newblock A convergence analysis of gradient descent for deep linear neural networks.
\newblock In \emph{International Conference on Learning Representations}, 2019.

\bibitem[Bah et~al.(2022)Bah, Rauhut, Terstiege, and Westdickenberg]{Bah}
Bah, B., Rauhut, H., Terstiege, U., and Westdickenberg, M.
\newblock Learning deep linear neural networks: Riemannian gradient flows and convergence to global minimizers.
\newblock \emph{Information and Inference: A Journal of the IMA}, 11\penalty0 (1):\penalty0 307--353, 2022.

\bibitem[Bonnard et~al.(2018)Bonnard, Chyba, and Rouot]{Bonnard}
Bonnard, B., Chyba, M., and Rouot, J.
\newblock \emph{{Geometric and Numerical Optimal Control - Application to Swimming at Low Reynolds Number and Magnetic Resonance Imaging}}.
\newblock SpringerBriefs in Mathematics. {Springer Int. Publishing}, 2018.

\bibitem[Bregman(1967)]{bregman1967relaxation}
Bregman, L.~M.
\newblock The relaxation method of finding the common point of convex sets and its application to the solution of problems in convex programming.
\newblock \emph{USSR computational mathematics and mathematical physics}, 7\penalty0 (3):\penalty0 200--217, 1967.

\bibitem[Bubeck et~al.(2015)]{bubeck2015convex}
Bubeck, S. et~al.
\newblock Convex optimization: Algorithms and complexity.
\newblock \emph{Foundations and Trends{\textregistered} in Machine Learning}, 8\penalty0 (3-4):\penalty0 231--357, 2015.

\bibitem[Chizat \& Bach(2020)Chizat and Bach]{chizat}
Chizat, L. and Bach, F.
\newblock Implicit bias of gradient descent for wide two-layer neural networks trained with the logistic loss.
\newblock In \emph{Conf. on Learning Theory}, pp.\  1305--1338. PMLR, 2020.

\bibitem[Du et~al.(2018)Du, Hu, and Lee]{Du}
Du, S.~S., Hu, W., and Lee, J.~D.
\newblock Algorithmic regularization in learning deep homogeneous models: Layers are automatically balanced.
\newblock \emph{Advances in Neural Information Processing Systems}, 31, 2018.

\bibitem[Głuch \& Urbanke(2021)Głuch and Urbanke]{głuch2021noether}
Głuch, G. and Urbanke, R.
\newblock Noether: The more things change, the more stay the same.
\newblock \emph{arXiv preprint arXiv:2104.05508}, 2021.

\bibitem[Ji \& Telgarsky(2019)Ji and Telgarsky]{Ji}
Ji, Z. and Telgarsky, M.
\newblock Gradient descent aligns the layers of deep linear networks.
\newblock In \emph{International Conference on Learning Representations}, 2019.

\bibitem[Kunin et~al.(2021)Kunin, Sagastuy-Brena, Ganguli, Yamins, and Tanaka]{Kunin}
Kunin, D., Sagastuy-Brena, J., Ganguli, S., Yamins, D.~L., and Tanaka, H.
\newblock Neural mechanics: Symmetry and broken conservation laws in deep learning dynamics.
\newblock In \emph{International Conference on Learning Representations}, 2021.

\bibitem[LeCun et~al.(2010)LeCun, Cortes, and Burges]{lecun2010mnist}
LeCun, Y., Cortes, C., and Burges, C.
\newblock Mnist handwritten digit database.
\newblock \emph{ATT Labs [Online]. Available: http://yann.lecun.com/exdb/mnist}, 2, 2010.

\bibitem[Lee \& Seung(1999)Lee and Seung]{lee1999learning}
Lee, D.~D. and Seung, H.~S.
\newblock Learning the parts of objects by non-negative matrix factorization.
\newblock \emph{Nature}, 401\penalty0 (6755):\penalty0 788--791, 1999.

\bibitem[Makkuva et~al.(2020)Makkuva, Taghvaei, Oh, and Lee]{makkuva2020optimal}
Makkuva, A., Taghvaei, A., Oh, S., and Lee, J.
\newblock Optimal transport mapping via input convex neural networks.
\newblock In \emph{International Conference on Machine Learning}, pp.\  6672--6681. PMLR, 2020.

\bibitem[Marcotte et~al.(2024)Marcotte, Gribonval, and Peyr{\'e}]{marcotte2023abide}
Marcotte, S., Gribonval, R., and Peyr{\'e}, G.
\newblock Abide by the law and follow the flow: Conservation laws for gradient flows.
\newblock \emph{Advances in Neural Information Processing Systems}, 36, 2024.

\bibitem[Martens(2010)]{martens2010deep}
Martens, J.
\newblock Deep learning via hessian-free optimization.
\newblock In \emph{Proceedings of the 27th International Conference on International Conference on Machine Learning}, pp.\  735--742, 2010.

\bibitem[Martens \& Grosse(2015)Martens and Grosse]{martens2015optimizing}
Martens, J. and Grosse, R.
\newblock Optimizing neural networks with kronecker-factored approximate curvature.
\newblock In \emph{International conference on machine learning}, pp.\  2408--2417. PMLR, 2015.

\bibitem[Martens \& Sutskever(2012)Martens and Sutskever]{martens2012training}
Martens, J. and Sutskever, I.
\newblock \emph{Training Deep and Recurrent Networks with Hessian-Free Optimization}.
\newblock Springer, 2012.

\bibitem[Min et~al.(2021)Min, Tarmoun, Vidal, and Mallada]{Min}
Min, H., Tarmoun, S., Vidal, R., and Mallada, E.
\newblock On the explicit role of initialization on the convergence and implicit bias of overparametrized linear networks.
\newblock In \emph{International Conference on Machine Learning}, pp.\  7760--7768. PMLR, 2021.

\bibitem[Nemirovskij \& Yudin(1983)Nemirovskij and Yudin]{nemirovskij1983problem}
Nemirovskij, A.~S. and Yudin, D.~B.
\newblock \emph{Problem complexity and method efficiency in optimization}.
\newblock Wiley-Interscience, 1983.

\bibitem[Nesterov(1983)]{nesterov1983method}
Nesterov, Y.~E.
\newblock A method of solving a convex programming problem with convergence rate o$\bigl(k^2 \bigr)$.
\newblock In \emph{Doklady Akademii Nauk}, volume 269, pp.\  543--547. Russian Academy of Sciences, 1983.

\bibitem[Noether(1918)]{Noether1918}
Noether, E.
\newblock Invariante variationsprobleme.
\newblock \emph{Nachrichten von der Gesellschaft der Wissenschaften zu Göttingen, Mathematisch-Physikalische Klasse}, 1918:\penalty0 235--257, 1918.

\bibitem[Polyak(1964)]{polyak1964some}
Polyak, B.~T.
\newblock Some methods of speeding up the convergence of iteration methods.
\newblock \emph{Ussr computational mathematics and mathematical physics}, 4\penalty0 (5):\penalty0 1--17, 1964.

\bibitem[Raskutti \& Mukherjee(2015)Raskutti and Mukherjee]{raskutti2015information}
Raskutti, G. and Mukherjee, S.
\newblock The information geometry of mirror descent.
\newblock \emph{IEEE Transactions on Information Theory}, 61\penalty0 (3):\penalty0 1451--1457, 2015.

\bibitem[Saxe et~al.(2013)Saxe, McClelland, and Ganguli]{Saxe}
Saxe, A.~M., McClelland, J.~L., and Ganguli, S.
\newblock Exact solutions to the nonlinear dynamics of learning in deep linear neural networks.
\newblock \emph{arXiv preprint arXiv:1312.6120}, 2013.

\bibitem[Su et~al.(2016)Su, Boyd, and Candes]{su2016differential}
Su, W., Boyd, S., and Candes, E.~J.
\newblock A differential equation for modeling nesterov's accelerated gradient method: Theory and insights.
\newblock \emph{Journal of Machine Learning Research}, 17\penalty0 (153):\penalty0 1--43, 2016.

\bibitem[Tanaka \& Kunin(2021)Tanaka and Kunin]{Tanaka}
Tanaka, H. and Kunin, D.
\newblock Noether’s learning dynamics: Role of symmetry breaking in neural networks.
\newblock \emph{Advances in Neural Information Processing Systems}, 34, 2021.

\bibitem[Tarmoun et~al.(2021)Tarmoun, Franca, Haeffele, and Vidal]{Tarmoun}
Tarmoun, S., Franca, G., Haeffele, B.~D., and Vidal, R.
\newblock Understanding the dynamics of gradient flow in overparameterized linear models.
\newblock In \emph{International Conference on Machine Learning}, pp.\  10153--10161. PMLR, 2021.

\bibitem[{The Sage Developers}(2022)]{sagemath}
{The Sage Developers}.
\newblock \emph{{S}ageMath, the {S}age {M}athematics {S}oftware {S}ystem ({V}ersion 9.7)}, 2022.
\newblock {\tt https://www.sagemath.org}.

\bibitem[Wibisono et~al.(2016)Wibisono, Wilson, and Jordan]{jordan}
Wibisono, A., Wilson, A.~C., and Jordan, M.~I.
\newblock A variational perspective on accelerated methods in optimization.
\newblock \emph{proceedings of the National Academy of Sciences}, 113\penalty0 (47):\penalty0 E7351--E7358, 2016.

\bibitem[Zhao et~al.(2023)Zhao, Ganev, Walters, Yu, and Dehmamy]{zhao}
Zhao, B., Ganev, I., Walters, R., Yu, R., and Dehmamy, N.
\newblock Symmetries, flat minima, and the conserved quantities of gradient flow.
\newblock In \emph{The Eleventh International Conference on Learning Representations}, 2023.

\end{thebibliography}
\bibliographystyle{icml2024}

\newpage
\appendix
\onecolumn

\section{Proof of \Cref{thm:structure}} \label{appendix:structurethm}

\structurethm*

\begin{proof}
\textbf{GF case:} Let $h(t, \x)$ be a conserved function for the ODE \eqref{gradientflow} with a right-hand side equal to zero with any initial condition $\x_0$. In that case, \eqref{gradientflow} rewrites $\dot \x = 0$. 
For each initialization $\x_0$, the solution of this ODE is $\x(t) = \x_0$ {for $t \in \R$} and by definition of a conserved function, one has 
{$h(t, \x_0) =  h(t, \x(t)) = h(0, \x(0))= h(0, \x_0)$.}

\textbf{MF case:} Let $h(t, \x, \dot \x)$ be a conserved function for the ODE \eqref{momentumflow} with a right-hand side equal to zero, an arbitrary initial condition $(t_0, \x_0, \dot \x_0)$, and $\tau(t) \coloneqq \tau$. In that case, \eqref{momentumflow} rewrites $\ddot \x + \tau \dot \x = 0$. For each initialization $(\x_0, \dot \x_0)$, the solution of this ODE is:
$\x(t) = \x_0 + \dot \x_0/\tau  \left( 1-\exp(-t \tau) \right)$, {for every $t \in \R$, and it satisfies} 
$(\x(0), \dot \x(0)) = (\x_0, \dot \x_0)$.

Since $h$ is conserved, $h(0, \x_0, \dot \x_0) 
{=h(t,\x(t),\dot \x(t))}
= h\left(t, \x_0 + \dot \x_0/\tau  \left( 1-\exp(-t \tau) \right), \dot \x_0 \exp(-t \tau) \right)$ {for each $t \in \R$.}
{This holds for any $\x_0, \dot \x_0$, and given any $\x,\dot \x$ one can find $\x_0, \dot \x_0$ such that $(\x,\dot\x) = (\x_0 + \dot \x_0/\tau  \left( 1-\exp(-t \tau) \right), \dot \x_0 \exp(-t \tau))$. Expliciting the expression of such $\x_0,\dot \x_0$ in terms of $\x,\dot \x$ yields}
$$
h(t, \x, \dot \x ) = h(0,  \x - \dot \x/\tau (\exp(t \tau) - 1), \dot \x \exp(\tau t))=  {H}\left(\x + \frac{\dot \x}{\tau}, \dot \x \exp(\tau t)\right).\qedhere
$$
\end{proof}

\section{Proof of \Cref{prop:orthogonality}} \label{appendix:prop:orthogonality}
\proporthogonality*

\begin{proof}
{We will use that $\langle \nabla h(\alpha),(1,F(\alpha)^\top)^\top\rangle = \partial h(\alpha) (1,F(\alpha)^\top)^\top$ with $\partial h$ the Jacobian of $h$.}

    Assume that  $\partial h (\alpha) (1, F(\alpha)^\top)^\top  = 0$ for all $\alpha \in \bigouv$. Then for all $\texttt{init} \coloneqq (t_{\text{init}}, \omega_{\text{init}}) \in \bigouv$ and for all $t \in \left(t_{\text{init}} -  \eta_{\text{init}}, t_{\text{init}} +  \eta_{\text{init}}\right),$
    {denoting $\dot \omega(t,\texttt{init}) \coloneqq \frac{\mathrm{d}}{\mathrm{d}t} \omega(t,\texttt{init})$ we have}
    \begin{align*}
\frac{\mathrm{d}}{\mathrm{d}t}  h(t, \omega(t,\texttt{init})) = \partial h(t, \omega(t,\texttt{init})) 
{(1,  \dot{\omega}(t,\texttt{init})^\top)^\top }
= \partial h(t, \omega(t,\texttt{init})) (1, F(t, \omega(t,\texttt{init}))^\top)^\top =  0.
\end{align*}
Thus: $h(t, \omega(t,\texttt{init}))= h(\texttt{init})$, {\em i.e.}, $h$ is conserved through~$\chi$.

Conversely, assume that there exists $(t_0, \omega_0) \in \bigouv$ such that $\partial h(t_0, \omega_0) (1, F(t_0, \omega_0)^\top)^\top  \neq 0$. Then by continuity of $z \in \bigouv \mapsto \partial h(z) (1, F(z)^\top)^\top $, there exists $r > 0$ such that $\partial h(z) (1, F(z)^\top)^\top  \neq 0$ on $B(\left(t_0, \omega_0\right), r)$. 
With ${\texttt{init}} = (t_0, \omega_0)$ 
by continuity of $t \mapsto \omega(t, \texttt{init})$, there exists $\epsilon > 0$, such that for all $t \in (t_0, t_0 + \epsilon)$, $\omega(t, \texttt{init}) \in B(\left(t_0, \omega_0\right), r)$. Then for all $t \in (t_0, t_0 + \epsilon)$:$
\frac{\mathrm{d}}{\mathrm{d}t}  h(t, \omega(t,\texttt{init})) = \partial  h(t, \omega(t,\texttt{init})) (1, F(t, \omega(t,\texttt{init}))^\top)^\top  \neq  0,$
hence $h$ is not conserved through the flow induced by $F$.
\end{proof}

\section{Proof of \Cref{proptracerewritten}} \label{appendix:tracerewritten}
{Recall that $F_Z$ is defined in \eqref{eq:DefFZ}.}

\proptracerewritten*

\begin{proof}
Let us consider $h: \tilde{\Theta} \mapsto \R \in \mathcal{C}^1(\tilde{\Theta}, \R)$. 
 We first show the direct implication. We assume that $h$ is locally conserved on $\Theta$ for any data set. Let $\z \in \tilde{\Theta}$ and let $Z=(x_i, y_i)_i \in \mathcal{Z}_\x$. By definition of $\mathcal{Z}_\x$, there exists a neighborhood of $\x$ $\ouv  \subseteq \Theta$ such that for all $i$, $g(\cdot, x_i) \in \mathcal{C}^2(\ouv, \R^n)$ and such that $M_Z(\cdot) \in \mathcal{C}^{1}(\tilde{\ouv}, \R^{D \times D})$. Then by definition of being locally conserved on $\Theta$ for any data set, $h$ is in particular conserved on $\ouv$ for any data set. Thus in particular $h$ is conserved on $\ouv$ during the flow $F_Z$. By \Cref{prop:orthogonality}, $\nabla h(\z) \perp   (1, F_Z(\z)^\top)^\top$. This holds for any $Z \in \mathcal{Z}_\x$, thus $\nabla h(\z) \perp W_\z^{\mathtt{mom}}$. As it is true for any $\x \in \tilde{\Theta}$, we have the direct implication.

 We now show the converse implication.  We assume that $\nabla h(\z) \perp W_\z^{\mathtt{mom}} $ for all $\z \in \tilde{\Theta}$. Let us consider an open subset $\ouv \subseteq \Theta$, and let us consider $Z= (x_i, y_i)_i$ a data set such that $g(\cdot, x_i) \in \mathcal{C}^{2}(\ouv, \R)$ for each $i$ and $(t, \x, \dot \x) \mapsto M_Z(t, \x, \dot \x)  \in \mathcal{C}^{1}(\bigouv, \R^{D \times D})$. In particular, $Z \in \mathcal{Z}_\x$. Thus, for any $\z \in \bigouv$, as $\nabla h(\z) \perp W_\z^{\mathtt{mom}} $, one has $\nabla h(\z) \perp  (1, F_Z(\z)^\top)^\top$, and by \Cref{prop:orthogonality}, $h$ is conserved on $\ouv$ during the flow $F_Z$. As this holds for any $\ouv \subseteq \Theta$, $h$ is locally conserved on $\Theta$ for any data set.
\end{proof}

\section{Proof of \Cref{prop:link_mom_grad}} \label{appendix:link_MG}
\linkMGflow*
\begin{proof}
Let $z \coloneqq (t, \x, \dot{\x}) \in \tilde{\Theta}$.
Let us denote $\mathcal{Z}_\x'$ the 
{collection of all}
 data set $Z=(x_i, y_i)_i$ such that for all $i$, $g(\cdot,x_i)$ is $\mathcal{C}^2$-differentiable in the neighborhood of $\x$.
Let  $z \coloneqq (t, \x, \dot{\x})\in \tilde{\Theta}$.
First, let us recall that (See Proposition 2.7 of \cite{marcotte2023abide}) we have: 
\eq{ \label{eq:GFeuclidien}
W_\x
= \underset{Z  \in \mathcal{Z}_\x'}{\linspan} \{ \nabla \mathcal{E}_Z (\x) \}.
}
Then as by assumption 
${M(\cdot)} \in \mathcal{C}^1(\tilde{\Theta}, \R^{D \times D})$, 
we have that 
{$\mathcal{Z}'_\x = \mathcal{Z}_\x$ with $\mathcal{Z}_\x$ defined as in \Cref{proptracerewritten}. 
}
Thus, we can rewrite: 
 \begin{align*}
 W_\z^{\mathtt{mom}}
    = \underset{Z \in \mathcal{Z}_\x'}{\linspan} \left\{ (1, F_Z(\z)^\top)^\top  \right\} \stackrel{\eqref{eq:DefFZ}}{=} \underset{Z \in \mathcal{Z}_\x'}{\linspan} \left\{  \smallvec{
   1 \\ \dot{\theta} \\ -\tau(t)  \dot{\theta}
  } +  \smallvec{ 0 \\  0 \\
     -M(t, \x, \dot \x) \nabla \mathcal{E}_Z(\theta) } \right\}.
\end{align*}
Thus:
$
W_\z^{\mathtt{mom}} \subseteq \R 
 \smallvec{
   1 \\ \dot{\theta} \\ -\tau(t)  \dot{\theta}
    } + \underset{Z \in \mathcal{Z}_\x'}{\linspan} \smallvec{ 0 \\  0 \\
     -M(t, \x, \dot \theta) \nabla \mathcal{E}_Z(\theta) } = \R 
    \smallvec{1 \\ \dot{\theta} \\ -\tau(t)  \dot{\theta}}
    + 
    \smallvec{0 \\  0 \\
     - M(t, \x, \dot \x) W_{\x} },
$
which gives the direct inclusion.

Let us show the converse inclusion. By assumption, there exists $Z_0 \in \mathcal{Z}_\x'$ such that $\nabla \mathcal{E}_{Z_0} (\x) = 0$, so that $W_\z^{\mathtt{mom}}\ni (1, F_{Z_0}(\z)^\top)^\top =  \smallvec{1 \\ \dot{\theta} \\ -\tau(t)  \dot{\theta}} $, 
and thus $\R  \smallvec{1 \\ \dot{\theta} \\ -\tau(t)  \dot{\theta}} \subseteq  W_\z^{\mathtt{mom}}$. 
Then for all $Z \in \mathcal{Z}_\x'$, 
$$
\smallvec{0 \\  0 \\
     {-M(t,}
     \x, \dot \x) \nabla \mathcal{E}_Z(\theta) } = \underbrace{ \smallvec{1 \\ \dot{\theta} \\ -\tau(t)  \dot{\theta}} +   \smallvec{0 \\  0 \\
          {-M(t,}
     \x, \dot \x) \nabla \mathcal{E}_Z(\theta) }}_{{=(1,F_Z(\alpha)^\top)^\top} \in W_\z^{\mathtt{mom}} } - \underbrace{ \smallvec{1 \\ \dot{\theta} \\ -\tau(t)  \dot{\theta}}}_{\in W_\z^{\mathtt{mom}}} \in W_\z^{\mathtt{mom}},
    $$
    and thus $  \smallvec{0 \\  0 \\
     - M(t, \x, \dot \x) W_{\x} } 
     = \underset{Z \in \mathcal{Z}_\x'}{\linspan} \smallvec{ 0 \\  0 \\
     -M(t, \x, \dot \x) \nabla \mathcal{E}_Z(\theta) } \subseteq W_\z^{\mathtt{mom}}$, which
   gives the converse inclusion. 
\end{proof}

We now show that the following assumption holds for classical losses in machine learning.
\begin{assumption}\label{assumptionloss}
For all $\x \in {\Theta}$, there exists $Z \in \mathcal{Z}_\x'$ such that $\nabla \mathcal{E}_Z(\x) = 0$.
\end{assumption}
{
\begin{lemma}
    \Cref{assumptionloss} holds for the mean-square error loss and the cross-entropy loss.
\end{lemma}
}
\begin{proof}
Let $\x \in \Theta$.

\textit{Mean-square error loss.} The mean-squared error loss is defined by $(z, y) \mapsto \ell_2(z, y) \coloneqq \| y-z \|^2$.
Let consider $x$ such that $g(\cdot, x)$ is $\mathcal{C}^2$ on a neighborhood of $\x$. Then consider $y \coloneqq g(\x, x)$. By definition $Z = (x, y) \in \mathcal{Z}_\x'$. Moreover one has in a neighborhood of $\x$: $\nabla \mathcal{E}_Z (\x') = \partial g(\x', x)^\top \nabla_z \ell_2(g(\x', x), y)) =  2 \partial g(\x', x)^\top  (g(\x', x)- y))$. Thus $\nabla \mathcal{E}_Z (\x) = 0$.

\textit{Cross-entropy loss.} The cross-entropy loss is defined by $(z \in \R^{n}, y \in \Sigma_n) \mapsto \ell_{\mathtt{cross}}(z,y) \coloneqq \text{KL}(\text{softmax}(z), y)$, where $\Sigma_n \coloneqq \{ (y_1, \cdots, y_n) {\in \R_+^n}: \sum_{i =1}^n y_i = 1 \}$ {is the simplex and} 
KL is the Kullback–Leibler divergence defined on $\Sigma_n \times \Sigma_n$ by $\text{KL}(r, q) \coloneqq \sum_{i=1}^n r_i \log(r_i/q_i)$. In particular by taking $x$ such that $g(\cdot, x)$ is $\mathcal{C}^2$ on a neighborhood of $\x$ and by taking $y = \text{softmax} (g(\x, x))$ and $Z = (x, y)$ (in $\mathcal{Z}_\x'$), we obtain $\nabla \mathcal{E}_Z(\x) = 0$.
\end{proof}

\section{Proof of \Cref{thm:conservedfunctionsEuclidean}} \label{appendix:preuveconservationeuclidienne}
{We consider linear networks with $\x \coloneqq (U_1,\ldots,U_q)$, $g(\x,x) \coloneqq U_1\ldots U_qx$.}
\conservedfunctionsEuclidean*

\begin{proof}
{We first treat the general case. The specific case $(n_{i-1},n_{i+1}) = (1,1)$ will come next.}
    Let us denote for $A \in \mathcal{A}_{n_i}$:
    \begin{equation}\label{eq:DefHAAnnex}
    h_A(t, \x, \dot \x) \coloneqq 
\exp\Big( \int^t \tau(s) \mathrm{d}s\Big) \left( 
\left\langle \dot U_{i},  U_{i}A
   \right\rangle  +
   \left\langle 
    \dot U_{i+1}, A^\top U_{i+1}  \right\rangle  \right).
    \end{equation}
{Observe that \eqref{momentumflow} implies 
\begin{equation}    \label{eq:DefMFBlockAnnex}
\ddot U_i(t)+\tau(t)\dot U_i(t) = -\nabla_{U_i} \mathcal{E}(\x(t))
\end{equation}
for every $1\leq i \leq q$, where to ease further computation each $\nabla_{U_i}$ is reshaped as the $n_{i-1} \times n_i$ matrix $U_i$.}
Then
\begin{align*}
    \frac{\mathrm{d}}{\mathrm{d}t} h_A (t, \x(t), \dot \x(t)) &= \exp\Big( \int^t \tau(s) \mathrm{d}s\Big) \left(  \left\langle \ddot U_{i},  U_{i}A
   \right\rangle  +  \underbrace{\left\langle \dot U_{i},  \dot U_{i}A \right\rangle}_{= 0 \text{ as } A \text{ is a skew matrix }}
    +
   \left\langle 
    \ddot U_{i+1}, A^\top U_{i+1}  \right\rangle  + \underbrace{\left\langle 
    \dot U_{i+1}, A^\top \dot U_{i+1}  \right\rangle }_{= 0 \text{ as } A \text{ is a skew matrix }} \right) \\
    & \quad \quad + \tau(t) h_A (t, \x(t), \dot \x(t)) \\
    &\stackrel{\eqref{eq:DefMFBlockAnnex}}{=} - \exp\Big( \int^t \tau(s) \mathrm{d}s\Big) \left( \langle  \nabla_{U_{i}} \mathcal{E}_Z(\x),  U_{i}A
   \rangle +  \langle  \nabla_{U_{i+1}} \mathcal{E}_Z(\x), \underbrace{A^\top}_{= - A}  U_{i+1}
   \rangle \right) \\
    &= 0 \text{ as } T^A  \text{ leaves } 
    \mathcal{E}_Z \text{invariant (See \eqref{eq:invarianceloss} and \Cref{example:linear_transf})}, 
\end{align*}
where here
$
T^A(\epsilon,U_{i}, U_{i+1}) =
   \big(  U_i  \exp(\epsilon A),\exp(-\epsilon A) U_{i+1} \big), 
$
and $\Delta_{T^A }(U_i, U_{i+1}) =  (U_i A, - A U_{i+1})$.

\textbf{Special case where $(n_{i-1} = n_{i+1} = 1)$. } For any $A \in \mathcal{A}_{n_i}$, we denote:
$$
g_A(t, \x, \dot \x) \coloneqq \exp\Big( \int^t \tau(s) \mathrm{d}s\Big) \left( 
\left\langle \dot U_{i},   U_{i+1}^\top A
   \right\rangle  +
   \left\langle 
    \dot U_{i+1}^\top,  U_{i} A \right\rangle  \right)
$$
{By our assumption on the dimensions, $U_{i+1}^\top A$ is an $n_{i+1} \times n_i = 1 \times n_i$ matrix, just as $\dot U_i$, and similarly for $U_i A$ and $\dot U_{i+1}^\top$, so $g_A$ is indeed well-defined. We will prove below that (again with gradients properly reshaped as matrices)}
\eq{ \label{eq:prop1}
\nabla_{U_{i}} \mathcal{E}_Z(\x) \propto  U_{i+1}^\top 
\ \text{and}\ 
\nabla_{U_{i+1}} \mathcal{E}_Z(\x) \propto  U_{i}^\top. 
}
{As a result, using again \eqref{eq:DefMFBlockAnnex} we compute}
\begin{align*}
     \frac{\mathrm{d}}{\mathrm{d}t} g_A (t, \x(t), \dot \x(t)) &= \exp\Big( \int^t \tau(s) \mathrm{d}s\Big) \left(  \left\langle \ddot U_{i},  U_{i+1}^\top A
   \right\rangle  +   
   \left\langle 
    \ddot U_{i+1}^\top, U_{i} A  \right\rangle +
 \underbrace{\left\langle \dot U_{i},  \dot U_{i+1}^\top A \right\rangle
    + \left\langle 
    \dot U_{i+1}^{{\top}}, \dot U_i A  \right\rangle }_{
    {= \langle \dot U_{i+1}\dot U_i + \dot U_i^\top \dot U_{i+1}^\top,A\rangle}
    = 0 \text{ as } A \text{ is a skew matrix }} \right) \\
    & \quad \quad + \tau(t) g_A (t, \x(t), \dot \x(t)) \\
     &= - \exp\Big( \int^t \tau(s) \mathrm{d}s\Big) \left( \underbrace{\langle \underbrace{ \nabla_{U_{i}} \mathcal{E}_Z(\x)}_{ \propto  U_{i+1}^\top \text{ by } \eqref{eq:prop1}},  U_{i+1}^\top A
   \rangle }_{= 0 \text{ as } A \text{ is a skew matrix }}+ \underbrace{ \langle  \underbrace{ [\nabla_{U_{i+1}} \mathcal{E}_Z(\x)]^\top}_{ \propto  U_i \text{ by } \eqref{eq:prop1}},  U_{i} A
   \rangle}_{= 0 \text{ as } A \text{ is a skew matrix }} \right) \\
   & = 0.
\end{align*}

{Let us now prove \eqref{eq:prop1} as claimed.}
First, to simplify the notations, let us define:
$$
\widehat{V_j} =  \left\{
    \begin{array}{ll}
 U_1...U_{j-1} \mbox{ for all } j = 2,..., q, \\
        \mId{n_0} \mbox{ for } j=1.
    \end{array}
\right.
$$
and
$$
\widehat{W_j} =  \left\{
    \begin{array}{ll}
U_{j+1}... U_q \mbox{ for all } j = 1,..., q-1, \\
        \mId{n_q} \mbox{ for j=q.}
    \end{array}
\right.
$$
Let us now derive an expression $\nabla_{U_i} \mathcal{E}_Z$. 
{Given that $g(\x,x) = U_1\ldots U_q x$,}
we can factorize $\mathcal{E}_Z(\x) = F \circ \phi(\x)$ 
{with $\phi(\x) \coloneqq U_1 \ldots U_q \in \R^{n_0 \times n_q}$ (identified with $\Rd$, $d=n_0n_q$), and $F: \Rd \to \R$. The chain rule for Jacobians yields}
\begin{equation} \label{comp}
\partial \mathcal{E}_Z \left( \x \right) = \partial F \left( \phi(\x) \right) \cdot \partial \phi \left(\x \right)
\end{equation}
hence for all $(H_1, \cdots, H_q) \in \R^{D}$ ({with components} $H_i \in \R^{n_{i-1} \times n_{i}}$ {seen as vectors in $\R^{n_{i-1}n_i}$}):

\begin{align*}
\left\langle \nabla  \mathcal{E}_Z \left( \x \right), (H_1, \cdots, H_q) \right\rangle 
&=
\partial  \mathcal{E}_Z \left( \x \right) \cdot (H_1, \cdots, H_q) 
\stackrel{\eqref{comp}}{=}  \left \langle \nabla F\left( \phi(\x) \right), \partial \phi \left( \x \right) \cdot (H_1,\cdots, H_q) \right\rangle \\
&= \left\langle \nabla F \left( \phi(\x) \right), \sum_{j=1}^q U_1...U_{j-1} H_j U_{j+1} ... U_q \right\rangle 
= \sum_{j=1}^q \left\langle \nabla F \left( \phi(\x) \right), \widehat{V_j} H_j \widehat{W_j} \right \rangle \\
&=  \sum_{j=1}^q \mathtt{Tr}  \left( [{\nabla F \left( \phi(\x) \right)}]^\top \widehat{V_j} H_j \widehat{W_j} \right) 
= \sum_{j=1}^q \mathtt{Tr}  \left( H_j \widehat{W_j}[{\nabla F \left( \phi(\x) \right)}]^\top \widehat{V_j} \right) \\
&= \sum_{j=1}^q  \left \langle H_j,  \widehat{V_j}^\top {\nabla F \left( \phi(\x) \right)} \widehat{W_j}^\top \right \rangle. 
\end{align*}
Then by Riez theorem, {viewing again for convenience $\nabla F$ as its matrix reshaped version, we get that for each $1 \leq j \leq q$}:
\begin{equation} \label{b}
\nabla_{U_j}  \mathcal{E}_Z \left( \x \right)  =  \underbrace{\widehat{V_j}^\top}_{n_{j-1} \times n_0 } \underbrace{ {\nabla F \left( \phi(\x) \right)} }_{n_0 \times n_q } \underbrace{\widehat{W_j}^\top}_{n_q \times n_{j}}.
\end{equation}
Thus by {specializing to $j \in \{i,i+1\}$ and by} developing \eqref{b}, one has:
$$
\nabla_{U_{i}} \mathcal{E}_Z(\x) = \underbrace{a}_{\text{ size } n_{i-1} \times n_{i+1}}  U_{i+1}^\top, \quad \text{ and }
 \quad \nabla_{U_{i+1}} \mathcal{E}_Z(\x) =  U_{i}^\top \underbrace{b}_{\text{ size } n_{i-1} \times n_{i+1}},
$$
thus establishing
 \eqref{eq:prop1}.
\end{proof}

\section{Application of Noether Theorem} \label{appendix:noether}
{Invariances of $\mathcal{E}_Z$ with respect to certain transformations have been used in the proof of \Cref{thm:conservedfunctionsEuclidean}. Here we establish more direct connections with Noether's theorem (see \Cref{thm:noether}). For the sake of brevity we describe the case $q=2$ but the same reasoning can easily be adapted to any $q \geq 2$ by considering the invariances associated to each pair $U_{i}$, $U_{i+1}$.}

{\textbf{Invariances valid for any dimension.}}
We consider the flow \eqref{momentumflow} in the Euclidean case ($M = \mId{d}$), with $\theta \coloneqq (U, V) \in \R^{n\times r} \times \R^{m\times r}$ and $g(\x, x) \coloneqq U V^\top x$. We recall that for all $A \in \R^{r\times r}$,  the linear transformation $ 
T^A$ from \eqref{linear_transf} leaves the cost \eqref{eq:erm} invariant, and that  (see  \Cref{example:linear_transf}) that:
$\Delta_{T^A } (U, V) =   (U A, -V A^\top).$
Moreover for any $A$ in $\mathcal{A}_r$ the space of skew-symmetric matrices in $\R^{r \times r}$, $T^A$ leaves invariant the following Lagrangian (See for example \cite{głuch2021noether}), with which the Euclidian MF is compatible:
\begin{equation} \label{lag:gen}
\mathcal{L}\left(t, \x, v\right) \coloneqq \exp\Big( \int^t \tau(s) \mathrm{d}s\Big)  \left(  1/2 \|  v \|^2 - \mathcal{E}_Z(\x) \right).    
\end{equation}

 Thus by Noether theorem (\Cref{thm:noether}), for any $A \in \mathcal{A}_r$, a conserved function is
\begin{align*}
{h(t,\x,\dot x) \coloneqq}   \partial_{v} \mathcal{L} (t, \x, \dot{\x}) \cdot \Delta_T^A (\x(t))  &= 
\exp\Big( \int^t \tau(s) \mathrm{d}s\Big)  \left(\langle \dot U, UA \rangle - \langle \dot V, V A^\top \rangle \right) \\
 &=\exp\Big( \int^t \tau(s) \mathrm{d}s\Big)  \left(\langle \dot U, UA \rangle +\langle \dot V, V A \rangle \right).
\end{align*}

\textbf{A supplementary invariance when $(n ,m) = (1, 1)$.} Assume $(n, m) = (1, 1)$, 
{so that $U,V \in \R^{1 \times r}$ and $\x \in \R^{1 \times 2r}$, and}
consider $T'^A$ the linear transformation where for all $\epsilon \geq 0$: 
\eq{  \label{eq:newtransformation}
T'^A(\epsilon,U, V) =
   \Big(V \exp(-\epsilon A^\top),  U  \exp(\epsilon A)\Big). 
}
One has: $\Delta_{T'^A }(U, V) =  (-V A^\top, UA)$.
{It is easy to check that since $g(\x,x) \coloneqq UV^{\top}x$, we have $g(T'^A(\epsilon , \cdot)) = g(\x, \cdot)$ hence}
$T'^A$ leaves invariant the cost \eqref{eq:erm}.
Moreover, 
{one can also easily check that for any skew matrix $A \in \mathcal{A}_{r}$, the transformation $T'^{A}$ also}
leaves the Lagrangian \eqref{lag:gen} invariant. Thus, by Noether theorem, for any $A \in \mathcal{A}_r$, the quantity

\begin{align*}
{h(t,\x,\dot x) \coloneqq}   \partial_{v} \mathcal{L} (t, \x, \dot{\x} )\cdot \Delta_{T'^A }(\x) &= 
\exp\Big( \int^t \tau(s) \mathrm{d}s\Big)  \left(\langle \dot U, -VA^\top \rangle + \langle \dot V, UA \rangle \right) \\
 &=\exp\Big( \int^t \tau(s) \mathrm{d}s\Big)  \left(\langle \dot U, VA \rangle +\langle \dot V, UA \rangle \right)
\end{align*}
is conserved.

\section{Proof of \Cref{prop:counting}} \label{appendix:counting}

\countingtheconservedfunctions*

\begin{proof}
First, observe that if
$r = 1$, then \Cref{thm:conservedfunctionsEuclidean} gives zero conserved functions as $\mathcal{A}_1 = \{ 0 \}$, 
and indeed here $2(n+m) > r=1$ and $r(r-1)=0$. 

We now focus on the case where $r >1$ and denote $U = (u_1, \cdots, u_r)$, $V = (v_1, \cdots, v_r)$,  $\dot U = (\dot u_1, \cdots, \dot u_r)$, $ \dot V = (\dot v_1, \cdots, \dot v_r)$, with $u_{i},\dot{u_{i}} \in \R^{n}$ and $v_{i},\dot{v}_{i} \in \R^{m}$. 
For every $1\leq i, j \leq r$, by \Cref{thm:conservedfunctionsEuclidean} the elementary skew matrix $A_{i,j}\coloneqq E_{i,j}-E_{j,i} \in \mathcal{A}_{r}$ leads to a conserved function  $h_{i,j}\coloneqq h_{A_{i,j}}$ (see~\eqref{eq:DefHAAnnex}).
Since $(n,m) \neq (1,1)$, each conserved function predicted by  \Cref{thm:conservedfunctionsEuclidean} is of the form $h_{A}$, $A \in \mathcal{A}_{r}$, hence is a linear combination of $h_{i,j}$, $1 \leq i,j \leq r$. 
As a result, conserved function predicted by \Cref{thm:conservedfunctionsEuclidean}  satisfy $\nabla h_{A}(\alpha) \in \linspan \{
\nabla h_{i,j}(\alpha): 1 \leq i,j \leq r\}$ for every  $\alpha = (t, \x,\dot \x)$. 
We will show below that for any $\alpha \in \widetilde{\Theta}$, there is a set of indices $S \subseteq \{1,\ldots r\}$ of cardinality $R = \min(r,2(n+m))$ and a neighborhood $\widetilde{\Omega}$ of $\alpha$ such that, with $T \coloneqq \{(i,j) \in S \times S, i<j\} \cup (S \times S^{c})$, we have for every $\alpha'  \in \widetilde{\Omega}$:
\begin{itemize}
\item $\nabla h_{k,\ell}(\alpha') \in \linspan\{\nabla h_{i,j}(\alpha'): (i,j) \in T\}$ for every $(k,\ell) \notin T$;
\item the vectors $\nabla h_{i,j}(\alpha')$, $(i,j) \in T$ are linearly independent.
\end{itemize}
This will conclude since
 $$\sharp T = 
 R(R-1)/2 + R(r-R) = \begin{cases}
 r(r-1)/2 & \text{if}\ R=r<2(n+m)\\
(n+m)[(r-2(n+m))+r-1] & \text{if}\ R=2(n+m) \leq r.
\end{cases}
$$

\indent To proceed, specialize the definition \eqref{eq:DefHAAnnex} of $h_{A}$ to our context:  $q=2$, $U_{1} = U$ and $U_{2}=V^{\top}$ yields
$h_{A}(\alpha)=h_{A}(t,\x,\dot \x) = \exp\Big( \int^t \tau(s) \mathrm{d}s\Big)  H_{A}(\x,\dot \x)$ with
$
H_{A}(\x,\dot \x) \coloneqq \langle \dot U, U A\rangle + \langle \dot V^{\top}, V^{\top}A^{\top}\rangle =  \langle \dot U, U A\rangle +  \langle \dot V, V A\rangle 
=  \langle \dot U A^{\top}, U \rangle +  \langle \dot V A^{\top}, V \rangle
$. 
As a result
\begin{equation}\label{eq:hfromH}
\nabla h_{i, j} (t, \x, \dot \x) = (\tau(t) h_{i, j}(t, \x, \dot \x),  \exp( \int^t \tau(s) \mathrm{d}s) \times [\nabla H_{A_{i, j}}(\x, \dot \x)]^{\top})^\top
\end{equation}
where, up to proper reshaping
\[
\nabla H_{A}(\x,\dot \x) = \left(\begin{matrix}\nabla_{U} H_{A}\\
\nabla_{V} H_{A}\\
\nabla_{\dot U} H_{A}\\
\nabla_{\dot V} H_{A}
\end{matrix}\right)
=
\left(\begin{matrix}
\dot U A^{\top}\\
\dot V A^{\top}\\
U A\\
VA
\end{matrix}\right)
\stackrel{A^\top= - A}{=}
\left(\begin{matrix}
-\dot U A\\
-\dot V A\\
U A\\
VA
\end{matrix}\right)
=
\underbrace{(-\dot U; -\dot V; U; V)}_{=:W} A.
\] 
hence the $2(n+m) \times r$ matrix $W = W(\alpha)$ will play a special role. Denote 
 $w_{j}(\alpha)$, $1 \leq j \leq r$ its columns.  
 
 Observe that given $\alpha = (t,\x, \dot \x) \in \widetilde{\Theta}$, $R \coloneqq \min(r,2(n+m))$ is both the  rank of $(\x,\dot \x) = (U; V; \dot U; \dot V)$ and the rank of $W$. Hence, given $\alpha \in \widetilde{\Theta}$, there is a subset of indices $S \subseteq \{1,\ldots r\}$ of cardinality $R$ such that the vectors $w_{j}(\alpha)$, $j \in S$ are linearly independent, while for $k \notin S$ we have  $w_{k}(\alpha) \in \linspan (w_{j}(\x,\dot \x),j \in S)$. By standard calculus, there is a neighborhood $\widetilde{\Omega} \subseteq \widetilde{\Theta}$ of $\alpha$  such that these properties remain valid (with the same $S$) for every $\alpha' \in \widetilde{\Omega}$. We show below that this implies the claimed linear (in)dependence properties of the vectors $\nabla h_{i,j}(\alpha')$. From now on we omit the dependence in $\alpha'$ for the sake of brevity.

By~\eqref{eq:hfromH}, to show that the vectors $\nabla h_{i,j}$, $(i,j) \in T$ are linearly independent, it is sufficient to show the linear independence of $\nabla H_{A_{i,j}} \coloneqq W A_{i,j}$, $(i,j) \in T$. 
Assume that $0_{2(n+m) \times r} = \sum_{(i,j) \in T} \lambda_{i,j} W A_{i,j}$. Our goal is now to show that  $\lambda_{k,\ell}=0$ for every  $(k,\ell) \in T = (T \cap (S \times S)) \cup (S \times S^{c})$. We first prove it for every $(k,\ell) \in S \times S^{c}$, then on $T \cap (S \times S)$.

First consider $(k,\ell) \in S \times S^{c}$. For any $(i,j) \in T$, by the definition of $T$ we have $i \in S$, hence $\ell \neq i$ since $\ell \notin S$. Using the standard notation for canonical vectors and Kronecker deltas, we thus have $e_{i}^{\top}e_{\ell}=\delta_{i,\ell}=0$, and since $A_{i,j} = e_{i}e_{j}^{\top}-e_{j}e_{i}^{\top}$ we obtain $A_{i,j}e_{\ell} = e_{i}\delta_{j,\ell}$, so that $WA_{i,j}e_{\ell}= \delta_{j,\ell} w_{i}$ and
\[
0 = \sum_{(i,j) \in T} \lambda_{i,j} WA_{i,j}e_{\ell}= \sum_{(i,j) \in T} \lambda_{i,j} \delta_{j,\ell}w_{i}
=\sum_{i \in S} \Big(\sum_{j: (i,j) \in T} \lambda_{i,j} \delta_{j,\ell}\Big) w_{i}. 
\]
where by convention an empty sum is zero. By the linear independence of $w_{i}$, $i \in S$ we get for each $i \in S$ 
that $0 = \sum_{j: (i,j) \in T} \lambda_{i,j} \delta_{j,\ell}$. Since $(k,\ell) \in T$, specializing to $i\coloneqq k \in S$, we obtain $0 = \sum_{j: (k,j) \in T} \lambda_{k,j}\delta_{j,\ell} = \lambda_{k,\ell}$ as claimed.

Since the above holds for any $(k,\ell) \in S \times S^{c}$, and given the definition of $T$, we have established that in fact 
\[
0_{2(n+m) \times r} = \sum_{ (i,j) \in T \cap (S\times S)} \lambda_{i,j} W A_{i,j} = W \Big( \sum_{(i,j) \in T\cap (S \times S)} \lambda_{i,j} A_{i,j}\Big) = WB,
\]
 where $B \coloneqq \sum_{(i,j) \in T \cap (S \times S)} \lambda_{i,j} A_{i,j}$. Observe that, by definition of $A_{i,j}$, we have $B(S^{c},:) = 0$ (the rows of $B$ indexed by $S^{c}$ are zero), hence $0 = WB = W(:,S)B(S,:)$. By the linear independence of the columns $w_{j}$, $j \in S$ of $W(:,S)$, we conclude that $B(S,:)=0$, hence $B=0$, that is to say $\sum_{(i,j) \in T \cap (S \times S)} \lambda_{i,j} A_{i,j}=0$. Since $(i,j) \in T$ implies $i\neq j$ and as the matrices $A_{i,j}$, $i \neq j$, are linearly independent, we conclude that $\lambda_{i,j}=0$ for every $(i,j) \in T \cap (S \times S)$. 

This establishes that 
 $\nabla h_{i,j}(\alpha')$, $(i,j) \in T$, are linearly independent.

To conclude the proof, there remains to show that $\nabla h_{k,\ell}(\alpha') \in \linspan\{\nabla h_{i,j}(\alpha'): (i,j) \in T\}$ for every $(k,\ell) \notin T$.

Consider $(k,\ell) \notin T$. As $w_k$ and $w_\ell$ are linear combinations of $\{w_j\}_{j \in S}$ there exists $(\alpha_j)_j \neq (0)$ and $(\beta_j)_j \neq (0)$ such that $w_k = \sum_{j \in S} \alpha_j w_j$ and $w_\ell = \sum_{j \in S} \beta_j w_j$. As a result 
    $$
    \nabla h_{k, \ell} = \sum_{j \in S}(\alpha_j \nabla h_{j, \ell} - \beta_j \nabla h_{j, k}) - \sum_{i, j \in S} \alpha_i \beta_j \nabla h_{i, j}. 
    $$
\end{proof}

\section{Proof of \Cref{theorem:dimliealgebra}} \label{appendix:dimliealgebra}
First, we recall some definitions/results about Lie algebra, directly taken from \cite{marcotte2023abide} and {we add a supplementary lemma (\Cref{lemma:traces_lie_equal}) that states useful results.}

\subsection{Background on Lie algebra}  \label{liealgebrabackground}
A Lie algebra {$\mathcal{A}$} is a vector space endowed with
a bilinear map  $[\cdot, \cdot]$, called a Lie bracket, that verifies for all $X, Y, Z \in {\mathcal{A}}$:
$
[X, X]= 0$ and the Jacobi identity: $
[X, [Y, Z]] + [Y, [Z, X]] + [Z, [X, Y]] = 0.
$

Typically, the Lie algebra of interest is the set of infinitely smooth vector fields ${\mathcal{A} \coloneqq} \mathcal{C}^\infty(\tilde{\Theta}, \R^{2D+1})$, endowed with the Lie bracket $[\cdot, \cdot]$ defined by
\begin{equation}\label{eq:def-lie-brac}
[\chi_1,\chi_2]:\quad 
\z \in \tilde{\Theta} \mapsto [\chi_1, \chi_2](\z)\coloneqq \partial \chi_1(\z) \chi_2(\z) - \partial \chi_2(\z) \chi_1(\z),
\end{equation}
with $\partial \chi (\z) \in \R^{(2D+1) \times (2D+1)}$ the jacobian of $\chi$ at $\z$.
The space $\R^{n \times n}$ of matrices is also a Lie algebra endowed with the Lie bracket
$[A, B] \coloneqq AB-BA.$   
This can be seen as a special case of~\eqref{eq:def-lie-brac} in the case of \emph{linear} vector fields, i.e. $\chi(\z)=A\z$. 
\paragraph{Generated Lie algebra} Let {$\mathcal{A}$} be a Lie algebra and let {$\mathcal{W} \subset \mathcal{A}$} be a vector subspace of {$\mathcal{A}$}. There exists a smallest Lie algebra that contains $\mathcal{W}$. It is denoted $\lie(\mathcal{W})$ and called the generated Lie algebra of~$\mathcal{W}$. The following proposition (See Definition 20 of \cite{Bonnard})
constructively characterizes $\lie(\mathcal{W})$, where for vector subspaces $[\mathcal{W},\mathcal{W}'] \coloneqq \{[\chi_1,\chi_2]: \chi_1 \in \mathcal{W}, \chi_2 \in \mathcal{W'}\}$, and $\mathcal{W}+\mathcal{W}' = \{\chi_1+\chi_2: \chi_1 \in \mathcal{W}, \chi_2 \in \mathcal{W}'\}$.

\begin{proposition} \label{buildingliealgebra}
    Given any vector subspace $\mathcal{W} \subseteq \mathcal{A}$ we have $\lie(\mathcal{W}) = \bigcup_k \mathcal{W}_k$ where:
\vspace{-0.5em}
     \begin{equation*} 
\left\{
    \begin{array}{ll}
        \mathcal{W}_0 &\coloneqq \mathcal{W}\\
        \mathcal{W}_k &\coloneqq  \mathcal{W}_{k-1} + [\mathcal{W}_0, \mathcal{W}_{k-1}]\ \text{ for }\ k \geq 1.
    \end{array}
\right.
\end{equation*}
\end{proposition}

\begin{lemma} \label{lemma:traces_lie_equal}
    let $\mathcal{W} \subseteq \mathcal{X}\left(\tilde{\Omega}\right)$ be a vector space. Then by considering 
    \eq{ \label{eq:operator_D}
    \mathcal{D}(\mathcal{W}) \coloneqq \linspan \{ a(\cdot) \chi(\cdot): a \in \mathcal{C}^\infty(\tilde{\Omega}, \R), \chi \in  \mathcal{W} \},
    }
    one has  $\mathcal{D}(\lie ( \mathcal{D}(\mathcal{W}))) = \lie ( \mathcal{D}(\mathcal{W}))$ and for all $\alpha \in \tilde{\Omega}$, $
    [\lie(\mathcal{W}) ](\alpha) = [\lie ( \mathcal{D}(\mathcal{W}))](\alpha)$.
\end{lemma}

\begin{proof}
\textbf{We first show by recursion that for all $k \in \mathbb{N}$}:
\eq{ \label{eq:recurrence}
\mathcal{D}(\mathcal{W})_k = \mathcal{D} \left(\mathcal{W}_k \right).
} 
    where the iterates are defined in \Cref{buildingliealgebra}.

    \textit{Initialisation.} One has: 
$     \mathcal{D}(\mathcal{W})_0  =\mathcal{D}(\mathcal{W}) =   \mathcal{D}(\mathcal{W}_0)$ by definition of the first iterate.

\textit{Recursion.} Let $k \in \mathbb{N}$ and assume that $k$ satifies
\eqref{eq:recurrence}.
Let us show that $\mathcal{D}(\mathcal{W})_{k+1} = \mathcal{D} \left(\mathcal{W}_{k+1} \right)$.
By definition (\Cref{buildingliealgebra}), one has 
\eq{ \label{eq:rec_lie}
\mathcal{D} \left(\mathcal{W} \right)_{k+1} = \mathcal{D} \left(\mathcal{W}\right)_k  + [\mathcal{D} \left(\mathcal{W} \right), \mathcal{D} \left(\mathcal{W}\right)_k ].
}
Thus $
   \mathcal{D} \left(\mathcal{W}\right)_{k+1}  \stackrel{\eqref{eq:recurrence}}{=}  \mathcal{D} \left(\mathcal{W}_k \right)+  [\mathcal{D} \left(\mathcal{W} \right), \mathcal{D} \left(\mathcal{W}_k \right)].
$
We first show the direct inclusion $\mathcal{D}(\mathcal{W})_{k+1} \subseteq \mathcal{D}(\mathcal{W}_{k+1})$. Since 
$\mathcal{W}_{k} \subseteq \mathcal{W}_{k+1}$ by construction, we have $\mathcal{D}(\mathcal{W}_{k}) \subseteq \mathcal{D}(\mathcal{W}_{k+1})$, hence it is enough to show that $ [\mathcal{D} \left(\mathcal{W} \right), \mathcal{D} \left(\mathcal{W}_k \right)] \subseteq  \mathcal{D} \left(\mathcal{W}_{k+1} \right)$.
Let $X \in \mathcal{D} \left(\mathcal{W} \right)$ and $ Y \in \mathcal{D} \left(\mathcal{W}_k \right)$. By definition of the operator $\mathcal{D}$, there are smooth real-valued functions $a_i,b_j$ and $\chi_i \in \mathcal{W}, \mu_j \in \mathcal{W}_k$ such that $X(\cdot) = \sum_{1}^{m_1} a_i(\cdot) \chi_i(\cdot)$ and $Y(\cdot) = \sum_{1}^{m_2} b_j(\cdot) \mu_j(\cdot)$ on $\tilde{\Omega}$  and we deduce by bilinearity of the Lie brackets that 
$[X, Y](\cdot) =  \sum_{i, j} [a_i \chi_i, b_j \mu_j](\cdot)$ on $\tilde{\Omega}$.
Moreover, one has:

\eq{ \label{eq:relations_brackets}
   [a_i \chi_i, b_j \mu_j]
     = \underbrace{a_ib_j [\chi_i, \mu_j]}_{\in \mathcal{D}([ \mathcal{W}, \mathcal{W}_k])}
     +
    \underbrace{ b_j [(\partial a_i) \mu_j] \chi_i }_{\in \mathcal{D}(\mathcal{W})}- \underbrace{a_i [(\partial b_j) \chi_i] \mu_j}_{\in \mathcal{D}(\mathcal{W}_k))},
}
   where, due to dimensions, both $(\partial a_i) \mu_j$ and $(\partial b_j) \chi_i$ are smooth scalar-valued functions. Thus $[X, Y] \in \mathcal{D}\left([ \mathcal{W}, \mathcal{W}_k] + \mathcal{W}_k \right) = \mathcal{D}(\mathcal{W}_{k+1})$. 
Finally, $[\mathcal{D} \left(\mathcal{W} \right), \mathcal{D} \left(\mathcal{W}_k \right)] \subseteq  \mathcal{D} \left(\mathcal{W}_{k+1} \right)$ and thus $\mathcal{D}(\mathcal{W})_{k+1} \subseteq   \mathcal{D} \left(\mathcal{W}_{k+1} \right)$. 
We now show the converse inclusion. Let $X \in \mathcal{D}(\mathcal{W}_{k+1})$. There are smooth real-valued functions $a_i$ and $Y_i \in \mathcal{W}_{k+1}$, such that $X = \sum_i a_i Y_i$. As by definition: $\mathcal{W}_{k+1} = \mathcal{W}_{k} + [\mathcal{W}_{k}, \mathcal{W}]$, there exists $\chi_i, v_i \in \mathcal{W}_{k}, \mu_i \in \mathcal{W}$ such that: $Y_i = v_i + [\chi_i, \mu_i]$.
Thus $X = \sum_i a_i \left( v_i + [\chi_i, \mu_i] \right)$.
As $\mathcal{D}(\mathcal{W}_{k+1})$ is a linear space, it is enough to show that both $a_i v_i$ and  $a_i [\chi_i, \mu_i]$ are in $\mathcal{D}(\mathcal{W}_{k+1})$. As $v_i \in \mathcal{W}_{k} \subseteq \mathcal{W}_{k+1} $, one has directly that $a_i v_i \in \mathcal{D}(\mathcal{W}_{k+1})$. Finally by using again  equality \eqref{eq:relations_brackets} (with $b_i \equiv 1$ and $j=i$), one has:
$$
a_i [\chi_i, \mu_i] = [ a_i \chi_i, \mu_i] - (\partial a_i \mu_i)\chi_i.
$$
Both $a_i \chi_i$ and $(\partial a_i \mu_i) \chi_i$ are in $\mathcal{D}(\mathcal{W}_{k}) \stackrel{\eqref{eq:recurrence}}{=}  \mathcal{D}(\mathcal{W})_k$ since $a_i$ and $\partial a_i \mu_i$ are smooth real-valued functions. We also have $\mu_i \in \mathcal{W} \subseteq \mathcal{D}(\mathcal{W})$ hence, using the characterization \eqref{eq:rec_lie}, one has $a_i [\chi_i, \mu_i]  \in \mathcal{D}(\mathcal{W})_{k+1}$ and thus $X \in \mathcal{D}(\mathcal{W})_{k+1}$, which concludes the recursion.

\textbf{We now prove that $\mathcal{D}(\lie ( \mathcal{D}(\mathcal{W}))) = \lie ( \mathcal{D}(\mathcal{W}))$}.
One inclusion is trivial so we only need to prove the other one. 
Let $X \in \lie ( \mathcal{D}(\mathcal{W}))$ and let us consider $a$ a smooth real-valued function. To conclude, we need to show that $aX \in \lie ( \mathcal{D}(\mathcal{W}))$.
By definition of the generated Lie algebra \Cref{buildingliealgebra}, there exists $k$ such that $X \in \mathcal{D}(\mathcal{W})_k $. Then, by using \eqref{eq:recurrence}, one has  $X \in \mathcal{D}(\mathcal{W}_k)$ and thus  $a X \in \mathcal{D}(\mathcal{W}_k) \stackrel{\eqref{eq:recurrence}}{=}  \mathcal{D}(\mathcal{W})_k \subseteq \lie ( \mathcal{D}(\mathcal{W}))$. 

\textbf{Finally, we now prove that for all $\alpha \in \tilde{\Omega}$, $
    [\lie(\mathcal{W}) ](\alpha) = [\lie ( \mathcal{D}(\mathcal{W}))](\alpha)$.}
Let $\alpha \in \tilde{\Omega}$.
  By using \eqref{eq:recurrence}, one has for all $k \in \mathbb{N}$, 
    $
    [\mathcal{D}(\mathcal{W})_k](\alpha) = [\mathcal{D} \left(\mathcal{W}_k \right)](\alpha) = [ \mathcal{W}_k ](\alpha).
    $
Then $
[\lie ( \mathcal{D}(\mathcal{W}))](\alpha) = [\lie(\mathcal{W})] (\alpha) $, which concludes the proof.
\end{proof}
%
\subsection{Proof of~\Cref{theorem:dimliealgebra}}

{First, we recall the statement of the theorem for the reader's convenience.}

\dimliealgebra*

{The proof of this theorem relies on \Cref{thm:algebrelie}: we compute $\lie (\mathcal{W}_{\phi}^{\mathtt{mom}})(\alpha)$ with $\phi = \phi_{\mathtt{Lin}}$ and show that its dimension is locally constant around $\alpha = (t,\x,\dot\x)$ where, $\x = \mathtt{vec}((U;V))$ and $\dot \x = \mathtt{vec}((\dot U;\dot V))$ are vectorized versions of two $(n+m)\times r$ matrices such that $(U;V; \dot U; \dot V) \in \R^{2(n+m)\times r}$ has full rank.}

{{\bf First we make $\mathcal{W}_{\phi}^{\mathtt{mom}} $ more explicit.} Recall that $\mathcal{W}_{\phi}^{\mathtt{mom}} $ is the functional linear space spanned by the functions $\alpha \mapsto \chi_{i}(\alpha)$ defined in~\eqref{eq:v-phi}. Since we consider the Euclidean geometry, we have $M = \mId{D}$ so that any function $\chi \in \mathcal{W}_{\phi}^{\mathtt{mom}}$ is a linear combination of $\chi_{i}(\cdot)$, $0 \leq i \leq d$, hence it satisfies
\[
\chi(\alpha) = \beta \chi_0(\alpha)+\sum_{i = 1}^{d} \gamma_{i}\chi_{i}(\alpha) = 
\begin{pmatrix}
 \beta\\   
\beta \dot \x\\
[\partial \phi(\x)]^\top \gamma - \tau(t) \beta \dot \x
 \end{pmatrix}
\]
for some $\beta \in \R$ and $\gamma \coloneqq (\gamma_i)_{i=1}^d \in \R^d$. Since $d = nm$, we can write $\gamma = \mathtt{vec}(\Delta)$ with $\Delta \in \R^{n \times m}$, and }
leveraging \citep[Proposition H.2]{marcotte2023abide}, we obtain 
$
\partial \phi(\x)^{\top}\gamma = \partial \phi(\x)^{\top}\mathtt{vec}(\Delta) = \mathtt{vec}\left(S_\Delta (U;V)\right)
$
where
    $S_\Delta \coloneqq \begin{pmatrix}
       0 & \Delta \\
       \Delta^\top & 0
   \end{pmatrix} \in \R^{(n+m)\times (n+m)}$.
Using a basic property\footnote{this property can indeed serve as an {\em operational definition} of the Kronecker product between matrices.} of Kronecker products ($(\mathbf{A} \otimes \mathbf{B})\mathtt{vec}(\mathbf{X}) = \mathtt{vec}(\mathbf{B}\mathbf{X}\mathbf{A}^\top)$), this is further rewritten as
$\partial \phi(\x)^{\top}\gamma = \mathtt{vec}\left(S_\Delta (U;V)\mId{r}\right) = (\mId{r} \otimes S_\Delta) \x$. Overall we obtain that $\mathcal{W}_{\phi}^{\mathtt{mom}}$ is the collection of all vector fields
\[
\chi_{\beta,\Delta}(\alpha) \coloneqq
\begin{pmatrix}
    \beta\\
    \beta \dot \x\\
    (\mId{r} \otimes S_\Delta) \x - \tau(t)\beta \dot \x
\end{pmatrix}.
\]
{\bf Second, we express} $\lie(\mathcal{W}_{\phi}^{\mathtt{mom}})(\alpha)$. \textcolor{purple}{We highlight in purple the results and reasoning steps specific to $n = m = 1$.}

\begin{proposition} \label{lieW}
Denote $\mathcal{S}_{\ell} \subset \R^{\ell\times \ell}$ the space of symmetric matrices, $H \coloneqq   \begin{pmatrix}  0 & \mId{n+m}\\ -\mId{n+m} & 0            \end{pmatrix}$, and for any square matrix $M$ of size $2(n+m)$ and $\beta \in \R$ denote
\[
\eta_{\beta,M}(\alpha) \coloneqq 
\begin{pmatrix}
0\\
[\mId{r} \otimes (HM)] \cdot \begin{pmatrix}\x\\ \dot \x\end{pmatrix}
\end{pmatrix}
+\beta \begin{pmatrix}
1\\
0_{(n+m)r}\\
-\tau(t)\dot\x
\end{pmatrix}.
\]
We have $\lie(\mathcal{W}_{\phi}^{\mathtt{mom}})(\alpha) =   \{\eta_{\beta,M}(\alpha): \beta \in \R, M \in \mathcal{S}\}$ where $\mathcal{S} \coloneqq \mathcal{S}_{2(n+m)}$ when $(n,m) \neq (1,1)$, while
{\color{purple}{for $n = m = 1$, $\mathcal{S} \coloneqq \tilde{\mathcal{S}}_{4} \coloneqq \Big\{ \begin{pmatrix}
       S_1  & S \\
       S & S_2
   \end{pmatrix}: S, S_1, S_2  \in \mathcal{S}_2'  \Big\} \subsetneq \mathcal{S}_{4}$, where $\mathcal{S}_2' \subsetneq \mathcal{S}_{2}$ is the set of symmetric matrices on the form $\begin{pmatrix}
       a & b \\
       b & a
   \end{pmatrix}$.}}
    \end{proposition}

\begin{proof}
{\bf First we show that $\mathcal{D}(\W^{\mathtt{mom}}) \subseteq \mathcal{D}(\mathcal{W})$,} where $ \mathcal{W} \coloneqq \{\eta_{
\beta,M}: \beta \in \R, M \in \mathcal{S}\}$ and with the operator $\mathcal{D}$ defined by \eqref{eq:operator_D}. 
One has (see \Cref{subsection:groscalculs} for more details):
\eq{  \label{eq:lien_chi_et_mu}
\chi_{\beta,\Delta} = \eta_{\beta,M}, \text{ where } M \coloneqq \begin{pmatrix}-S_{\Delta} & 0\\
0 & \beta \mId{(n+m)}\end{pmatrix}.
}
The matrix $M$ is symmetric and belongs to $\mathcal{S}$ even when $n=m=1$.
Thus $\W^{\mathtt{mom}}\subseteq \mathcal{W}$ and $\mathcal{D}(\W^{\mathtt{mom}}) \subseteq \mathcal{D}(\mathcal{W})$.  

{\bf We now prove that $\mathcal{D}(\mathcal{W})$ is a Lie algebra.} 
Since for all $a, b$ smooth real-valued functions and for all $X, Y \in \mathcal{W}$ one has:
 $$
   [a X, b Y]
     = ab [X, Y]
     +
    \underbrace{ b (\partial a Y) X }_{\in \mathcal{D}(\mathcal{W})}- \underbrace{a (\partial b X) Y}_{\in \mathcal{D}(\mathcal{W})},
     $$ 
   where, due to dimensions, both $(\partial a) Y$ and $(\partial b)X$ are smooth scalar-valued functions, it is enough to show that $[\mathcal{W}, \mathcal{W}] \subseteq \mathcal{D}(\mathcal{W})$. Moreover, since $\eta_{\beta,M} = \eta_{0,M}+\beta \eta_{1,0}$, 
it is enough to check that  $[\eta_{0,M},\eta_{0,M'}]$ and $[\eta_{1,0},\eta_{0,M}]$ are elements of $\mathcal{D}(\mathcal{W})$ whenever $M,M' \in \mathcal{S}$. 
\begin{itemize}
\item We first obtain (see \Cref{subsection:groscalculs} for more details) that
\eq{ \label{eq:liebracketdeuxM}
[\eta_{0,M},\eta_{0,M'}] = \eta_{0,M''},  \text{ where } M'' \coloneqq MHM'-M'HM.
}
Since $H^{\top}=-H$ it is straightforward to check that $M'' \in \mathcal{S}_{2(n+m)}$, hence $M'' \in \mathcal{S}$ when $(n,m) \neq (1,1)$, and we let the reader check that when $(n,m) = (1,1)$ we also have $M'' \in \mathcal{S}$ as soon as $M,M' \in \mathcal{S}$ (note that $\mathcal{S}_2'$ is stable by matrix multiplication and is commutative). Therefore $[\eta_{0,M},\eta_{0,M'}] = \eta_{0,M''} \in \mathcal{W} \subseteq \mathcal{D}( \mathcal{W})$. 
\item We now show that $[\eta_{0,M},\eta_{1,0}]
\in \mathcal{D}(\mathcal{W})$ for every $M \in \mathcal{S}$. Since $M \in \mathcal{S}$ we can write it as $M = \begin{pmatrix}
    S_1 & S \\
    S^\top & S_2
\end{pmatrix}$ with $S \in \R^{(n+m) \times (n+m)}, S_1, S_2 \in \mathcal{S}_{n+m}$ (when $n=m=1$
we further have $S,S_{1},S_{2} \in \mathcal{S}'_{2}$).
We then obtain (see \Cref{subsection:groscalculs} for more details):
\eq{ \label{eq:liebracketmixé}
[\eta_{0,M},\eta_{1,0}](\alpha) = \tau(t) \eta_{0,M'} (\alpha) \text{ with } M' \coloneqq \begin{pmatrix}
    S_1 & 0 \\
    0 & -S_2
\end{pmatrix}.
}

Since $t \mapsto \tau(t)$ is $\mathcal{C}^\infty$ and $M' \in \mathcal{S}$ (even when $n=m=1$)  this implies  $[\eta_{0,M},\eta_{1,0}] \in \mathcal{D}(\mathcal{W})$.
\end{itemize}
This establishes as claimed that $\mathcal{D}(\mathcal{W})$ is indeed a Lie algebra.

{\bf Finally we prove that $\mathcal{W} \subseteq \lie( \mathcal{D}( \mathcal{W}_{\phi}^{\mathtt{mom}}))$}. 
 This is where the necessity to impose the more restricted definition of $\mathcal{S}$ for $n=m=1$ will become evident.  Before proving this inclusion, observe that by~\Cref{lemma:traces_lie_equal} it will imply $\mathcal{D}(\mathcal{W}) \subseteq \mathcal{D}\left(\lie( \mathcal{D}( \mathcal{W}_{\phi}^{\mathtt{mom}})) \right) = \lie( \mathcal{D}( \mathcal{W}_{\phi}^{\mathtt{mom}}))$, hence combined with what we already proved it implies $\mathcal{D}(\mathcal{W}) = \lie( \mathcal{D}( \mathcal{W}_{\phi}^{\mathtt{mom}}))$.

As a shorthand denote $\mathcal{V} \coloneqq \mathcal{D}(\mathcal{W}_{\phi}^{\mathtt{mom}})$. First, we will  building matrix sets $\mathcal{M}_k,\mathcal{M}'_k$ such that $\eta_{0,M} \in \mathcal{V}_k$ for every 
$M \in \mathcal{M}_k \cup \mathcal{M}'_k$. 
Then we will show that $\mathcal{W} \subseteq  \{\eta_{0,M},  M \in \mathcal{M}\} + {\mathcal{V}}$
where $\mathcal{M}$ is the linear span of all built matrix sets. 
Since $\mathcal{V}_k \subseteq \lie(\mathcal{V})$ for every $k$, this will yield the desired conclusion.

$\bullet$ \textit{ We first prove that $\eta_{0,M} \in \mathcal{V}_k$ for every $M \in \mathcal{M}_k$.}

{\bf \em  The set $\mathcal{M}_0$} is defined as the collection of all matrices that write as 
\[
M = \begin{pmatrix}
    -S_\Delta & 0\\
    0 & 0
\end{pmatrix}
\]
for some $\Delta$. By~\eqref{eq:lien_chi_et_mu}, for any $\Delta$ we have $\eta_{0, M} = \chi_{0,\Delta} \in \mathcal{V}_0$.

{\bf \em  The set $\mathcal{M}_1$} is defined as the linear span of $\mathcal{M}_0$ and of the set of all matrices that write as
\[
M = \begin{pmatrix}  0 & S_{\Delta} \\
   S_{\Delta} & 0 \end{pmatrix}.
   \]
   
Consider $M_1 \coloneqq  \begin{pmatrix} -S_{\Delta}&0\\0 &  0 \end{pmatrix} \in \mathcal{M}_0$
and 
$
M_2 \coloneqq  \begin{pmatrix} 0 &0\\0 &  \mId{n+m}\end{pmatrix}
$.
As $\eta_{1,M_2} = \eta_{0, M_2} + \eta_{1, 0}$, by bilinearity of Lie brackets:
$$
[\chi_{0,\Delta},\chi_{1,0}](\alpha) 
\stackrel{\eqref{eq:lien_chi_et_mu}}{=} [\eta_{0,M_1},\eta_{1,M_2}](\alpha)
= [\eta_{0,M_1}, \eta_{0, M_2} ] (\alpha)+ [\eta_{0,M_1}, \eta_{1, 0} ] (\alpha) \stackrel{ \eqref{eq:liebracketdeuxM},\eqref{eq:liebracketmixé}}{=} \eta_{0, M_3}(\alpha) + \tau(t) \eta_{0, M_4}(\alpha) ,
$$
where $M_3 \stackrel{\eqref{eq:liebracketdeuxM}}{\coloneqq} M_1 H  M_2 - M_2 H M_1 = \begin{pmatrix}  0 & -S_{\Delta} \\
   - S_{\Delta} & 0 \end{pmatrix}$ and $M_4 \stackrel{\eqref{eq:liebracketmixé}}{=} M_1 \in \mathcal{M}_0$.
Thus:
$
\eta_{0, M_3} (\alpha) = [\chi_{0,\Delta},\chi_{1,0}](\alpha) - \tau(t) \eta_{0, M_4}(\alpha) \in \mathcal{V}_1$ since $t \mapsto \tau(t)$ is $\mathcal{C}^\infty$.

{\bf \em  The set $\mathcal{M}_2$} is defined as the linear span of $\mathcal{M}_1$ and of the set of matrices
\[
M = \begin{pmatrix}  0 & 0 \\
   0 & S_{\Delta} \end{pmatrix}.
   \]
Consider $M_1 \coloneqq \begin{pmatrix}  0 &S_{\Delta} \\
    S_{\Delta} & 0 \end{pmatrix} \in \mathcal{M}_1$ and
 $M_2 \coloneqq  \begin{pmatrix} 0 &0\\0 &  \mId{n+m}\end{pmatrix}$. 
Since $\eta_{0,M_1} \in \mathcal{V}_1$ and $\chi_{1,0} \in \mathcal{V}$, we have
$[\eta_{0,M_1},\eta_{1,M_2}] 
\stackrel{\eqref{eq:lien_chi_et_mu}}{=} [\eta_{0,M_1} , \chi_{1, 0}] \in  \mathcal{V}_2
$
and as:
$[\eta_{0,M_1},\eta_{1,M_2}](\alpha)  =  [\eta_{0,M_1},\eta_{0,M_2}](\alpha) +  [\eta_{0,M_1},\eta_{1,0}](\alpha) \stackrel{\eqref{eq:liebracketdeuxM}, \eqref{eq:liebracketmixé}}{=} \eta_{0, M_3}(\alpha) + \tau(t) \eta_{0, M_4}(\alpha), 
$
where 
$M_3 \stackrel{\eqref{eq:liebracketdeuxM}}{\coloneqq} M_1 H  M_2 - M_2 H M_1 = -2 \begin{pmatrix}
         0 &  0\\
         0 & S_\Delta 
    \end{pmatrix}$ and $M_4 \stackrel{\eqref{eq:liebracketmixé}}{=} 0$.
Thus $\eta_{0,M_4}=0$ and we obtain $\eta_{0, M_3} = [\eta_{0,M_1} , \chi_{1, 0}]  \in  \mathcal{V}_2$.

{\bf \em The set $\mathcal{M}_3$} is defined as the linear span of $\mathcal{M}_2$ and of the set of matrices 
\[
M = \begin{pmatrix}  0 & S_\Delta S_{\Delta'}  \\
   S_{\Delta'}S_\Delta & 0 \end{pmatrix},
   \]
   for any $\Delta, \Delta'$.
Such a matrix satisfies $M = M_2 H M_1-M_1HM_2$ where $M_1 ,M_2 
$ are specified as $M_1 \coloneqq  \begin{pmatrix} 0 & 0\\0 & S_\Delta \end{pmatrix} \in \mathcal{M}_2$
and 
$
M_2 \coloneqq  \begin{pmatrix} S_{\Delta'} &0\\0 &  0\end{pmatrix} \in \mathcal{M}_0.
$
By~\eqref{eq:liebracketdeuxM} it follows that $\eta_{0,M} = [\eta_{0,M_2},\eta_{0,M_1}] \in \mathcal{V}_3$.

Before defining $\mathcal{M}_5$ we further explicit matrices contained in $\mathcal{M}_3$. 
{Given any $\Delta, \Delta'  \in \R^{n \times m}$, denote $B_1 \coloneqq \Delta {\Delta'}^\top$ and $B_2 \coloneqq \Delta^\top {\Delta'}$. Since $S_\Delta S_{\Delta'}  = \begin{pmatrix}
   \Delta {\Delta'}^\top  & 0 \\
   0 & \Delta^\top {\Delta'}\end{pmatrix}$
   we have}

\eq{ \label{eq:diagenpartie}
\begin{pmatrix}
   0 & 0 & B_1 & 0 \\
  0  &  0 & 0 & B_2 \\
  B_1^\top & 0 & 0 & 0 \\
0 & B_2^\top & 0 & 0
\end{pmatrix} \in \mathcal{M}_3,
}
{for any pair of matrices $B_1,B_2$ that can be written as above. We explicit a few such matrices. }

{\color{purple} In the case $n=m=1$, $B_1=B_2 = b \in \R$, and as $\begin{pmatrix}  0 & S_{\Delta} \\
   S_{\Delta} & 0 \end{pmatrix} \in \mathcal{M}_1$ {and $S_\Delta = \begin{pmatrix}
       0 & a\\ a& 0
   \end{pmatrix}$ with $a \in \R$}, one has \eq{ \label{eq:diagcasnm1}
\begin{pmatrix}
0 & B \\
  B & 0 
\end{pmatrix} \in \mathcal{M}_3, \text{ for any } B \in \mathcal{S}'_2.
}}

When $(n,m) \neq (1,1)$, consider any $1 \leq i,k \leq n$, $1 \leq j, l \leq m$, and $\Delta \coloneqq E_{i, j} \in \R^{n \times m}$ and $\Delta' \coloneqq E_{k, l}  \in \R^{n \times m}$. Since
$B_1 \coloneqq \Delta {\Delta'}^\top = E_{i, j} E_{l, k} = \delta_{j, l} E_{i,k} \in \R^{n \times n}$ and $B_2 \coloneqq \Delta^\top {\Delta'} = E_{j, i} E_{k, l} = \delta_{i, k} E_{j, l} \in \R^{m \times m}$, we can reach:
\begin{itemize}
    \item $(B_1,B_2) = (E_{i, i},E_{j,j})$, for any $1 \leq i \leq n$ and $1 \leq j \leq m$, by choosing $k := i$ and $l := j$;
    \item  (in the case $m > 1$) $(B_1,B_2)= (0,E_{j,l})$ for any $1 \leq j \neq l \leq m$, by choosing e.g. $i=k=1$;
    \item (in the case $ n> 1$): $(B_1,B_2) = (E_{i, k},0)$ for any  $1 \leq i \neq k \leq n$, by choosing $j = l=1$.
\end{itemize}
{\bf The two following steps are specific to the case $(n, m) \neq (1, 1)$ only.}

{\bf \em The set $\mathcal{M}_5$} (defined only when $(n,m) \neq (1,1)$ -- and we skip the definition of $\mathcal{M}_4$) is defined as the linear span of $\mathcal{M}_3$ and of the set of matrices
\begin{equation}
M = \begin{pmatrix}  0 & 0 & 0 & E_{i, j} \\
0&0& -E_{i, j}^\top & 0 \\
0 & -E_{i, j} & 0 & 0 \\
E_{i, j}^\top  & 0 & 0 & 0
\end{pmatrix}
\label{eq:M5}
\end{equation}
with 
$1 \leq i \leq n$ and $1 \leq j \leq m$. Since $(n,m) \neq (1,1)$, without loss of generality, assume that $n > 1$ (a similar construction can be done in the case $m > 1$).
Observe that $M = M_2 H M_1-M_1HM_2$ where $M_1, M_2 
$ are specified as $M_1 \coloneqq  \begin{pmatrix} 0 & S_\Delta \\S_\Delta & 0 \end{pmatrix} \in \mathcal{M}_1$ (NB: not $ \mathcal{M}_0$) with $\Delta := E_{i, j}$, 
and 
$
M_2 \coloneqq \begin{pmatrix}
   0 & 0 & E_{k, k} & 0 \\
  0  &  0 & 0 & E_{j, j} \\
  E_{k, k} & 0 & 0 & 0 \\
0 & E_{j, j} & 0 & 0
\end{pmatrix}  \in \mathcal{M}_3$ {(by \eqref{eq:diagenpartie})} for some $1 \leq k \neq i \leq n$ (such a choice of $k$ is possible since $n > 1$).
By~\eqref{eq:liebracketdeuxM} it follows (using Jacobi identity and the fact that $M_1 \in  \mathcal{M}_1$) that $\eta_{0,M} = [\eta_{0,M_2},\eta_{0,M_1}] \in \mathcal{V}_5$ (but not $\mathcal{V}_4$).

Again, before defining the set $\mathcal{M}_7$ we 
show that for any $\Delta_1\in \R^{n \times m}, \Delta_2 \in \R^{m \times n}$ we have
\eq{ \label{eq:contre_diag_final}
{M' \coloneqq} \begin{pmatrix}
   0 & 0 & 0 & \Delta_1 \\
  0  &  0 & \Delta_2 & 0 \\
  0 &  \Delta_2^\top & 0 & 0 \\
 \Delta_1^\top & 0 & 0 & 0
\end{pmatrix} \in \mathcal{M}_5.
}
Indeed
$
2M'
= \begin{pmatrix}
   0 & 0 & 0 & \Delta_1 \\
  0  &  0 & -\Delta_1^\top & 0 \\
  0 &  -\Delta_1 & 0 & 0 \\
 \Delta_1^\top & 0 & 0 & 0
\end{pmatrix}  +  \begin{pmatrix}  0 & S_{\Delta_1} \\
   S_{\Delta_1} & 0 \end{pmatrix} +  \begin{pmatrix}
   0 & 0 & 0 & -\Delta_2^\top \\
  0  &  0 & \Delta_2 & 0 \\
  0 &  \Delta_2^\top & 0 & 0 \\
 -\Delta_2 & 0 & 0 & 0
\end{pmatrix}  +  \begin{pmatrix}  0 & S_{\Delta_2^\top} \\
   S_{\Delta_2^\top} & 0 \end{pmatrix}
$
where the first and third terms are combinations of matrices shaped as \eqref{eq:M5}, while the second and last belong to $\mathcal{M}_1$.

{\bf \em The set $\mathcal{M}_7$} (defined only when $(n,m) \neq (1,1)$ -- and again we skip the definition of $\mathcal{M}_6$)  is defined as the linear span of $\mathcal{M}_5$ and of the set of matrices
\begin{equation}
\label{eq:M7}
M = \begin{pmatrix}  0 & 0 & E_{i, i} & 0 \\
0&0& 0 &-E_{j, j}  \\
 E_{i, i} & 0 & 0 & 0 \\
0 & -E_{j, j}  & 0 & 0
\end{pmatrix}
\end{equation}
with $1 \leq i \leq n$ and $1 \leq j \leq m$. Observe that $M = M_2 H M_1-M_1HM_2$ where $M_1 \coloneqq  \begin{pmatrix}0 &  S_\Delta\\  S_\Delta& 0 \end{pmatrix} \in \mathcal{M}_1$ (NB: not   $\mathcal{M}_0$) with $\Delta = E_{i, j}$, 
and 
$
{M_2 \coloneqq} \begin{pmatrix}
   0 & 0 & 0 & 0 \\
  0  &  0 & E_{j,i} & 0 \\
  0 &  E_{j, i}^\top & 0 & 0 \\
 0 & 0 & 0 & 0
\end{pmatrix} \in \mathcal{M}_5
$ {(by \eqref{eq:contre_diag_final})}.
By~\eqref{eq:liebracketdeuxM} it follows (using again Jacobi identity and the fact that $M_1 \in  \mathcal{M}_1$) that $\eta_{0,M} = [\eta_{0,M_2},\eta_{0,M_1}] \in \mathcal{V}_7$ (and not $\mathcal{V}_6$).

As $\mathcal{M}_7$ is a vector space and since we have already \eqref{eq:diagenpartie} with $(B_1,B_2) = (E_{ii},E_{jj})$, by linear combination with matrices shaped as in \eqref{eq:M7}  we obtain that any matrix shaped as in
\eqref{eq:diagenpartie} with arbitrary diagonal $B_1,B_2$ also belongs to $\mathcal{M}_7$. Arbitrary off-diagonal terms can be obtained by combining matrices shaped as in 
~\eqref{eq:diagenpartie} (if $m>1$, $B_1 = 0$ and $B_2 = E_{j ,l}$ and if $n>1$, $B_1 = E_{i, k}$ and $B_2 = 0$), and we obtain
\eq{ \label{eq:diag}
M = \begin{pmatrix}
   0 & 0 & B_1 & 0 \\
  0  &  0 & 0 & B_2 \\
  B_1^\top & 0 & 0 & 0 \\
0 & B_2^\top & 0 & 0
\end{pmatrix} \in \mathcal{M}_7, \text{ for each } B_1 \in \R^{n \times n}, B_2 \in \R^{m \times m}. 
}

Finally, combining \eqref{eq:diag} and \eqref{eq:contre_diag_final} one has
\eq{ \label{eq:anti_diagonal}
  \begin{pmatrix}
    0 & B \\
   B^\top  & 0
 \end{pmatrix} \in \mathcal{M}_7
 \text{ for any } B \in \R^{(n + m)\times (n+m)}.
 }

{\bf \em  The set 
{$\mathcal{M}_8$}
} {\color{purple}{(resp. $\mathcal{M}'_4$)}} is defined as  the linear span of 
{$\mathcal{M}_7$}
{\color{purple}{(resp. of 
{$\mathcal{M}_3$}
)}} and of the set of matrices 
\[
M =  \begin{pmatrix}
   \Delta_1 \Delta^\top + \Delta \Delta_1^\top & 0 & 0 & 0\\
  0 &  \Delta_2 \Delta + \Delta^\top \Delta_2^\top & 0 & 0 \\
  0 & 0 & 0 & 0 \\
0 & 0 & 0 & 0
\end{pmatrix}, \text{{\color{purple}{ (resp. with $\Delta_2 = \Delta_1^\top$)}} },
\]
and 
\[
M' =  \begin{pmatrix}
   0 & 0 & 0 & 0\\
  0 & 0 & 0 & 0 \\
  0 & 0 & \Delta \Delta_2 + \Delta_2^\top  \Delta^\top & 0 \\
0 & 0 & 0 &  \Delta^\top  \Delta_1 + \Delta_1^\top \Delta
\end{pmatrix}, \text{{\color{purple}{ (resp. with $\Delta_2 = \Delta_1^\top$)}} },
\]

which satisfy $M = M_2HM_1 - M_1HM_2$ and $M'= M_2'HM_1 - M_1HM_2'$ where 
$
M_1 \coloneqq  \begin{pmatrix}
   0 & 0 & 0 & \Delta_1 \\
  0  &  0 & \Delta_2 & 0 \\
  0 &  \Delta_2^\top & 0 & 0 \\
 \Delta_1^\top & 0 & 0 & 0
\end{pmatrix} \in \mathcal{M}_5$ \textcolor{purple}{(in the case $(n, m) = (1, 1)$, $M_1 = \begin{pmatrix}  0 &S_{\Delta_1} \\
    S_{\Delta_1} & 0 \end{pmatrix} \in \mathcal{M}_1 $)},  $M_2 \coloneqq  \begin{pmatrix}S_{\Delta} & 0\\
0 & 0 \end{pmatrix} \in \mathcal{M}_0
$ and 
$
M_2' \coloneqq -\begin{pmatrix} 0 & 0\\
0 & S_{\Delta} \end{pmatrix} \in 
 \mathcal{M}_2. 
$

By~\eqref{eq:liebracketdeuxM} it directly follows that $\eta_{0,M} = [\eta_{0,M_2},\eta_{0,M_1}] \in 
\mathcal{V}_6 \subseteq \mathcal{V}_8$ \textcolor{purple}{(resp. $\eta_{0,M} \in 
\mathcal{V}_2 \subseteq \mathcal{V}_4$)}. 
Similarly, using Jacobi identity and the fact that $M'_2 \in \mathcal{M}_2$ \textcolor{purple}{(resp. that $M_1 \in \mathcal{M}_1$)}
we obtain that $\eta_{0,M'} = [\eta_{0,M_2'},\eta_{0,M_1}] \in 
\mathcal{V}_8$ 
\textcolor{purple}{(resp. $\eta_{0,M'} \in 
\mathcal{V}_4$).}
%

Again, we now explicit matrices belonging to $\mathcal{M}_8$ \textcolor{purple}{(resp. to $\mathcal{M}'_4$)}.

By considering $M_3 \coloneqq \begin{pmatrix}S_{\Delta''} & 0\\
0 & 0 \end{pmatrix} \in \mathcal{M}_0 
$, we obtain that 
$
M
+ M_3
=  \begin{pmatrix}
   S_1 & \Delta'' & 0 & 0 \\
  {\Delta''}^\top  &  S_2 & 0 & 0 \\
  0 & 0 & 0 & 0 \\
0 & 0 & 0 & 0
\end{pmatrix} \in \mathcal{M}_8' \text{{\color{purple}{ (resp. $\mathcal{M}_4'$)}} },
$
with $S_1 \coloneqq  \Delta_1 \Delta^\top + \Delta \Delta_1^\top \in \mathcal{S}_n$ and $S_2 \coloneqq  \Delta_2 \Delta + \Delta^\top \Delta_2^\top \in \mathcal{S}_m$ \textcolor{purple}{(resp. in the case $n = m = 1$, one has $S_1 = S_2$ as $\Delta_2 = \Delta_1^\top$)}, and since this holds for any choice of $\Delta,\Delta_1,\Delta_2,\Delta''$, one has
\eq{ \label{eq:coin_sup}
\begin{pmatrix}
S & 0 \\
    0 & 0
\end{pmatrix} \in \mathcal{M}_8'  \text{{\color{purple}{ (resp. $\mathcal{M}_4'$)}} }
\text{ for any } S \in \mathcal{S}_{n+m}  \text{{\color{purple}{ (resp. $\mathcal{S}_{n+m}'$)}} }.
}
Similarly, by considering $M_3' \coloneqq \begin{pmatrix}0 & 0\\
0 & S_{\Delta''} \end{pmatrix} \in \mathcal{M}_2 
$, we obtain that 
$
M'
+ M_3'
=  \begin{pmatrix}
  0&0 & 0 & 0 \\
 0&0 & 0 & 0 \\
  0 & 0 & S_1 & \Delta''\\
0 & 0 &  {\Delta''}^\top  &  S_2
\end{pmatrix} \in \mathcal{M}_8' \text{{\color{purple}{ (resp. $\mathcal{M}_4'$)}} },
$
with $S_1 \coloneqq  \Delta_1 \Delta^\top + \Delta \Delta_1^\top \in \mathcal{S}_n$ and $S_2 \coloneqq  \Delta_2 \Delta + \Delta^\top \Delta_2^\top \in \mathcal{S}_m$ \textcolor{purple}{(resp. in the case $n = m = 1$, one has $S_1 = S_2$ as $\Delta_2 = \Delta_1^\top$)}, and thus one has
\eq{ \label{eq:coin_inf}
\begin{pmatrix}
    0 & 0 \\
    0 & S
\end{pmatrix} \in \mathcal{M}'_8  \text{\color{purple}{ (resp. $\mathcal{M}'_{4}$)}}
\text{ for any } S \in \mathcal{S}_{n+m} \text{\color{purple}{ (resp. $\mathcal{S}'_{2}$)}}.
}
Thus by combining \eqref{eq:coin_sup} and \eqref{eq:coin_inf}, one has:
\eq{ \label{eq:big_diagonal}
  \begin{pmatrix}
   S_1 & 0  \\
  0 & S_2
\end{pmatrix} \in \mathcal{M}'_8 \text{\color{purple}{ (resp. $\mathcal{M}'_{4}$)}}
\text{ 
for each }  S_1, S_2 \in \mathcal{S}_{n+m} {\color{purple}{\text{ (resp. } 
\mathcal{S}_2')}}.
}

$\bullet$ \textit{ We now show that $\mathcal{W} \subseteq \{\eta_{0,M},  M \in \mathcal{M}\} + 
{\mathcal{V}}
$ where $\mathcal{M}$ is the linear span of all built matrix sets.}

First by combining \Cref{eq:anti_diagonal} and \Cref{eq:big_diagonal}, one has 
$   \{\eta_{0,M}, M \in \mathcal{S}\} \subseteq \{\eta_{0,M}, M \in \mathcal{M}\}$

For any $\beta \in \R$ one has:
$ \beta \eta_{1, 0}(\alpha) = \beta
\begin{pmatrix}
    1\\
    0 \\
   - \tau(t)\dot \x
\end{pmatrix} = 
\chi_{\beta, 0} (\alpha)- \beta \begin{pmatrix}
   0\\
     \dot \x\\
   0
\end{pmatrix} = \chi_{\beta, 0} (\alpha)-  \eta_{0, \beta M}(\alpha)$,
where $M \coloneqq\begin{pmatrix} 0 &0\\0 &  \mId{n+m}\end{pmatrix} \in 
\mathcal{M}$ {(by~\eqref{eq:coin_inf})}, and  $\chi_{\beta, 0} \in 
{\mathcal{V}}$. Therefore $\beta \eta_{1, 0} \in  \{\eta_{0,M},  M \in \mathcal{M}\} + 
{\mathcal{V}}
$. Then, for any $\beta$, $M \in \mathcal{S}$,  $\eta_{\beta, M} = \beta \eta_{1, 0} + \eta_{0, M} \in    \{\eta_{0,M},  M \in \mathcal{M}\} +
{\mathcal{V}}
$ so that: $\mathcal{W} \subseteq \{\eta_{0,M},  M \in \mathcal{M}\} +
{\mathcal{V}}
$.

$\bullet$ \textit{ Conclusion.}

Since $\mathcal{V}_k \subseteq \lie(\mathcal{V})$ for every $k$ and since $
\mathcal{V} \subseteq \lie(\mathcal{V} )$, we get  $\mathcal{W} \subseteq \{\eta_{0,M},  M \in \mathcal{M}\} + 
{\mathcal{V}} {\subseteq \cup_k \mathcal{V}_k}
\subseteq \lie (\mathcal{V}).$

\textbf{Conclusion.} By using \Cref{lemma:traces_lie_equal}, one has for all $\alpha$, $\lie\left(\mathcal{D}\left(\W^{\mathtt{mom}}\right)\right) (\alpha) = \lie\left(\W^{\mathtt{mom}}\right) (\alpha) $, and thus $\mathcal{W} (\alpha) = \mathcal{D}(\mathcal{W}) (\alpha) =  \lie\left(\mathcal{D}\left(\W^{\mathtt{mom}}\right)\right) (\alpha) =\lie\left(\W^{\mathtt{mom}}\right) (\alpha)$, which concludes the proof of \Cref{lieW}.
\end{proof}

Eventually, what we need to compute is the dimension of the trace $\lie(\W^{\mathtt{mom}}) (\alpha)$ for any $\alpha = (t, U, V, \dot{U}, \dot{V})$. 

\begin{proposition}  \label{dim-lie-algebra}
Consider $\alpha = (t, 
{\x, \dot \x}
)$ such that $\left(U; V;\dot{U}; \dot{V}\right) \in \R^{2(n+m) \times r}$ has full rank
{where $\x = \mathtt{vec}(U;V)$ and $\dot \x = \mathtt{vec}(\dot U; \dot V)$}. Then:
\begin{enumerate}
    \item if $2(n+m) \leq r$ and if $(n, m) \neq (1, 1)$, then $\vdim \lie(\W^{\mathtt{mom}}) \left(\alpha\right) = (n+m)(2(n+m)+1) +1$;
    \item if $2(n+m) > r$ and if $(n, m) \neq (1, 1)$, then $\vdim \lie(\W^{\mathtt{mom}}) \left(\alpha\right) = 2(n+m)r + 1 - r(r-1)/2$;
   { \color{purple}{\item if $(n, m) = (1, 1)$ and if $r \geq 4$, then $\vdim \lie(\W^{\mathtt{mom}}) \left(\alpha\right)= 6 + 1$.}}
\end{enumerate}
\end{proposition}
\begin{proof}
Let us consider the linear applications:
$$
\Gamma: M \in \mathcal{S} 
\mapsto
(\mId{r} \otimes (HM))\begin{pmatrix}\x\\\dot \x
    \end{pmatrix}
    = \mathtt{vec}\left(HM \begin{pmatrix}
        U\\
        V\\
        \dot U\\
        \dot V
    \end{pmatrix}\right)
\quad \text{and}\quad 
\bar{\Gamma}: M  \in \mathcal{S}
\mapsto \begin{pmatrix}
0\\
\Gamma(M)
\end{pmatrix}.
$$ 
By \Cref{lieW}, we have 
$\lie(\mathcal{W}_{\phi}^{\mathtt{mom}})
(\alpha) = \linspan\{\eta_{\beta,M}(\alpha): \beta \in \R, M \in \mathcal{S}\}$. 
By linearity of $(\beta,M) \mapsto \eta_{\beta,M}(\alpha)$ it follows that $\lie(\mathcal{W}_{\phi}^{\mathtt{mom}})(\alpha) = \R \eta_{1,0}(\alpha)+\bar{\Gamma}(\mathcal{S})$.
Since the first coordinate of $\eta_{1,0}(\alpha)$ is nonzero, it does not belong to $\bar{\Gamma}(\mathcal{S})$, hence $\vdim\ \lie(\mathcal{W}_{\phi}^{\mathtt{mom}})(\alpha) = 
\vdim (\bar{\Gamma} ( \mathcal{S})) +1
=
\vdim (\Gamma ( \mathcal{S})) +1 = \operatorname{rank}(\Gamma) + 1.
$
By the rank–nullity theorem, we have:
$\vdim\ \operatorname{ker}\ (\Gamma) + \operatorname{rank}\ (\Gamma) = \vdim\ \mathcal{S}. 
$
We now distinguish two cases.

\textit{1st case: $2(n+m) \leq r$.} Then as $H$ is invertible ($H^{-1} = -H$) and $\left(U; V;\dot{U}; \dot{V}\right)$ has full rank $2(n+m)$, $\Gamma$ is injective and we obtain $\operatorname{rank} (\Gamma)=  \vdim \mathcal{S}$. When 
$(n,m) \neq (1,1)$ we have $\mathcal{S} = \mathcal{S}_{2(n+m)}$ hence this yields $\operatorname{rank}(\Gamma) = \frac{2(n+m) [2(n+m)-1]}{2} = (n+m)[2(n+m)+1]$. 
\textcolor{purple}{In the case $n = m = 1$, {the assumption $2(n+m) \leq r$ reads $r \geq 4$, and} the associated rank is equal to $6$.}

\textit{2d case:  $2(n+m) > r$.} 
Since $H$ is invertible, $\operatorname{ker}( \Gamma)$ is the set of matrices $M \in \mathcal{S}$ such that $M (U;V; \dot U; \dot V)=0$. Denote $M_i$, $1 \leq i \leq 2(n+m)$ the rows of such a matrix, so that $M^\top = (M_1; \cdots; M_{2(n+m)})$. Denoting
 $C_j$, $1 \leq j \leq r$ the columns of $\left(U; V;\dot{U}; \dot{V}\right)$.
and $\mathcal{C} \coloneqq \underset{j = 1, \cdots, r}{\linspan} C_j$, we observe that since  $\left(U; V;\dot{U}; \dot{V}\right)$ has full rank $r = \min(2(m+n),r)$ the columns $C_j$ are linearly independent and $\vdim\ \mathcal{C} = r$.
Since $M\times \left(U; V;\dot{U}; \dot{V}\right) = 0$, we have $\langle M_i, C_j \rangle = 0$ for all $1 \leq i \leq 2(n+m)$ and $1 \leq j \leq  r$, i.e., each $M_i \in \R^{2(n+m)}$ belongs to $\mathcal{C}^\perp$, of dimension $\vdim\ \mathcal{C}^\perp = 2(n+m)-r$. 

To determine  $\vdim\ \operatorname{ker}(\Gamma)$ we now count the number of degrees of freedom to choose $M \in \mathcal{S}$ such that $M_i \in \mathcal{C}^\perp$ for every $i$. 
We only treat the case $(n,m) \neq (1,1)$, where $\mathcal{S}$ is simply the set of symmetric matrices characterized by $M^\top=M$.

We first show the following lemma.
\begin{lemma} \label{lemma:completion}
    The matrix $C \in \R^{2(n+m) \times r}$ has full rank $r$ if and only if there exists a subset $T$ of $2(n+m)-r$ indices such that the horizontal concatenation $\left(C, \mId{T}\right)$ is invertible, where $\mId{T} \in \R^{2(n+m) \times (2(n+m)-r)}$ is {the restriction of the identity matrix to its columns indexed by }$T$.
\end{lemma}

\begin{proof}
    The converse implication is clear. Let us show the direct one.
    By denoting $e_1, \cdots, e_{2(n+m)}$ the canonical basis in $\R^{2(n+m)}$, there is $i_1$ such that $e_{i_1}$ is linearly independent from all $C_j$: otherwise all $e_i$ would be spanned by $C_1, \cdots, C_r$, i.e. we would have ${\linspan} \{ e_i: 1 \leq i \leq 2(n+m)\}\subseteq \mathcal{C}$ hence  $2(n+m) \leq r$, which contradicts our assumption. 
    Similarly, by recursion, after finding $i_1, \cdots, i_{k}$ for some $k < 2(n+m)-r$ such that $i_1, \cdots, i_k$ are linearly independent from $C_1, \dots, C_r$ (so that $\tilde{\mathcal{C}} \coloneqq {\linspan} \{\mathcal{C} , e_{i_l}: 1\leq l \leq k \}$ has dimension $r+k < 2(n+m)$), there exists $i_{k+1}$ such that $e_{i_{k+1}}$ is linearly independent from all $C_j$ and all $e_{i_1}, \cdots, e_{i_k}$.  
    Stopping this construction when $k = 2(n+m)-r$ yields $T:= \{i_1,\ldots, i_k\}$.
\end{proof}
Consider the index set $T = \{i_1, \cdots, i_{2(n+m)-r} \}$ given by~\Cref{lemma:completion}, so that
$\left(C, \mId{T}\right) \in \R^{2(n+m) \times 2(n+m)}$ is invertible.

We first build the column $M_{i_1}$, which can be chosen arbitrarily in $\mathcal{C}^\perp$, a space of dimension $2(n+m)-r$.
Then, the $i_1$-th coordinate of $M_{i_2}$ is determined by $M_{i_1}$ (and equal to its $i_2$-th one) as $M$ is a symmetric matrix, and its remaining $2(n+m)-1$ coordinates can be freely chosen provided that
$M_{i_2}$ belongs to  $\mathcal{C}^\perp$. Thus, $M_{i_2}$ can be arbitrarily chosen in the affine space of dimension $2(n+m)-r-1$ defined by 
$$
\left(C, e_{i_1}\right)^\top M_{i_2} = \begin{pmatrix}
    0 \\ \cdots \\ 0 \\ M_{i_1}[i_2]
\end{pmatrix},
$$
where the matrix $\left(C, e_{i_1}\right)^\top \in \R^{(r+1) \times 2(n+m)}$ has full rank $r+1$ by construction.
By recursion, after building $k$ columns $M_{i_1}, \cdots, M_{i_k}$ with $k < 2(n+m)-r$, the coordinates indexed by $i_1, \cdots, i_k$ of the column $M_{i_{k+1}}$ are determined by $M_{i_1}, \cdots, M_{i_k}$ to ensure that $M$ is a symmetric matrix, and the remaining $2(n+m)-k$ coordinates must ensure that   $M_{i_{k+1}} \in \mathcal{C}^\perp$. Thus $M_{i_{k+1}}$ can be arbitrarily chosen in the affine space of dimension 
$2(n+m)-r-k$ defined by 
$$
\left(C, e_{i_1}, \cdots, e_{i_{k}} \right)^\top M_{i_{k+1}} = \begin{pmatrix}
    0 \\ \cdots \\ 0 \\ M_{i_1}[i_{k+1}] \\ \cdots \\ M_{i_k} [i_{k+1}]
\end{pmatrix},
$$
where the matrix $\left(C, e_{i_1}, \cdots, e_{i_{k}} \right)^\top\in \R^{(r+k) \times 2(n+m)}$ has full rank $r+k$ by construction.
Finally the dimension of $\mathrm{ker}(\Gamma)$ is equal to:
$$
\sum_{i=0}^{2(n+m)-r}(2(n+m)-r-i) = (2(n+m)-r)(2(n+m)-r+1)/2.
$$
Eventually we obtain $\operatorname{rank}(\Gamma)= 2(n+m)r -r(r-1)/2$.
\end{proof}

\subsection{Some derivations.} \label{subsection:groscalculs}

\paragraph{Details on how we obtained \eqref{eq:lien_chi_et_mu}.}
Given $\beta,\Delta$ the matrix $M \coloneqq \begin{pmatrix}-S_{\Delta} & 0\\
0 & \beta \mId{(n+m)}\end{pmatrix}$ is symmetric and belongs to $\mathcal{S}$ even when $n=m=1$, and $HM = \begin{pmatrix}0 & \beta \mId{n+m}\\ S_{\Delta} & 0\end{pmatrix}$ so that for any $\alpha = (t,\x,\dot \x)$ we have
\[
(\mId{r} \otimes HM) (\x;\dot \x) = \mathtt{vec}(HM (U;V;\dot U; \dot V))
=
\mathtt{vec}(
(\beta (\dot U; \dot V);S_{\Delta}(U;V)) =
(\beta \dot \x; (\mId{r} \otimes S_{\Delta})\x)
\]
and we deduce that $\chi_{\beta,\Delta} \in \mathcal{W}$ since
\[
\chi_{\beta,\Delta}(\alpha)
=
\begin{pmatrix}
\beta\\
\beta \dot \x\\
(\mId{r} \otimes S_{\Delta})\x-\tau(t)\beta \dot \x
\end{pmatrix} 
= 
\begin{pmatrix}0\\
(\mId{r}\otimes (HM)) \cdot \begin{pmatrix}\x\\\dot \x\end{pmatrix}
 \end{pmatrix}
+\beta
\begin{pmatrix}1\\
0\\
-\tau(t) \dot \x\end{pmatrix}
= \eta_{\beta,M}(\alpha).
\]

\paragraph{Details on how we obtained \eqref{eq:liebracketdeuxM}.}
Using that $(\mId{} \otimes \mathbf{A})(\mId{}\otimes \mathbf{B}) = \mId{} \otimes (\mathbf{A}\mathbf{B})$ we obtain that for any $\alpha$
\begin{align*}
\partial \eta_{0,M}(\alpha) \eta_{0,M'}(\alpha) 
&= 
\begin{pmatrix}0 & 0 \\ 0 & \mId{r} \otimes (HM)\end{pmatrix}
\begin{pmatrix} 0\\ [\mId{r} \otimes (HM')] \cdot 
	\begin{pmatrix}\x\\ \dot \x\end{pmatrix}
\end{pmatrix}
=
\begin{pmatrix}0\\ [\mId{r} \otimes (HMHM')]
	\begin{pmatrix}\x\\ \dot \x\end{pmatrix}
\end{pmatrix}\\
[\eta_{0,M},\eta_{0,M'}](\alpha) & = 
\begin{pmatrix}0\\
[\mId{r} \otimes (HMHM'-HM'HM)]\begin{pmatrix}\x\\ \dot \x\end{pmatrix}
\end{pmatrix} =
\begin{pmatrix}0\\
[\mId{r} \otimes \left(HM''\right)]\begin{pmatrix}\x\\ \dot \x\end{pmatrix}
\end{pmatrix}=\eta_{0,M''}(\alpha)
\end{align*}
with $M'' \coloneqq MHM'-M'HM$.

\paragraph{Details on how we obtained \eqref{eq:liebracketmixé}.}
Similarly one has
\begin{align*}
 \partial \eta_{1,0}(\alpha) \eta_{0,M}(\alpha) &=
 \begin{pmatrix}
0 & 0 & 0\\
0&0&0\\
-\tau'(t) \dot \x & 0 & -\tau(t)\mId{}
\end{pmatrix}
\begin{pmatrix} 0\\ [\mId{r} \otimes (HM)] \cdot 
	\begin{pmatrix}\x\\ \dot \x\end{pmatrix}
\end{pmatrix}
= 
\tau(t)\begin{pmatrix}
0\\
\begin{pmatrix}0 & 0 \\
0 &-  \mId{}\end{pmatrix}
 (\mId{r} \otimes (HM)) \cdot \begin{pmatrix}\x\\\dot\x\end{pmatrix}
\end{pmatrix}\\
\partial \eta_{0,M}(\alpha) \eta_{1,0}(\alpha) &=
 \begin{pmatrix}0 & 0 \\ 0 & \mId{r} \otimes (HM)\end{pmatrix}
\begin{pmatrix}
1\\
	\begin{pmatrix}0\\
		-\tau(t) \dot \x
	\end{pmatrix}
\end{pmatrix}
=
-\tau(t)\begin{pmatrix}
0\\
 (\mId{r} \otimes (HM)) 
 \begin{pmatrix}0 \\
 \dot \x\end{pmatrix}
 \end{pmatrix}, \text{ and thus,}
 \\
 [\eta_{1,0},\eta_{0,M}](\alpha) &= 
\tau(t)  
\begin{pmatrix}
0\\
\left[ 
	\begin{pmatrix}0 & 0\\
	0 &-\mId{}\end{pmatrix}
	(\mId{r} \otimes (HM)) 
	+(\mId{r} \otimes (HM)) 
	 \begin{pmatrix}0 & 0\\
	0 & \mId{} \end{pmatrix}
\right]	
  \begin{pmatrix}\x\\\dot\x\end{pmatrix}
  \end{pmatrix}.
\end{align*}
Since $M \in \mathcal{S}$ we can write it as $M = \begin{pmatrix}
    S_1 & S \\
    S^\top & S_2
\end{pmatrix}$ with $S \in \R^{(n+m) \times (n+m)}, S_1, S_2 \in \mathcal{S}_{n+m}$ (when $n=m=1$
we further have $S,S_{1},S_{2} \in \mathcal{S}'_{2}$). 
We now prove that 
\begin{equation}\label{eq:tmpkron}
\left[ 
	\begin{pmatrix}0 & 0\\
	0 &-\mId{}\end{pmatrix}
	(\mId{r} \otimes (HM)) 
	+(\mId{r} \otimes (HM)) 
	 \begin{pmatrix}0 & 0\\
	0 & \mId{} \end{pmatrix}
\right]
= \mId{r} \otimes (HM')
\end{equation}
with  $M' \coloneqq \begin{pmatrix}
    -S_1 & 0 \\
    0 & S_2
\end{pmatrix}$.
To establish~\eqref{eq:tmpkron}, 
denoting $\mathtt{mat}(\x) \coloneqq (U;V)$ and $\mathtt{mat}(\dot \x) \coloneqq (\dot U; \dot V)$ we compute 
\begin{align*}
HM &= \begin{pmatrix}S^{\top} & S_{2}\\ -S_{1} & -S
\end{pmatrix}\\
	\begin{pmatrix}0 & 0\\	0 &-\mId{}\end{pmatrix}
	(\mId{r} \otimes (HM))   
	\begin{pmatrix}\x\\\dot\x\end{pmatrix}
&=
	\begin{pmatrix}0 & 0\\	0 &-\mId{}\end{pmatrix}
 	\mathtt{vec}\left(HM 	
		\begin{pmatrix}\mathtt{mat}(\x)\\ \mathtt{mat}(\dot \x)
		\end{pmatrix}\right)
		\\
&=
	\begin{pmatrix}0 & 0\\	0 &-\mId{}\end{pmatrix}
	\begin{pmatrix}
		\mathtt{vec}\left(S^{\top}\mathtt{mat}(\x)
		+S_{2}\mathtt{mat}(\dot \x)\right)\\
		\mathtt{vec}\left(- S_{1}\mathtt{mat}(\x)
	-S\mathtt{mat}(\dot \x)\right)\\
	\end{pmatrix}
&=
	\begin{pmatrix}
	0\\
	\mathtt{vec}\left(S_{1}\mathtt{mat}(\x)
	+S\mathtt{mat}(\dot \x)\right)\\
	\end{pmatrix}\\
(\mId{r} \otimes (HM)) 
 \begin{pmatrix}0 & 0\\
0 & \mId{} \end{pmatrix}	
\begin{pmatrix}\x\\\dot\x\end{pmatrix}
&=
(\mId{r} \otimes (HM)) 
\begin{pmatrix}0\\\dot\x\end{pmatrix}
=
\mathtt{vec}\left(
HM \begin{pmatrix}0 \\
\mathtt{mat}(\dot\x)\end{pmatrix}
\right)
&=
\begin{pmatrix}
\mathtt{vec}\left(
	S_{2}\mathtt{mat}(\dot\x)
\right)
	\\
\mathtt{vec}\left(
	-S\mathtt{mat}(\dot\x)
\right)
\end{pmatrix}
\end{align*}
so that
\begin{align*}
\left[ 
	\begin{pmatrix}0 & 0\\
	0 &-\mId{}\end{pmatrix}
	(\mId{r} \otimes (HM)) 
	+(\mId{r} \otimes (HM)) 
	 \begin{pmatrix}0 & 0\\
	0 & \mId{} \end{pmatrix}
\right]\begin{pmatrix}\x\\\dot\x\end{pmatrix}
&=
\begin{pmatrix}
\mathtt{vec}\left(
	S_{2}\mathtt{mat}(\dot\x)
\right)
	\\
\mathtt{vec}\left(
	S_{1}\mathtt{mat}(\x)
\right)
\end{pmatrix}\\
&=
\mathtt{vec}\left(
\underbrace{\begin{pmatrix} 0 & S_{2}\\ S_{1} & 0\end{pmatrix}}_{=HM'}
\begin{pmatrix}\mathtt{mat}(\x)\\\mathtt{mat}(\dot\x)\end{pmatrix}
\right)
= (\mId{r} \otimes (HM')) \begin{pmatrix}\x\\\dot\x\end{pmatrix}.
\end{align*}

\section{Proof of \Cref{prop:lossnb}}\label{appendix:loss}

\lossnb*
\begin{proof}
Denoting $\mathtt{rk}_1 = \min(r,n+m)$ the rank of $(U; V) \in \R^{(n+m) \times r}$ and $\mathtt{rk}_2 = \min(r,2(n+m))$ the rank of $(U; V; \dot U; \dot V) \in \R^{2(n+m) \times r}$ and using \Cref{nb_GF} and \Cref{theorem:dimliealgebra}, one has
    $$
    N_{GF} \coloneqq \mathtt{rk}_1 /2 (2r + 1 - \mathtt{rk}_1) \text{ and } N_{MF} \coloneqq \mathtt{rk}_2 /2 (2r - 1 - \mathtt{rk}_2).
    $$
Thus: 
\begin{align*}
    N_{GF} - N_{MF} &= \frac{\mathtt{rk}_1}{2} (2r + 1 - \mathtt{rk}_1) - \frac{\mathtt{rk}_2}{2} (2r - 1 - \mathtt{rk}_2) \\
    &= - \frac{\mathtt{rk}_1^2 }{2} + \frac{\mathtt{rk}_2^2 }{2} + r (\mathtt{rk}_1  - \mathtt{rk}_2) + \frac{\mathtt{rk}_1 + \mathtt{rk}_2}{2} \\
    &={\tfrac{1}{2}}[(\mathtt{rk}_2 - r)^2 - (\mathtt{rk}_1 - r)^2 ]
    + \frac{\mathtt{rk}_1 + \mathtt{rk}_2}{2}. 
\end{align*}
We now distinguish 3 cases.

\textit{1st case: $r \leq n+m$.} Then $\mathtt{rk}_1= \mathtt{rk}_2 = r$, and thus $  N_{GF} - N_{MF} = \frac{\mathtt{rk}_1 + \mathtt{rk}_2}{2} = r > 0$.

\textit{2d case: $ n+m < r \leq  2(n+m)$.} Then $\mathtt{rk}_1= n+m$ and $\mathtt{rk}_2 = r$, and thus:
$
  N_{GF} - N_{MF} = \frac{-(n+m - r)^2}{2} +  \frac{r+n+m}{2}. 
$
Let us show that in that case, we always have: $ N_{GF} - N_{MF} \leq 0$.
Denoting $x \coloneqq n+m\geq 1$ and $y\coloneqq r \geq 2$, we have:
$N_{GF} - N_{MF} =  -(x-y)^2/2 + (x+y)/ 2= -{\tfrac{1}{2}} [x^2 -(1+2y)x+ y(y-1)] = -{\tfrac{1}{2}} (x-x_1) (x-x_2)$, where:
$x_1 x_2 = y(y-1)$ (and thus $x_1$ and $x_2$ have the same sign) and $x_1 + x_2 = 1 + 2y$ (and thus they are positive). By using the equality $x_1 x_2 = y(y-1)$, we necessarily have $x_1, x_2 \in [y-1, y]$. Thus as $x \leq y-1$ ($n, m, r \in \mathbb{N}$), we then have $N_{GF} - N_{MF} \leq 0$.

\textit{3rd case: $2(n+m) < r$.} Then $\mathtt{rk}_1= n+m$ and $\mathtt{rk}_2 = 2(n+m)$, and thus: 
\begin{align*}
  N_{GF} - N_{MF} &=   - \frac{\mathtt{rk}_1^2 }{2} + \frac{\mathtt{rk}_2^2 }{2} + r (\mathtt{rk}_1  - \mathtt{rk}_2) + \frac{\mathtt{rk}_1 + \mathtt{rk}_2}{2} \\
  &= (n+m) \left[ \tfrac{3}{2}(n+m) + \tfrac{3}{2} - r \right] \\
  & \leq 0 \text{ as } 2(n+m) + 1 \leq r.
\end{align*}
\end{proof}

\section{Proof of \Cref{thm:conservedfunctionsNMF}} \label{appendix:conservedfunctionsNMF}
 
\conservedfunctionsNMF*

\begin{proof}
One has:
$
    \frac{\mathrm{d}}{\mathrm{d}t} \left(1_n^\top U - 1_m^\top V \right) = 1_n^\top \dot U - 1_m^\top \dot V 
    \stackrel{{\eqref{gradientflow}}}{=} -1_n^\top \Big( U \odot \underbrace{\nabla_U \mathcal{E}_{Z}}_{\in \R^{n \times r}}\Big) + 1_m^\top \Big(V \odot \underbrace{\nabla_V \mathcal{E}_Z}_{ \in \R^{n \times r}}\Big).
$
Then by \eqref{b} one has:
$$
 \frac{\mathrm{d}}{\mathrm{d}t} \left(1_n^\top U - 1_m^\top V \right)  = - 1_n^\top \Big( U \odot \left( \nabla F(\phi(\x)) V \right) \Big) + 1_m^\top \Big(V \odot \left(\nabla F(\phi(\x))^\top U \right)\Big).
$$
Then by denoting $\nabla F(\phi(\x)) = (A_{i, j})_{i, j} \in \R^{n \times m}$, one has for all $j = 1, \cdots, r$:
\begin{align*}
   \left[  1_n^\top \Big( U \odot \left( \nabla F(\phi(\x)) V \right) \Big) - 1_m^\top \Big(V \odot \left(\nabla F(\phi(\x))^\top U \right)\Big)\right][j] &= \sum_{i = 1}^n  U_{i, j} \sum_{k=1}^m A_{i, k} V_{k, j} - \sum_{i = 1}^m  V_{i, j} \sum_{k=1}^n A_{k, i} U_{k, j} \\
   &=0.
\end{align*}
{This proves as claimed that for $1 \leq j \leq r$ the $j$-th column of $1_n^\top U - 1_m^\top V$ defines a conserved function $h_{j}(\x)$ for \eqref{gradientflow} with 
$M(\x)$. Since $h_{j}$ only depends on the corresponding columns $u_{j}, v_{j}$ 
of $U,V$, the gradients $\nabla h_{j}(\x)$, $\nabla h_{\ell})(\x)$ are orthogonal, $j \neq \ell$, hence these conserved functions are also independent. }
\end{proof}

\section{Proof of \Cref{thm:ICNN}} \label{appendix:icnn}
\thmICNN*

\begin{proof}
   {Consider the following extension of the linear transformation $T^{A}$ from \eqref{linear_transf} to cover the presence of biases, where by convention $U \in \R^{n\times r}$, $V \in \R^{m \times r}$, and $b \in \R^{1 \times r}$:}
$$
T^A(\epsilon,U, V, b) \coloneqq
   \Big(  U  \exp(\epsilon A), V \exp(-\epsilon A^\top), b \exp(-\epsilon A^\top )\Big).
$$
{Observe that as soon as } 
 $A \in \R^{r \times r}$ {is diagonal,} 
   $T^A$ is a linear transformation that leaves {$g(\cdot,x)$ invariant for each $x$, hence it also leaves} $\mathcal{E}_Z$ invariant.
 Moreover, 
   $\Delta_{T^A }(U, V, b) =  (U A, -V A^\top, - b A^\top)= (UA, -VA, -b A)$, as diagonal matrices are symmetric. Thus {by \eqref{eq:invarianceloss}}:
   $$
   \left\langle \nabla \mathcal{E}_Z (\x), \smallvec{UA  \\ -VA \\ -b A} \right\rangle = 0.
   $$
{Specializing to} 
 $A= E_{i, i} \in \R^{r \times r}$, the one-hot matrix with the $(i, i)$-th entry being 1, 
 {we obtain} 
   \begin{equation}\label{eq:TmpICNN}
   \langle \nabla_{U_i} \mathcal{E}_Z (\x), U_i \rangle  - \langle \nabla_{V_i} \mathcal{E}_Z (\x), V_i \rangle - \nabla_{b_i} \mathcal{E}_Z (\x)b_i  = 0.
   \end{equation}
{with $U_{i},V_{i}$ the columns of $U,V$}.   Finally, given any $j \in \{ 1, \ldots r \}$ we compute {using that $M(\x) = \diag[(U,1_{n\times r},1_{1 \times r})]$}
\begin{align*}
\frac{\mathrm{d}}{\mathrm{d}t} \left(1_n^\top U_j - \frac{1}{2} \left( \| V_j \|^2 + b_j^2 \right) \right) 
&= 1_n^\top \dot U_j - \left(\langle \dot V_j, V_j \rangle  +  b_j \dot b_j \right)\\
&    \stackrel{\eqref{gradientflow}}{=} 
- \underbrace{1_n^\top \Big( U_j  \odot \nabla_{U_j} \mathcal{E}_{Z}(\x)\Big)}_{= \langle \nabla_{U_j} \mathcal{E}_Z (\x), U_j \rangle } + \langle \nabla_{V_j} \mathcal{E}_{Z}(\x), V_j \rangle +  b_j \nabla_{b_j} \mathcal{E}_{Z}(\x)
\stackrel{\eqref{eq:TmpICNN}}{=}0.\qedhere
  \end{align*}
\end{proof}

\section{About the Natural gradient flow case {(cf \Cref{section:naturalgradient})}.} \label{appendix:naturalgradient}
Let $\x \in \Theta$. We consider $Z = (x_i, y_i)_i$ such that $g(\cdot,x_i)$ is $\mathcal{C}^{2}$ in a neighborhood of $\x$, and  the ODE:
\eq{\dot \x =  \underbrace{\Big( \frac{1}{n} \sum_i  [\partial_1 g(\x,x_i)]^\top \partial_1 g(\x,x_i)\Big)^\dagger}_{{=: M_{Z}(\x)}} \nabla \mathcal{E}_Z(\x),
}
where $A^{\dag}$ denotes the pseudo-inverse of $A$. We consider $\phi$ as in \Cref{as:main_assumption}.
Then {we deduce from \eqref{eq:elr-general} that $\partial_{1} g(\x,x_{i}) = \partial f_{1}(\phi(\x),x_{i}) \partial \phi(\x)$ for each $i$, so that}:
$$
{M_{z}(\x) = } 
\Big( \frac{1}{n} \sum_i  [\partial_1 g(\x,x_i)]^\top \partial_1 g(\x,x_i)\Big)^\dagger 
= \Big( \partial \phi(\x)^\top \frac{1}{n} \sum_i  \partial_1 f(\x,x_i)^\top \partial_1 f(\x,x_i) \partial \phi(\x) \Big)^\dagger.
$$
As $A^{\dag} = \lim_{ \delta \rightarrow 0} A^\top (A A^\top +  \delta I)^{-1}$, 
using the definitions of $W_{\x}^{\mathtt{grad}}$ and $\mathcal{W}_\phi(\x)$ (cf \Cref{proplinkGFnonEuclidean} and \eqref{eq:W_phi_grad}) we get
$W_{\x}^{\mathtt{grad}} \coloneqq 
\linspan_{Z \in \mathcal{Z}'_{\x}} \{ M_Z(\x) \nabla \mathcal{E}_Z(\x) \} \subseteq \linspan \{\nabla \phi_1(\x), \cdots, \nabla \phi_d(\x) \} =: \mathcal{W}_\phi(\x)$.
By \Cref{CL:nonEuclideanGF} and \eqref{eq:v-phi-GF}, any GF conservation law $h$ of $\phi$ for the Euclidean gradient flow \eqref{gradientflow} (i.e., with $M=\mId{}$) satisfies  $\nabla h(\x) \perp  \mathcal{W}_{\phi}(\x)$
for all 
$\x \in \Theta$, hence $\nabla h(\x) \perp W_{\x}^{\mathtt{grad}}$
  for every $\x \in \Theta$. 
  By \Cref{proplinkGFnonEuclidean}, this shows that $h$ is indeed locally conserved on $\Theta$ for any data set {\em through the natural gradient flow} \eqref{gradientflow} with $M_{Z}(\x)$.

\section{About Experiments with formal calculus}
\label{sec:sagemath}
Our code
is open-sourced and is available at \url{https://github.com/sibyllema/Conservation_laws_ICML}.
We used the software SageMath \cite{sagemath}, which relies on a Python interface. 
We compare the number of independent conservation laws given \Cref{thm:algebrelie}, with the number of independent polynomial conservation laws {found (with or without a change of variable)} as explained in \Cref{section:changeofvariable}.

\subsection{Euclidean MF} \label{sec:sagemathEMF} We first considered the case of Euclidean MF (the Euclidean GF has been studied in \cite{marcotte2023abide}) {from \Cref{section:linear,section:ReLU}}.
We tested both linear and ReLU neural networks (with and without biases) of various depths and widths, and observed that the two numbers matched in all our examples, with a constant $\tau$ either equal to $0$ or $1$.
For this, we 
{drew} $20$ random linear (resp. ReLU) neural networks, with depth drawn uniformly at random between $2$ to $4$ and i.i.d. layer widths drawn uniformly at random between $1$ to $4$, with a $\tau$ randomly chosen between $0$ or $1$. For ReLU architectures, the probability of including biases was $1/2$. Then we checked that the two numbers match. In particular, for ReLU neural networks, the number of conservation laws is always equal to zero.
\subsection{NMF} \label{sec:sagemathNMF} 
For the GF scenario (resp. the MF scenario) {with NMF from \Cref{section:NMF}}, we {drew} $20$ random $2$-layer linear neural networks, with i.i.d. layer widths drawn uniformly at random between $1$ to $8$ (resp. $1$ and $6$, with a $\tau$ randomly chosen between $0$ or $1$). We observed that the two numbers (number of independent conservation laws/number of independent ``polynomial'' conservation laws) matched in all our examples in the GF scenario and that there is no conservation law for the MF scenario.

\subsection{ICNN}\label{sec:sagemathICNN} 
For the GF scenario (resp. the MF scenario), we draw $20$ random $2$-layer ReLU neural networks (with and without biases), with i.i.d. layer widths drawn uniformly at random between $1$ to $6$ (resp. and with a $\tau$ randomly chosen between $0$ or $1$), with a probability of including biases of $1/2$. We observed that the two numbers (number of independent conservation laws/number of independent ``polynomial'' conservation laws) matched in all our examples in the GF scenario, and that there is no conservation law for the MF scenario.

\section{Numerical Simulation}
\label{sec:numerics}

In this section, we show numerical simulations on gradient flows and momentum flows to explore: (a) the influence of the time discretization on the preserved quantities, (b) the impact of momentum on the preservation of conservation laws for the gradient flows. Our code
is open-sourced and is available at \url{https://github.com/sibyllema/Conservation_laws_ICML}.

\subsection{Discretization of the flows}

To ease the description, we consider the following parameterization of the flows
$$
	\mu \ddot \theta + \nu \dot \theta = -{M}( \mu \dot \theta + \nu \theta ) \nabla \mathcal{E}_Z(\theta).
$$
Gradient flows correspond to $\mu=0$, while the momentum parameter is $\tau=\nu/\mu$. 
We consider the following time discretization of the flows, where time at step $k$ is $t = k \delta$ and $\delta>0$ is the time step
$$
    \mu \frac{\theta_{k+1} + \theta_{k-1} - 2 \theta_k}{\delta^2} 
    + \nu \frac{\theta_{k+1} - \theta_k }{\delta} = - {M}_k \nabla \mathcal{E}_Z(\theta_k)
    \quad\text{where}\quad
    {M}_k \coloneqq {M}\Big( \mu \frac{\theta_k-\theta_{k-1}}{\delta} + \nu \theta_k \Big)
$$
This can be re-written in the usual form of a gradient descent with momentum
$$
    \theta_{k+1} = \theta_k - \alpha {M}_k \nabla \mathcal{E}_Z(\theta_k)  + \beta (\theta_k-\theta_{k-1})
$$
where 
$$
    \alpha \coloneqq \frac{\delta}{\nu + \mu/\delta}
    \quad\text{and}\quad
    \beta \coloneqq    \frac{\mu}{ \delta \nu + \mu} < 1.
$$
Here $\beta \in [0,1)$ is the momentum (extrapolation) parameter, so that $\beta=0$ corresponds to usual gradient descent, and setting $\beta=1$ is maximum momentum (which is not in general ensured to converge). 

\subsection{MLP example}

We consider here a 3-layer MLP trained for classification on the MNIST dataset~\cite{lecun2010mnist} with the cross entropy loss function and a ReLU non-linearity. The input dimension is $28 \times 28 = 784$ (number of pixels), the inner layer dimensions are $(512, 256)$ and the output dimension is $10$ (number of classes). 
We focus on the conservation laws associated to 
{the $r=512$ neurons of the first two layers,}
and we denote $(U \in \mathbb{R}^{784 \times 512},V \in \mathbb{R}^{256  \times 512})$ the associated matrices, with associated columns neurons $(u_i \in \mathbb{R}^{784})_i$ and $(v_i \in \mathbb{R}^{256})_i$. The $r$ conserved quantities for the gradient flow, $\tau=+\infty$ are $\|u_i\|^2-\|v_i\|^2$ and there is no {{\em exactly} }preserved quantity for the momentum flow $\tau<+\infty$. 

\Cref{fig:mlp_laws}, left, shows the evolution of the loss for a range of step size $\delta$ up to almost no convergence, all with the same initialization of the weights. Note that despite the non-convexity of the loss function, the evolution converges to approximately the same loss value.  
%
\Cref{fig:mlp_laws}, 
right, shows one of the conservation laws (associated with the {neurons of the first layer}). One can see that even for relatively large step sizes, these quantities are almost perfectly conserved. It is only for step size on the edge of instabilities that these quantities are not well preserved. This validates the relevance of these conservation laws for the regime of the step size used for stable training of neural networks.
Figure~\ref{fig:mlp_laws_mu} shows how the evolution of the loss and the preserved quantities for GF is impacted by the momentum parameter $\mu=1/\tau$. As expected, increasing $\mu$ deteriorates the preservation of the conservation law.

\begin{figure}[h]
    \centering
    \includegraphics[width=.4\textwidth]{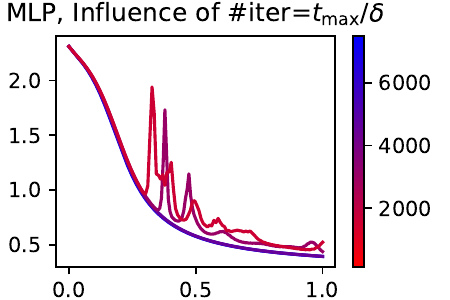} 
    \includegraphics[width=.4\textwidth]{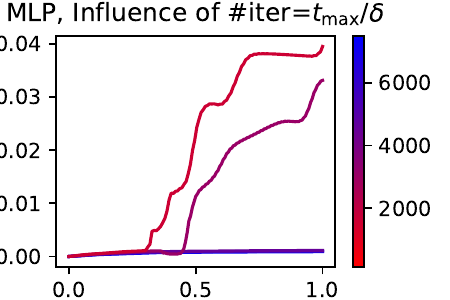} 
    \caption{Impact on the step size $\delta$ on the evolution of the loss (left) and on the preservation of one of the conservation laws (right).
    The colors are associated with the number of iterations $t_{\max}/\delta$ used to train the networks. 
    }
    \label{fig:mlp_laws}
\end{figure}

\begin{figure}[h]
    \centering
    \includegraphics[width=.4\textwidth]{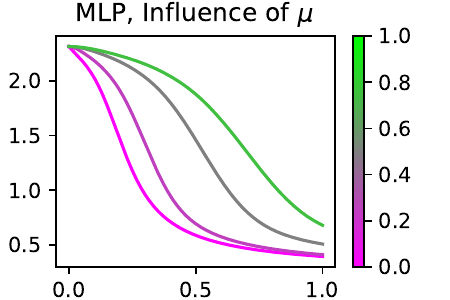} 
    \includegraphics[width=.4\textwidth]{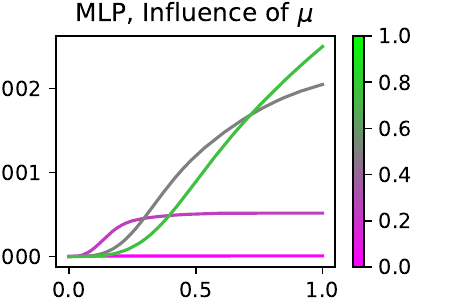} 
    \caption{Impact on the momentum parameter $\mu=1/\tau$ on the evolution of the loss (left) and on the preservation of one of the conservation laws (right).
    }
    \label{fig:mlp_laws_mu}
\end{figure}


\subsection{NMF example}

We consider here a non-negative matrix factorization so that the loss function is $\mathcal{E}_Z(U,V) = \|UV^\top - Y\|^2$ under positivity constraints. 
We thus use the metric of the mirror flow {associated to the Shannon potential function} ${M}_Z(t,\theta,\dot\theta)=\diag(\dot\theta+\tau \theta)$
The column of $Y \in \RR^{n \times p}$ are $p=6903$ images of $n=28 \times 28$ pixels from both the test and training sets of MNIST dataset~\cite{lecun2010mnist} associated to the digit 0, see Fig.~\ref{fig:mnf_data}, left.  We compute a factorization of rank $r=10$ so that $U \in \mathbb{R}^{n \times r}$ and $V \in \mathbb{R}^{p \times r}$. Figure~\ref{fig:mnf_data} shows examples of the $r$  factors (columns of $U$ displayed as positive images). Note that while the function is non-convex, in practice, gradient descent and momentum descent converge to global minimizers (and loss curves converge to the same values), as shown on the left of Figure~\ref{fig:nmf_laws}.  
We use a small step size to avoid discretization error (which impact is similar to the one reported in the previous section). 
For $\tau=+\infty$ the $r$ conservation laws are $1^\top U - 1^\top V$ and Figure~\ref{fig:nmf_laws} displays the evolution in time of the first of the quantities (associated with the first factor). As it is expected when $\mu=0$ (gradient flow) this law is perfectly conserved and is only approximately preserved for larger value of the momentum parameter $\mu$. Note however that an interesting phenomenon arises, that similarly to the MLP case, these quantities stay bounded for all time within a range depending on the momentum parameter $\mu$ (so if $\mu$ is small, approximate conservation holds for all time). Analyzing theoretically this non-trivial phenomenon is an interesting avenue for future work.

\begin{figure}[h]
    \centering
    \includegraphics[width=.4\textwidth]{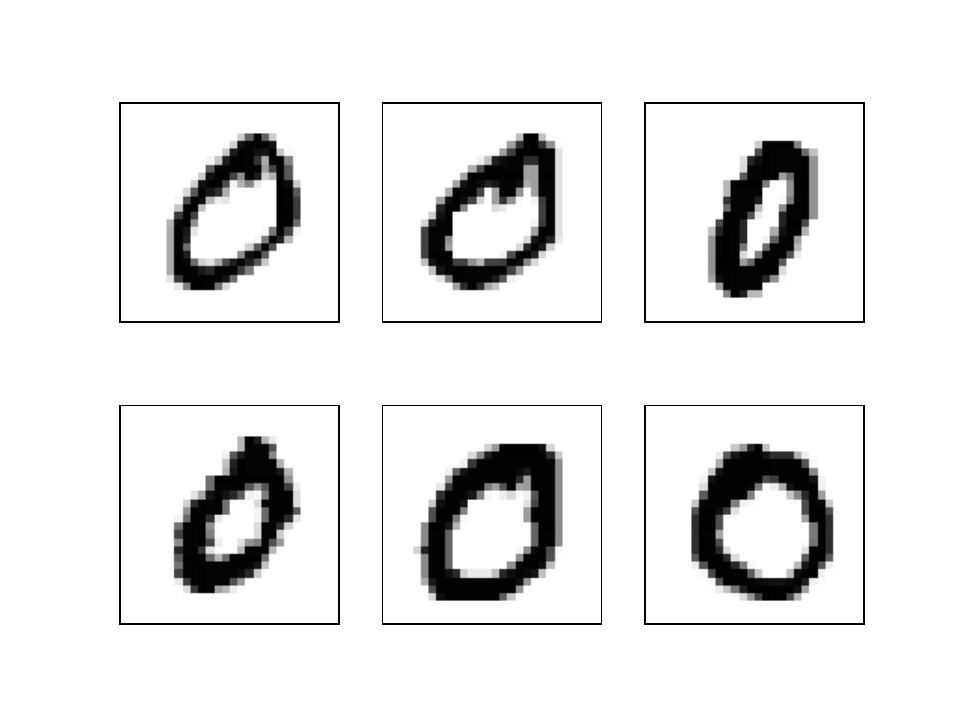} 
    \includegraphics[width=.4\textwidth]{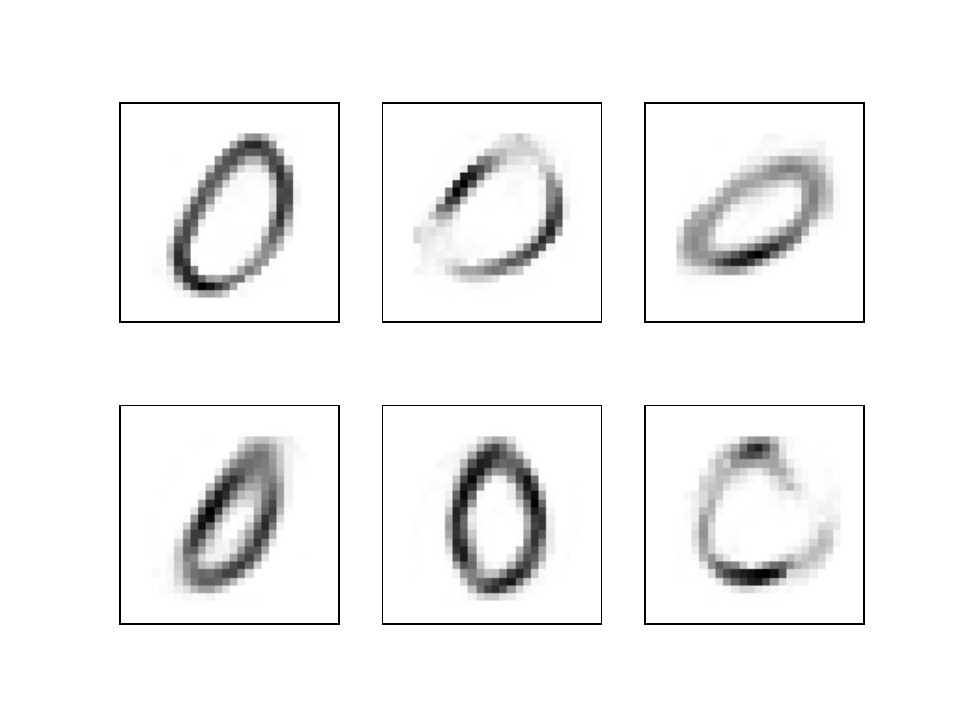} 
    \caption{Left: example of input images (columns of $Y$). Right: example of NMF factors (columns of $U$) at optimality. }
    \label{fig:mnf_data}
\end{figure}

\begin{figure}[h]
    \centering
    \includegraphics[width=.4\textwidth]{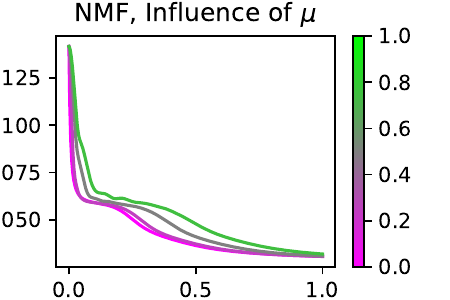} 
    \includegraphics[width=.4\textwidth]{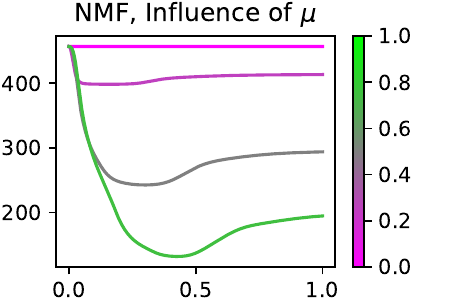} 
    \caption{Impact on the momentum parameter $\mu=1/\tau$ on the evolution of the loss (left) and on the preservation of one of the conservation laws (right). Here $\mu=0$ corresponds to the gradient flow (no momentum). }
    \label{fig:nmf_laws}
\end{figure}
\end{document}